\documentclass[11pt,a4paper]{article}
\usepackage{a4wide}
\usepackage{graphicx}
\usepackage[fleqn]{amsmath}
\usepackage{amssymb}
\usepackage{amsthm}
\usepackage{bm}
\usepackage{url}
\usepackage{color}

\usepackage{accents}

\newcommand{\transpose}{^{\mathrm{T}}}
\newcommand{\bbbr}{\mathbb{R}}
\newcommand{\ubar}[1]{\underaccent{\bar}{#1}}
\theoremstyle{plain}
\newtheorem{theorem}{Theorem}
\newtheorem{proposition}{Proposition}
\newtheorem{lemma}{Lemma}
\newtheorem{corollary}{Corollary}
\theoremstyle{definition}
\newtheorem{definition}{Definition}
\theoremstyle{remark}
\newtheorem{remark}{Remark}
\allowdisplaybreaks
\begin{document}
\title{Multivariate Median Filters and Partial Differential Equations}
\author{Martin Welk\\ University for Health Sciences, Medical Informatics and Technology (UMIT),\\ Eduard-Walln\"ofer-Zentrum 1, 6060 Hall/Tyrol, Austria\\ \url{martin.welk@umit.at} }
\date{March 18, 2015}
\maketitle

\begin{abstract}
Multivariate median filters have been proposed as generalisations
of the well-established median filter for grey-value images to
multi-channel images. As multivariate median, most of the recent
approaches use the $L^1$ median, i.e.\ the minimiser of an
objective function that is the sum of distances to all input points.
Many properties of univariate median filters generalise to
such a filter. However, the famous result by Guichard and Morel about
approximation of the mean curvature motion PDE by median filtering
does not have a comparably simple counterpart for $L^1$ multivariate median
filtering. We discuss the affine equivariant Oja median and the
affine equivariant trans\-for\-ma\-tion--re\-trans\-for\-ma\-tion 
$L^1$ median as alternatives to $L^1$ median filtering. We analyse 
multivariate median filters in a space-continuous setting, including 
the formulation of a space-continuous version of the 
transformation--retransformation $L^1$ median, and derive PDEs approximated
by these filters in the cases of bivariate planar images,
three-channel volume images and three-channel planar images.
The PDEs for the affine equivariant filters can be interpreted 
geometrically as combinations of a diffusion and a principal-component-wise
curvature motion contribution with a cross-effect term based on 
torsions of principal components.
Numerical experiments are presented that demonstrate the validity of
the approximation results.

\bigskip \noindent%
\end{abstract}

\sloppy
\section{Introduction}

Median filtering of signals and images goes back to the work of
Tukey \cite{Tukey-Book71} and has since then been established
in image processing as a simple nonlinear denoising method
for grey-value images with the capability to denoise even
impulse noise and similar types of noise with heavy-tailed 
distributions, and to retain at the same time sharp edges
in the denoising process.

Like other local image filters, the median filter consists of a
\emph{selection step} that identifies for each pixel location those pixels
which will enter the computation of the filtered value at that
location, followed by an \emph{aggregation step} that combines the intensities
of these pixels into the filtered value.
In the standard setting, the selection step uses a fixed-shape
sliding window, which can be called the \emph{structuring element} following
the naming convention from mathematical morphology.
The aggregation step
consists in taking the median of the selected intensities.
The process can be iterated, giving rise to what is called the
\emph{iterated median filter.}

The median filter, particularly in its iterated form, has been subject
to intensive investigation over the decades. For example,
\cite{Eckhardt-JMIV03} studied so-called
\emph{root signals,} non-trivial steady states that occur in the
iterated median filter and depend subtly on the choice of the structuring
element. Work by Guichard and Morel \cite{Guichard-sana97} has identified
iterated median filtering as an explicit nonstandard discretisation of
(mean) curvature motion \cite{Alvarez-SINUM92}, thus establishing a link
between the discrete filter concept and a partial differential equation (PDE).

\paragraph{Multivariate median filtering.}
Given the merit of median filtering in processing grey-value images one
is interested in stating also a median filter for multi-channel images
such as colour images, flow fields, tensor fields etc.
As the switch from single- to multi-channel images does not affect the
selection step mentioned above but solely the aggregation, it is clear
that what is needed to accomplish this goal is the definition of a
multivariate median. 

A straightforward approach to median filtering of multi-channel data
is to establish some kind of linear order in $\bbbr^n$. For example,
\cite{Caselles-JMIV00} considered a vector median filter of this type
(based on lexicographic ordering) and derived even PDEs for this filter.
A clear shortcoming of such an approach, however, is that mapping
$\bbbr^n$ to $\bbbr$ (which necessarily happens with a linear order) 
either breaks injectivity or continuity, and is usually incompatible with
natural geometric invariances of the data colour space, like symmetries
of colour spaces, or the Euclidean or affine structures of flow vectors 
or tensor spaces.

A starting point for a multivariate median definition that avoids these
problems is the
following characterisation of the univariate median: A
median of a tuple $\mathcal{X}=(x_1,\ldots,x_N)$ of real numbers 
$x_1,\ldots,x_N$ is a real number that minimises
the sum of distances to all numbers of the set,
\begin{equation}
m(\mathcal{X})
=\mathop{\operatorname{argmin}}\limits_{x\in\bbbr}
\sum_{i=1}^N\lvert x-x_i\rvert \;.
\label{m1d}
\end{equation}
Strictly speaking, this minimiser is unique only if the data set is of odd
cardinality; for even-numbered input sets, the two middle elements in the
rank order and all real numbers in between fulfil the criterion, making
$\mathrm{argmin}$ actually set-valued.
Heuristics like mean value are often used to disambiguate the median in
this situation. We will not consider this here but keep in mind that there
is a whole set of medians in this case. At any rate, in the univariate case,
there happens to always exist a number from the given data set which is a
median of this set, such that one can also write
\begin{equation}
m(\mathcal{X})
=\mathop{\operatorname{argmin}}\limits_{x\in\mathcal{X}}
\sum_{i=1}^N\lvert x-x_i\rvert \;.
\label{medoid1d}
\end{equation}

Early attempts to multi-channel median filtering in the computer science and
signal processing literature, starting from \cite{Astola-PIEEE90} in 1990, 
defined therefore a vector-valued ``median''
that selects \emph{from the set of input points} in $\bbbr^n$ the one that
minimises the sum of distances to all other sample points.
Given a tuple $\mathcal{X}:=(\bm{x}_1,\ldots,\bm{x}_N)$ of points
$\bm{x}_i\in\bbbr^n$, this amounts to
\begin{equation}
\bm{m}_{L^1\upharpoonright\mathcal{X}}(\mathcal{X}):=
\mathop{\operatorname{argmin}}\limits_{\bm{x}\in\mathcal{X}}
\sum_{i=1}^N\lVert\bm{x}-\bm{x}_i\rVert\;.
\label{medoidL1}
\end{equation}
In a more differentiating terminology, see e.g.\ \cite{Struyf-JSS97},
such a concept would rather be called a \emph{medoid.}

More recent approaches, such as
\cite{Kleefeld-cciw15,Spence-icip07} for colour images or \cite{Welk-dagm03}
for symmetric
matrices, rely on the same minimisation
but without the restriction to the
given data points, i.e.\ (in the same notations as before)
\begin{equation}
\bm{m}_{L^1}(\mathcal{X}):=
\mathop{\operatorname{argmin}}\limits_{\bm{x}\in\bbbr^n}
\sum_{i=1}^N\lVert\bm{x}-\bm{x}_i\rVert\;.
\label{mL1}
\end{equation}
The underlying multivariate median concept can be traced back in
the statistics literature to works by Hayford from 1902 \cite{Hayford-JASA1902} 
and Weber from 1909 \cite{Weber-Book1909}, followed by
\cite{Austin-Met59,Gini-Met29,Weiszfeld-TMJ37} and many others.
It is nowadays denoted as the \emph{spatial median} or \emph{$L^1$ median}. 
The $L^1$ median is unique for all non-collinear input data sets. Only for
collinear sets non-uniqueness as for the univariate median takes place; in
this case, the $\mathrm{argmin}$ in \eqref{mL1} is actually set-valued.
As these configurations are non-generic, we do not follow this issue further.
For the computation of $L^1$ medians, efficient algorithms are available,
see e.g.\ \cite{Vardi-MP01}.

However, the $L^1$ median is not the only multivariate median concept in 
literature. Another generalisation of the same minimisation property of the
univariate median was introduced by Oja in 1983 \cite{Oja-StPL83} and is
known as the \emph{simplex median} or \emph{Oja median.}
Here, distances between points on the real line from the univariate median 
definition are generalised not to distances in $\bbbr^n$ but to simplex 
volumes. Thus, the simplex median of a finite set of points
in $\bbbr^n$ is the point $\bm{m}\in\bbbr^n$ that minimises the sum of
simplex volumes $\lvert[\bm{m},\bm{a}_1,\ldots,\bm{a}_n]\rvert$ where
$\bm{a}_i$ are distinct points of the input data set, i.e.\
\begin{equation}
\bm{m}_{\mathrm{Oja}}(\mathcal{X}):=
\mathop{\operatorname{argmin}}\limits_{\bm{x}\in\bbbr^n}
\sum_{1\le i_1<\ldots<i_n\le N}\!\!
\lvert[\bm{x},\bm{x}_{i_1},\ldots,\bm{x}_{i_n}]\rvert \;.
\label{mOja}
\end{equation}
An advantage of this concept that is relevant
for many statistics applications is its affine equivariance,
i.e.\ that it commutes with affine
transformations of the data space. In contrast, the $L^1$ median
only affords Euclidean equivariance.
It should be noticed that also in an image processing context
affine equivariance offers an advantage over just Euclidean equivariance:
For images whose value ranges are not equipped with a meaningful Euclidean
structure, justification of Euclidean equivariant concepts like the $L^1$
median is questionable.

While there exist in any dimension even datasets that are not degenerated 
to hyperplanes whose Oja median is non-unique, these cases are non-generic.
A more substantial caveat is that the Oja median is always undefined when
the input data lie on a common hyperplane. Heuristics exist to cure this
but usually these interfere with affine equivariance.

Whereas the affine equivariance of the Oja median concept has been welcomed
in the statistical community, its computational complexity was considered
a problem from the beginning, see the discussion in Section~\ref{ssec-demo-num}. 
On one hand, there are some results regarding more efficient computation
of Oja medians, see e.g.\ \cite{Aloupis-CG03,Ronkainen-drs03}. On the other
hand, researchers have been inspired soon to design multivariate median
concepts that combine affine equivariance with the efficiency of the $L^1$
median \cite{Chakraborty-PAMS96,Hettmansperger-Biomet02,Rao-Sankhya88},
see also the survey in \cite{Chakraborty-StPL99}.
In these approaches, affine equivariance is
achieved using a \emph{transformation--retransformation} method.
Input data sets are normalised by a data-dependent affine transform
$\bm{T}_{\mathcal{X}}:\bbbr^n\to\bbbr^n$. Applying
the standard $L^1$ median and transforming back to the original data space
then yields an affine equivariant median operation
\begin{align}
\bm{m}_{L^1;\mathrm{aff}}(\mathcal{X})&:=
\bm{T}_{\mathcal{X}}^{-1}
\left(\bm{m}_{L^1} \bigl(\bm{T}_{\mathcal{X}}(\mathcal{X})\bigr)\right)
=\bm{T}_{\mathcal{X}}^{-1}
\left(\mathop{\operatorname{argmin}}\limits_{\bm{y}\in\bbbr^n}
\sum_{i=1}^N\lVert\bm{y}-\bm{T}_{\mathcal{X}}(\bm{x}_i)\rVert\right)\;.
\label{mL1a}
\end{align}
The data-dependent affine transform in these approaches is typically
based on an estimator of the covariance matrix of the distribution
underlying the observed data, such that the transformed data are
supposed to follow an isotropic distribution.

Besides these multivariate median concepts that generalise in different
ways the distance sum minimisation property of the univariate median,
there exist several other concepts which we will not consider here, see 
the review \cite{Small-ISR90}.

\paragraph{Multivariate median filters and PDE.}
While the above-mentioned relationship between univariate median filtering
and the mean curvature motion PDE could be extended to relate also adaptive
median filtering procedures \cite{Welk-JMIV11} and further discrete filters
\cite{Welk-Aiep14} to well-understood PDEs of image processing,
the picture changes when turning to multivariate median filtering.
As demonstrated in \cite{Welk-Aiep14}, it is possible to derive some
PDE for median filtering based on the spatial median as in
\cite{Spence-icip07}. 
However, this PDE
involves complicated coefficient functions coming from elliptic integrals
most of which cannot even be stated in closed form, see
\cite{Welk-Aiep14} and for the bivariate case \cite{Welk-ssvm15}.
During the present work it became evident that
the analysis of the $L^1$ median filter in \cite{Welk-Aiep14} 
contained a mistake with the consequence that one term was omitted in
the resulting PDE.
We will state in the present paper corrected results for the case of two-
and three-channel data, the latter restricted to a relevant special case.
A corrected result for the general multivariate case with proof will be 
provided in a forthcoming technical report \cite{Welk-tr15x}.

Given the unfavourable complexity of the PDE approximated by $L^1$
median filtering,
the question arises whether other multivariate median
concepts could be advantageous in multi-channel image processing.
The paper \cite{Welk-ssvm15} was intended
as a first step in this direction
which is continued in the present contribution.
Whereas in \cite{Welk-ssvm15} only bivariate images over planar domains
(like 2D flow fields or, somewhat artificial, two-colour images)
were covered, we extend the view here to include three-channel volume
images (like 3D flow fields) and three-channel planar images (like
colour images). Moreover, we include also an affine equivariant 
transformed $L^1$ median filter based on the transformation--retransformation
procedure in our analysis.

\paragraph{Our contribution.}
This paper extends the work from \cite{Welk-ssvm15}.
Regarding bivariate median filtering of planar images, we restate in this
paper the PDE approximation result for the Oja median from \cite{Welk-ssvm15}.
We present its proof from \cite{Welk-ssvm15} in a slightly modified and
more detailed form, and present a new, alternative proof. We compare the
PDE with that for bivariate $L^1$ median filtering and discuss the geometric 
meaning of these PDEs, showing that they combine an isotropic diffusion
contribution with a curvature motion part and torsion-based cross-effects
between the channels.
We also discuss the degeneracy of the PDE approximated
by the Oja median when the Jacobian of the input function becomes singular.
We also give a formulation for a space-continuous version of the
transformation--retransformation median, which enables us, by recombining
ideas from the analysis of the $L^1$ and Oja median filters, to derive a PDE 
approximation statement for this filter. The outcome is that the two affine
equivariant medians, Oja median and transformation--retransformation $L^1$
median, are asymptotically equivalent as image filters in the case of 
bivariate planar images.

In the case of
three-channel volume images, we prove PDE approximation results for
the Oja median and transformation--retransformation $L^1$ median.
The PDE is again identical for both filters, implying their asymptotical
equivalence. Its structure is analogous to the bivariate case, with the
diffusion, mean curvature motion and torsion-based cross-effect terms.

For three-channel planar images, for which the 3D Oja median on 
local neighbourhoods is degenerated or almost degenerated, we compare the
2D Oja median (minimiser of sum of triangle areas) applied to 3D data
with the transformation--retransformation $L^1$ median, and derive PDE
approximation results for both, which again display the same structure
as in the cases before and confirm asymptotical equivalence of the two
filters.

We test, and verify to reasonable accuracy,
the PDE approximations in all dimensional settings
by numerical experiments that compare discrete multivariate median filters 
for example functions sampled at high grid resolutions with theoretically
derived PDE time steps. Finally, we investigate iterated Oja and 
transformation--retransformation $L^1$ median filtering of RGB colour images 
and compare it to a numerical evaluation of the corresponding PDE. These
experiments, too, confirm the theoretical results.

\paragraph{Structure of the paper.}
The remainder of the paper is structured as follows.
In Section~\ref{sec-geo}, we demonstrate two- and three-channel median filters
on a 2D flow field, as a bivariate test case, and RGB colour images, as a
three-channel example. For the latter, we consider four variants of
three-channel medians: $L^1$, 2D Oja, 3D Oja and 
transformation--retransformation $L^1$ median.
Finally, we discuss basic geometric properties of the $L^1$ and Oja medians in
the bivariate setting.
Section~\ref{sec-pde} is dedicated to the analysis of multivariate median
filters for bivariate planar images, three-channel volume images and
three-channel planar images.
PDE approximation results generalising 
Guichard and Morel's \cite{Guichard-sana97} result for the univariate case
are derived in all settings, and discussed.
In Section~\ref{sec-ex} the results of the theoretical analysis are
validated by numerical experiments on analytic example functions and
RGB images, where the latter also cover iterated median filtering.
A summary and outlook is given in Section~\ref{sec-summ}.
Appendices~\ref{app-22proof1}--\ref{app-23proof} contain detailed proofs
for lemmas from Section~\ref{sec-pde}. Appendix~\ref{app-pdealgo}
details a finite-difference scheme for the PDE approximated by affine
equivariant median filters for RGB images that is used for the experiments
in Section~\ref{sec-ex}.

\section{Comparison of $L^1$ and Oja Median}\label{sec-geo}

To motivate our theoretical analysis, we will demonstrate
in this section the effects of image filters
based on the $L^1$ and Oja median by experiments on 
image and flow field data.
Additionally, some geometric intuition
about the $L^1$ and Oja medians in the bivariate case will be given to
help understanding their relations.

\subsection{Numerical Realisation of Multivariate Median Filters}
\label{ssec-demo-num}

Before we turn to presenting filtering experiments, some words
need to be said about the implementations of the filters as they are
used in this paper.
Given the focus of this work at theoretical connections,
simplicity and comparability of the implementations are in the foreground.
Computational efficiency is not a goal, thus possibilities for
improvements in this respect are only touched grazingly.

Since the objective functions of the $L^1$ and Oja medians are convex,
one can think of numerous generic minimisation algorithms. However,
the objective functions are only piecewise smooth, and may be
extremely anisotropic around their minima. This poses difficulties for
algorithms.
For the numerical computation of $L^1$ and Oja medians in this work,
we use therefore a gradient descent approach with adaptive step-size
control using a line search, similar to the proceeding described
in \cite{Welk-dagm03}.
The advantage of this approach is its simplicity and
the fact that it can be used in a straightforward way for all median
variants considered in this work.

For the $L^1$ median $\bm{m}_{L^1}(\mathcal{X})$, 
one reads off \eqref{mL1} the objective function
$f(\bm{x})=\sum_{i=1}^N\lVert\bm{x}-\bm{x}_i\rVert$.
Its gradients are
computed directly by summation over the data points, which has a
linear complexity $\mathcal{O}(N)$,
which is fast enough to filter e.g.\ $512\times512$
images with structuring elements of radius $5$ within less than 3 minutes
in single-threaded CPU computation on a 3\,GHz machine. 
A substantially more efficient computation would be possible by using an
iterative weighted means algorithm for the $L^1$ median, see 
\cite{Vardi-MP01}.

For Oja medians $\bm{m}_{\mathrm{Oja}}(\mathcal{X})$ 
of two- and three-dimensional input data, see \eqref{mOja}, the 
objective functions are sums of triangle areas,
$f(\bm{x})=\sum_{1\le i<j\le N}\lvert[\bm{x},\bm{x}_i,\bm{x}_j]\rvert$,
or tetrahedron volumes, 
$f(\bm{x})=\sum_{1\le i<j<k\le N}
\lvert[\bm{x},\bm{x}_i,\bm{x}_j,\bm{x}_k]\rvert$,
respectively.
Their gradients are computed here by summation over pairs or triples,
respectively, of data points, which amounts to an $\mathcal{O}(N^2)$
or $\mathcal{O}(N^3)$ complexity, respectively, and is therefore
computationally expensive. It is possible in this way
to compute two- and three-dimensional Oja medians of test functions
within sampled structuring elements and image filters based
on two-dimensional Oja medians, with computation times ranging from
minutes to hours in single-threaded CPU computation, depending on
image and structuring element sizes, and numerical convergence criteria
for the gradient descent. The convergence of the gradient descent
computation can be somewhat accelerated if the input data are 
subjected to an affine transformation that makes their distribution
more isotropic, which is possible based on the affine equivariance of
the Oja median. Principal axis transform of the covariance matrix
can be used to determine a suitable transformation.

In practical application contexts, the computational expense of
such an Oja median filter would be unacceptable. 
Let us therefore mention possible alternatives.
For the bivariate case, \cite{Aloupis-CG03} describes an algorithm
that allows to compute two-dimensional Oja medians in 
$\mathcal{O}(N\log^3 N)$ time. This is achieved by an angular reordering
of points in the gradient computation together with geometric considerations
that limit the possible locations for Oja medians to a small set of
discrete points in the plane.
It can be expected that using this
algorithm would speed up an image filter with a structuring element of 
radius $5$ (approx.\ $80$ sample points) by two to three orders of 
magnitude. Highly parallel computation such as on GPUs
would further improve on this.

For Oja medians in general dimensions, we refer to \cite{Ronkainen-drs03}
where several exact and stochastic algorithms are discussed.

An additional difficulty with Oja medians specifically in image
filtering results from the frequent occurrence of degenerated input data. 
In a multivariate image,
data vectors belonging to pixels from a small local neighbourhood will
often concentrate around or even lie on a lower-dimensional submanifold
of the actual data space. In such a case, the objective function of the
Oja median is made up by volumes of degenerated or almost degenerated
simplices, and the filtering result becomes undetermined or numerically
unstable.

One simple, albeit expensive, way to cope with these degeneracies of Oja
medians is to replace each input data point with a set of
data points that are isotropically scattered in a small neighbourhood
of the actual input point. Thereby one enforces the full dimensionality
of the input set, thus the input data are regularised. Note, however,
that the isotropic scattering of the new data points involves a notion
of metric, and thus goes at the cost of affine equivariance.
In our experiment series with Oja median filtering on one test image
(shown in Figure~\ref{fig-co01} in Section~\ref{ssec-demo-rgb} and used
again in Figure~\ref{fig-co01-evo} in Section~\ref{ssec-evo})
we perform this
kind of input regularisation by replacing each input point by the corners
of a regular simplex centered at the input point, along with the
above-mentioned principal axis transform. All other Oja median
experiments are done with the plain gradient descent algorithm
without these modifications.

To complement the standard $L^1$ median and Oja median filters,
we want to perform also filtering based on the 
affine equivariant transformed $L^1$ median \eqref{mL1a}.
The affine transform $\bm{T}_{\mathcal{X}}$ for a tuple $\mathcal{X}$ of input
data is computed from the 
same principal axis transform of the covariance matrix as mentioned above
in such a way that the covariance matrix for the transformed data
$\bm{T}_{\mathcal{X}}(\mathcal{X})$ becomes diagonal,
with the diagonal entries being $1$ in most cases. Only if the original
covariance matrix is singular or almost singular, some of the diagonal
entries will be close or equal to $0$.
The $L^1$ median $\bm{m}_{L^1}$ inside \eqref{mL1a} is computed
by our gradient descent method.

\subsection{Median Filtering of 2D Flow Fields}
\label{ssec-demo-ff}

\begin{figure}[t!]
\unitlength0.01\textwidth
\begin{picture}(100,24)
\put(0,0){\includegraphics[width=32\unitlength]{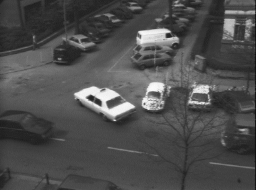}}
\put(34,0){\parbox[b]{66\unitlength}{%
\caption{\label{fig-taxi}
Frame 5 from the \emph{Hamburg taxi sequence} (author: H.-H.\ Nagel),
size: $256\times190$ pixels.}}}
\end{picture}
\end{figure}

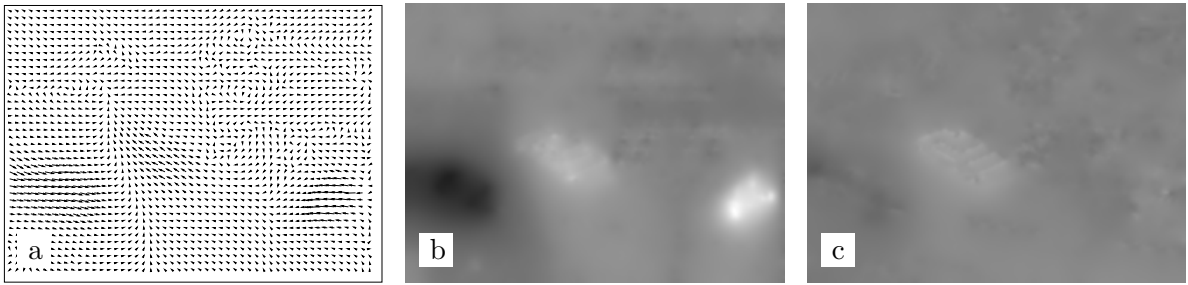
\begin{figure}[t!]
\unitlength0.01\textwidth
\begin{picture}(100,24)
\put(0,0){\includegraphics[width=32\unitlength]{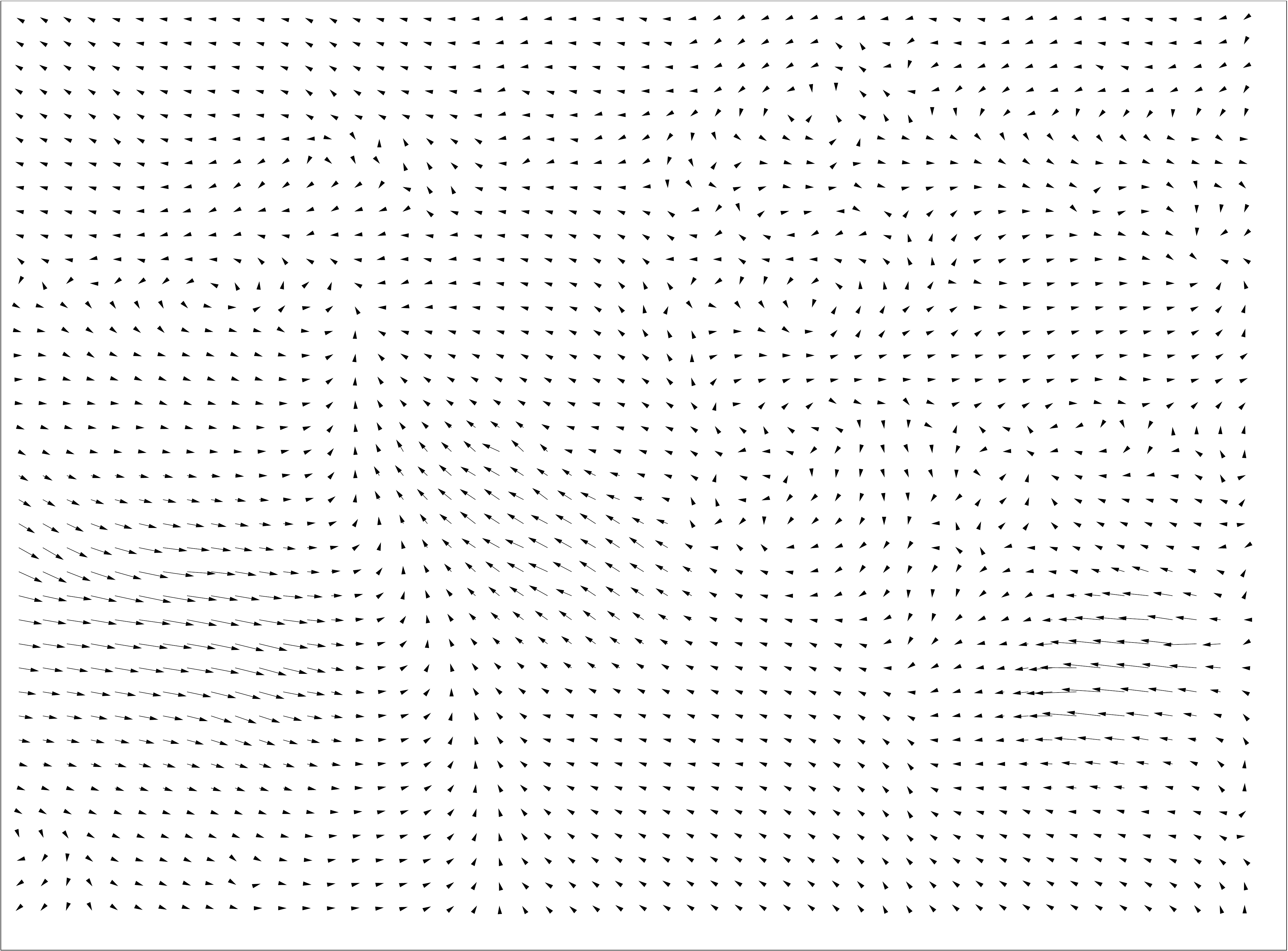}}
\put(34,0){\includegraphics[width=32\unitlength]{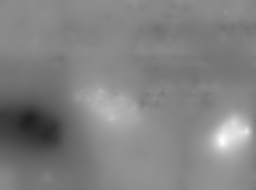}}
\put(68,0){\includegraphics[width=32\unitlength]{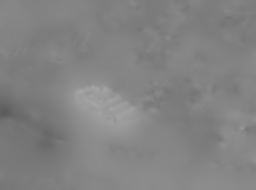}}
\put( 1.1,2){\colorbox{white}{\rule{0pt}{.6em}\hbox to.6em{\kern.1em\smash{a}}}}
\put(35.1,2){\colorbox{white}{\rule{0pt}{.6em}\hbox to.6em{\kern.1em\smash{b}}}}
\put(69.1,2){\colorbox{white}{\rule{0pt}{.6em}\hbox to.6em{\kern.1em\smash{c}}}}
\end{picture}
\caption{\label{fig-flow}
Optical flow between Frames 5 and 6 of the Hamburg taxi sequence,
computed by a coarse-to-fine Horn-Schunck method with warping.
Magnitudes of vector entries range up to approx.\ $2.44$.
\textbf{(a)} Flow field visualised by vector arrows, subsampled (every 5th
flow vector in $x$ and $y$ direction is shown). --
\textbf{(b)} Horizontal component of the same flow field. Grey (128)
represents zero, brighter values represent flows to the left,
darker values flows to the right. --
\textbf{(c)} Vertical component of the flow field.
Grey represents zero, brighter values represent upward flows,
darker values downward flows.
}
\end{figure}

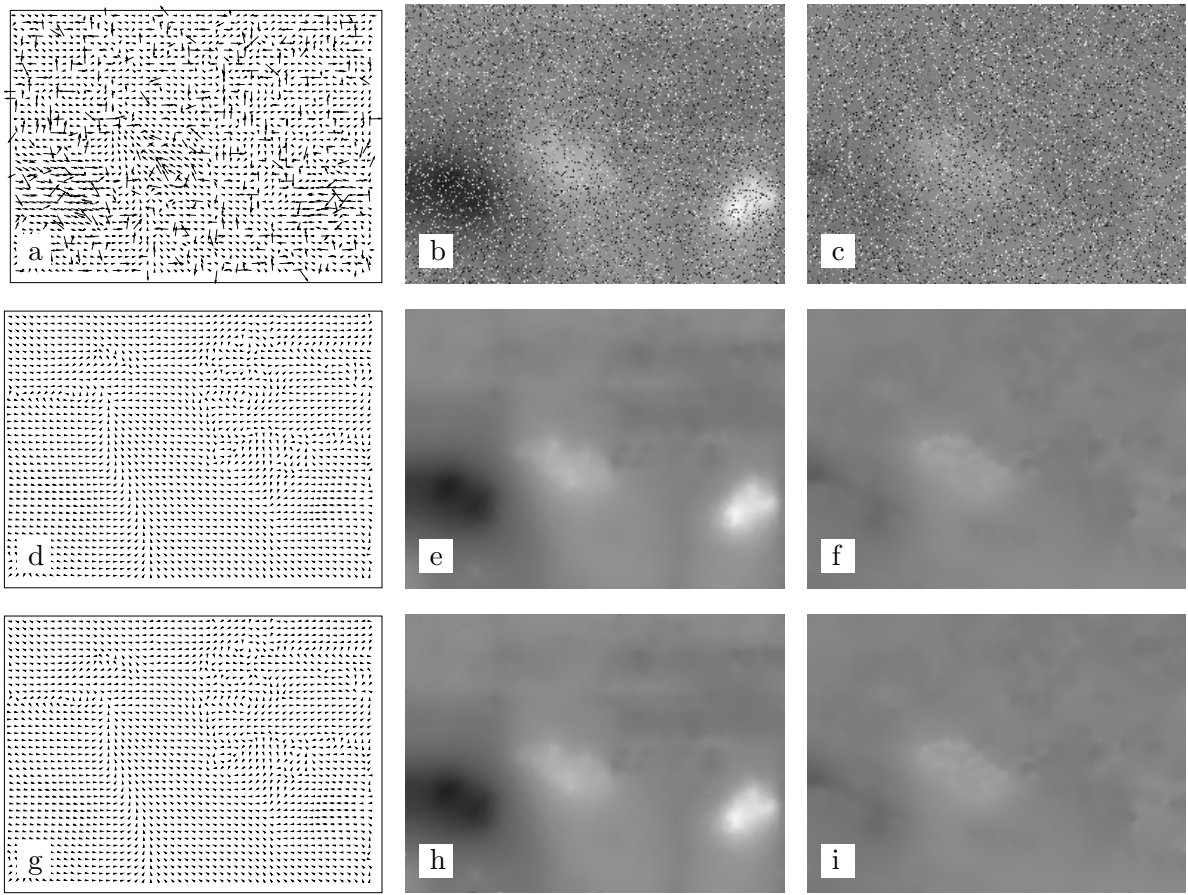
\begin{figure}[t!]
\unitlength0.01\textwidth
\begin{picture}(100,76)
\put(0,52){\includegraphics[width=32\unitlength]{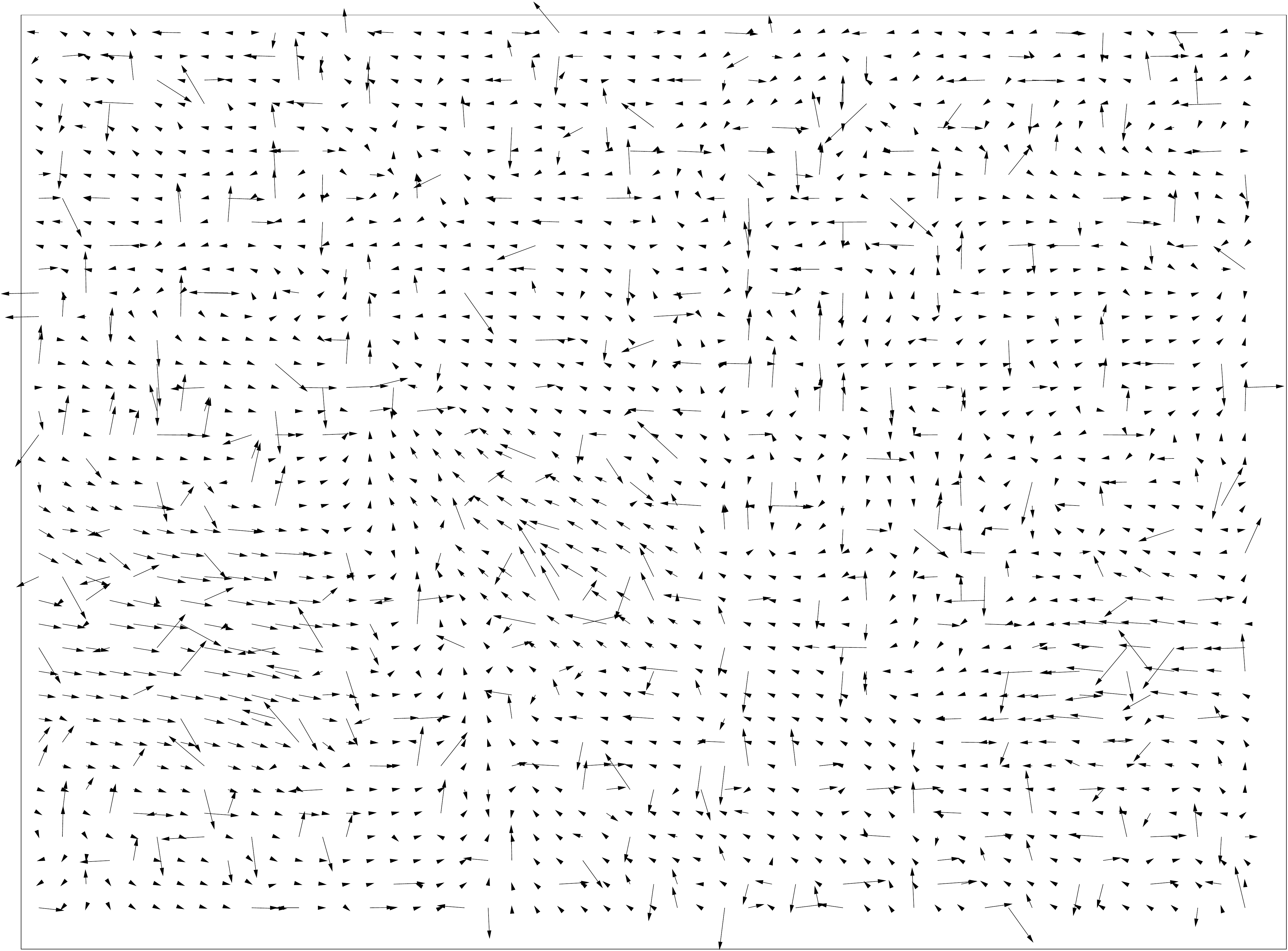}}
\put(34,52){\includegraphics[width=32\unitlength]{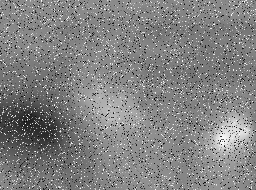}}
\put(68,52){\includegraphics[width=32\unitlength]{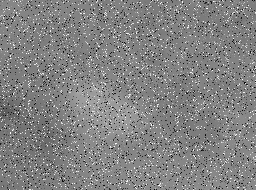}}
\put( 1.1,54){\colorbox{white}{\rule{0pt}{.6em}\hbox to.6em{\kern.1em\smash{a}}}}
\put(35.1,54){\colorbox{white}{\rule{0pt}{.6em}\hbox to.6em{\kern.1em\smash{b}}}}
\put(69.1,54){\colorbox{white}{\rule{0pt}{.6em}\hbox to.6em{\kern.1em\smash{c}}}}
\put(0,26){\includegraphics[width=32\unitlength]{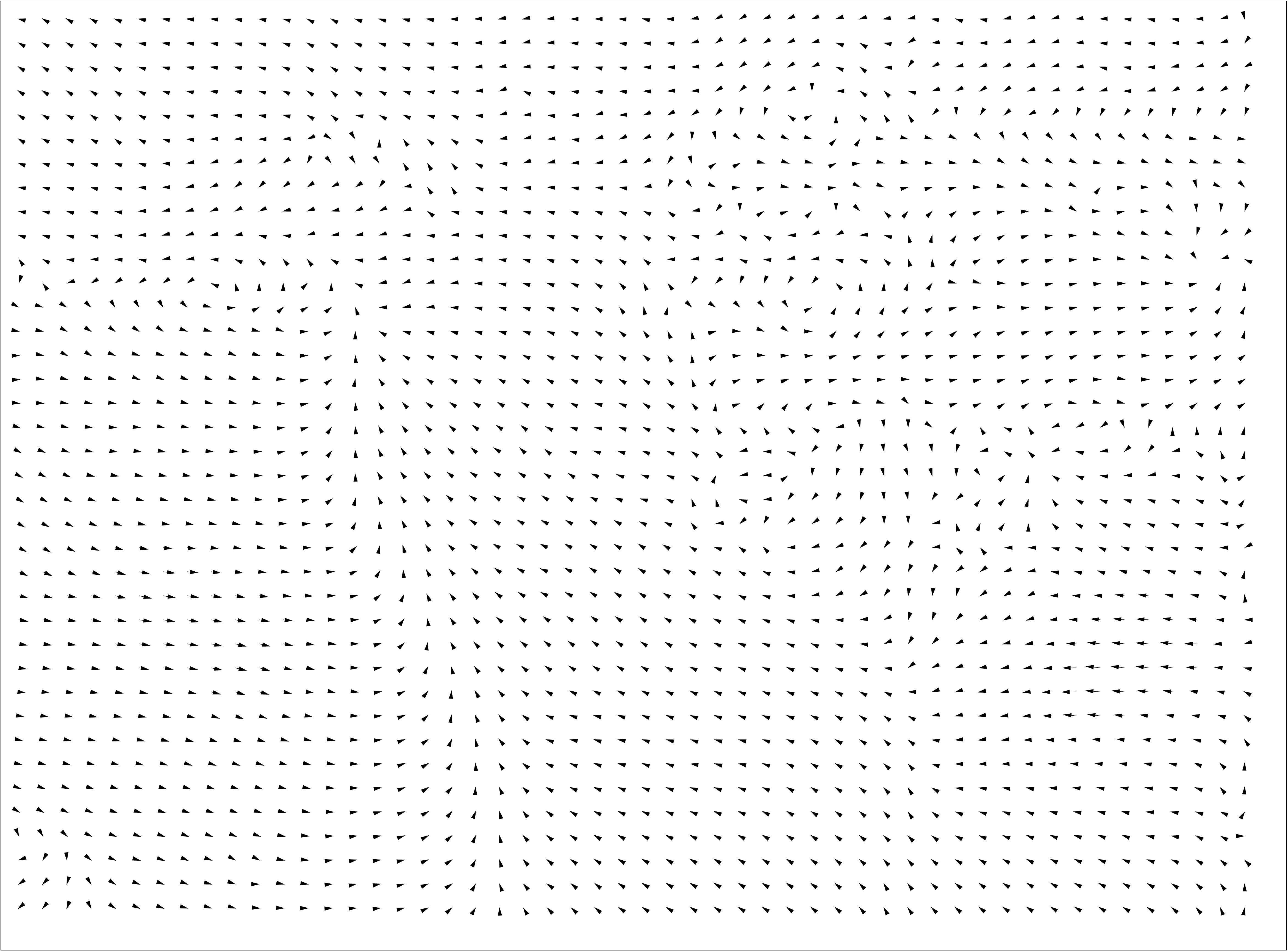}}
\put(34,26){\includegraphics[width=32\unitlength]{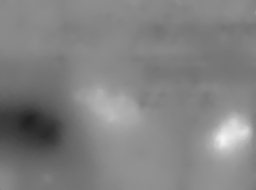}}
\put(68,26){\includegraphics[width=32\unitlength]{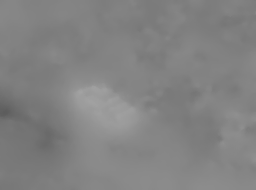}}
\put( 1.1,28){\colorbox{white}{\rule{0pt}{.6em}\hbox to.6em{\kern.1em\smash{d}}}}
\put(35.1,28){\colorbox{white}{\rule{0pt}{.6em}\hbox to.6em{\kern.1em\smash{e}}}}
\put(69.1,28){\colorbox{white}{\rule{0pt}{.6em}\hbox to.6em{\kern.1em\smash{f}}}}
\put(0,0){\includegraphics[width=32\unitlength]{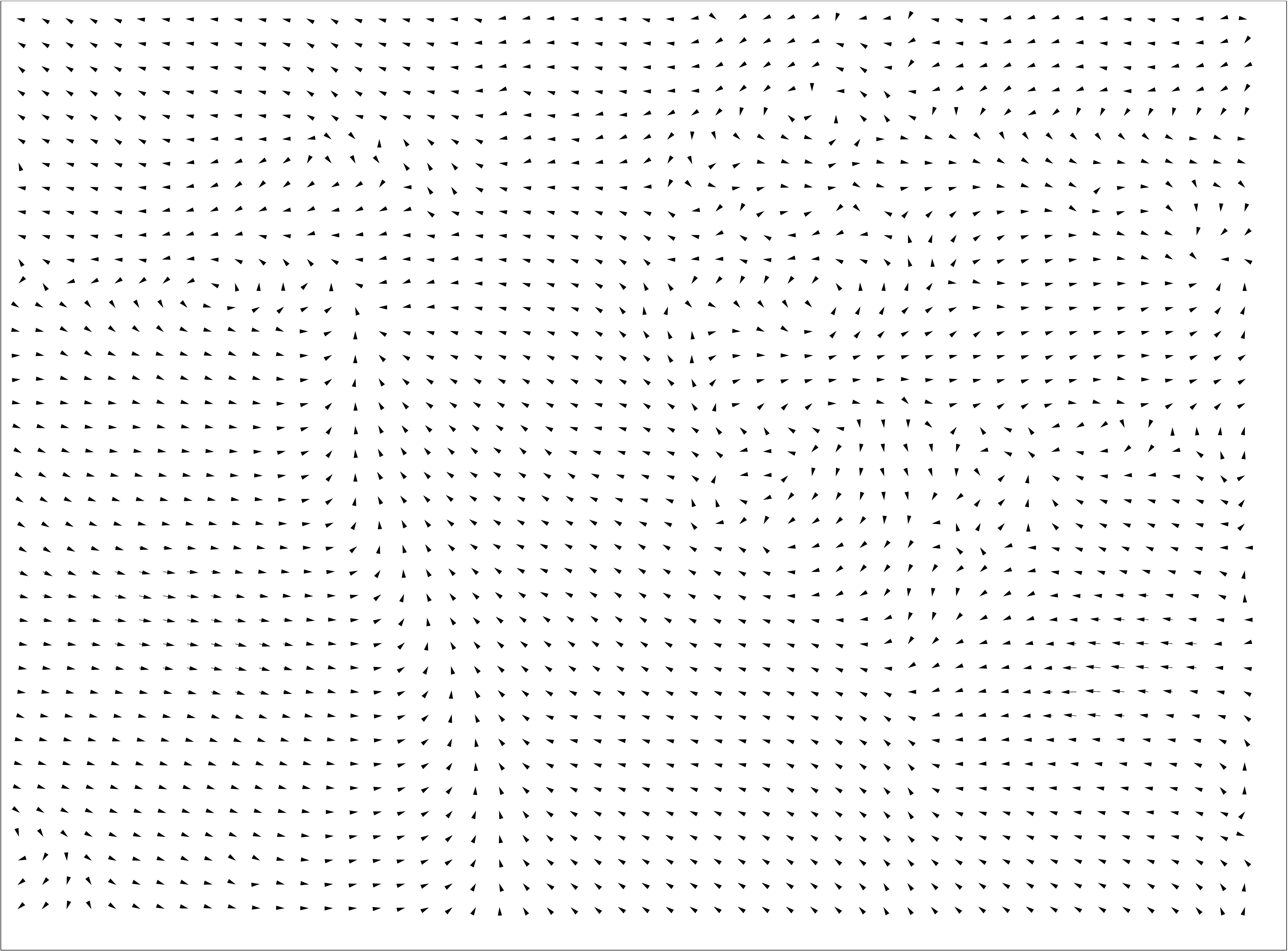}}
\put(34,0){\includegraphics[width=32\unitlength]{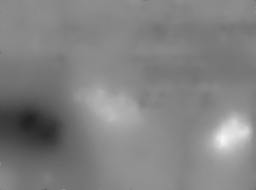}}
\put(68,0){\includegraphics[width=32\unitlength]{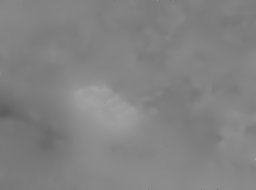}}
\put( 1.1,2){\colorbox{white}{\rule{0pt}{.6em}\hbox to.6em{\kern.1em\smash{g}}}}
\put(35.1,2){\colorbox{white}{\rule{0pt}{.6em}\hbox to.6em{\kern.1em\smash{h}}}}
\put(69.1,2){\colorbox{white}{\rule{0pt}{.6em}\hbox to.6em{\kern.1em\smash{i}}}}
\end{picture}
\caption{\label{fig-flow-un}
\textbf{Top row:}
Flow field from Figure~\ref{fig-flow} degraded by uniform noise where 
$20\,\%$ of the vector entries have been replaced by random values in
$[-2.44,2.44]$.
\textbf{(a)} Subsampled vector field representation. --
\textbf{(b)} Horizontal component. --
\textbf{(c)} Vertical component. --
\textbf{Middle row (d--f):}
Denoised by one step of $L^1$ median filtering with a disc-shaped 
structuring element of radius $3$.
\textbf{Bottom row (g--i):}
Denoised by one step of Oja median filtering with the same structuring
element as in (d--f).
}
\end{figure}

We turn first to the bivariate case. Possible applications for this
setup include two-channel colour images, for which an example was
presented in \cite{Welk-ssvm15}, or, with more practical relevance,
2D flow fields.

We demonstrate here bivariate median filtering on an
exemplary flow field computed from
two frames of the \emph{Hamburg taxi sequence}. The first of these 
frames is shown in Figure~\ref{fig-taxi}. Within the sequence, the
taxi moves in the upper left direction, whereas two vehicles enter the
scene from the left and right margin.
The flow field, visualised in Figure~\ref{fig-flow}, has been obtained
using an implementation of the Horn-Schunck method \cite{Horn-AI81}
within a coarse-to-fine multiscale approach with warping
\cite{Anandan-IJCV89,Memin-TIP98} in order to cope with displacements
larger than one pixel.

In the top row of Figure~\ref{fig-flow-un} this flow field has been
degraded by uniform impulse noise with $20\,\%$ density applied to
the horizontal and vertical flow components independently.
The middle and bottom row of Figure~\ref{fig-flow-un} show results
of $L^1$ and Oja median filtering, respectively, both of which succeed
to remove the noise and restore a smooth flow field similar to the
original one. Note that the filtering results of both median filters
are very similar.

\subsection{Median Filtering of RGB Images}
\label{ssec-demo-rgb}

\begin{figure}[t!]
\unitlength0.01\textwidth
\begin{picture}(100,74.7)
\put( 0.0,50.6){\includegraphics[width=24.1\unitlength]{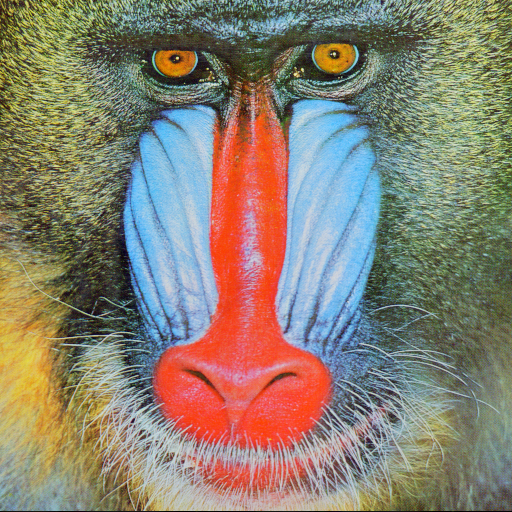}}
\put(25.3,50.6){\includegraphics[width=24.1\unitlength]{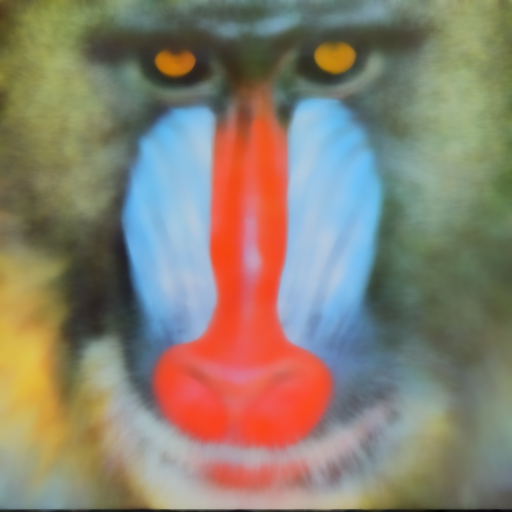}}
\put(50.6,50.6){\includegraphics[width=24.1\unitlength]{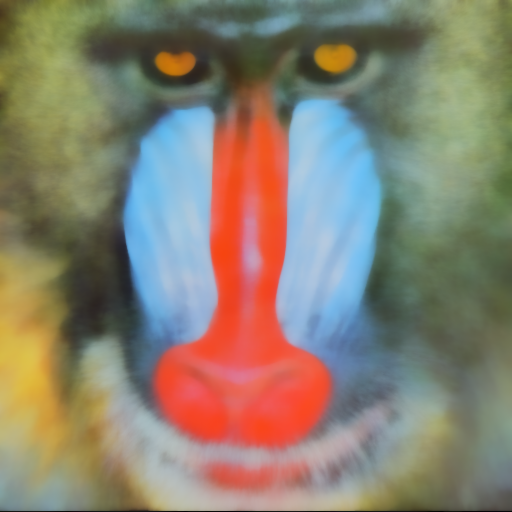}}
\put(75.9,50.6){\includegraphics[width=24.1\unitlength]{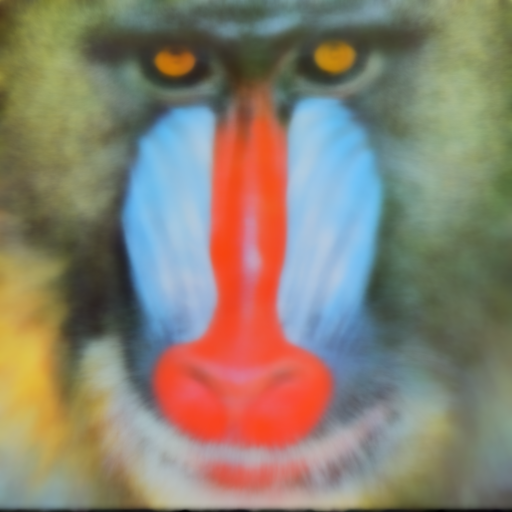}}
\put( 1.1,52.6){\colorbox{white}{\rule{0pt}{.6em}\hbox to.6em{\kern.1em\smash{a}}}}
\put(26.4,52.6){\colorbox{white}{\rule{0pt}{.6em}\hbox to.6em{\kern.1em\smash{b}}}}
\put(51.7,52.6){\colorbox{white}{\rule{0pt}{.6em}\hbox to.6em{\kern.1em\smash{c}}}}
\put(77.0,52.6){\colorbox{white}{\rule{0pt}{.6em}\hbox to.6em{\kern.1em\smash{d}}}}
\put( 0.0,25.3){\includegraphics[width=24.1\unitlength]{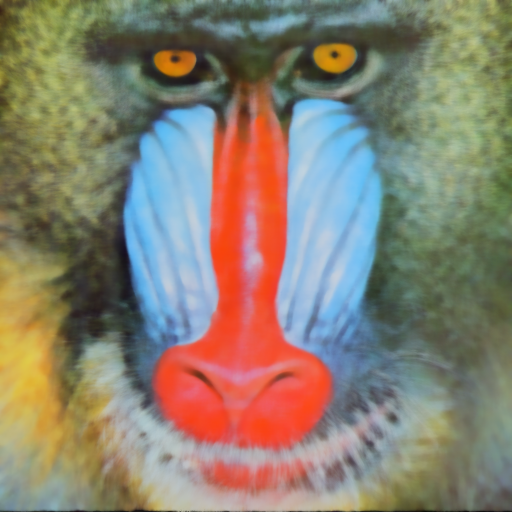}}
\put(25.3,25.3){\includegraphics[width=24.1\unitlength]{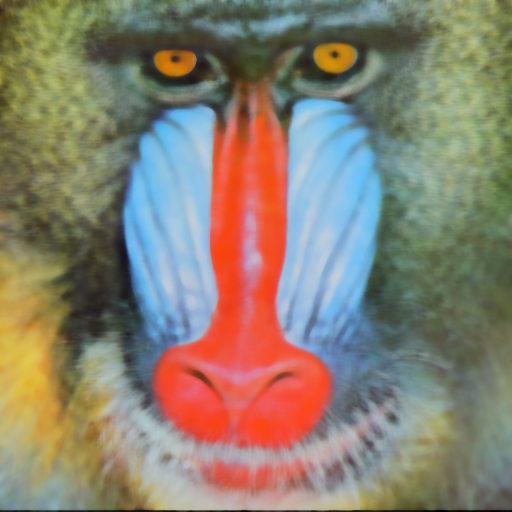}}
\put(50.6,25.3){\includegraphics[width=24.1\unitlength]{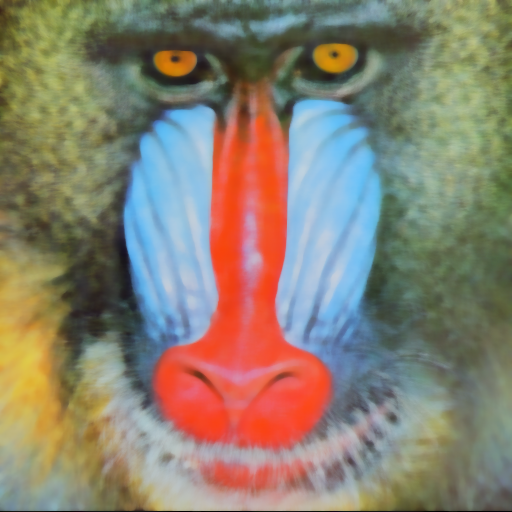}}
\put(75.9,25.3){\includegraphics[width=24.1\unitlength]{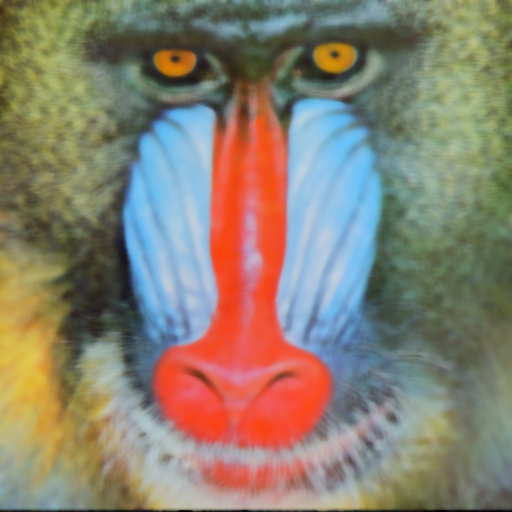}}
\put( 1.1,27.3){\colorbox{white}{\rule{0pt}{.6em}\hbox to.6em{\kern.1em\smash{e}}}}
\put(26.4,27.3){\colorbox{white}{\rule{0pt}{.6em}\hbox to.6em{\kern.1em\smash{f}}}}
\put(51.7,27.3){\colorbox{white}{\rule{0pt}{.6em}\hbox to.6em{\kern.1em\smash{g}}}}
\put(77.0,27.3){\colorbox{white}{\rule{0pt}{.6em}\hbox to.6em{\kern.1em\smash{h}}}}
\put( 0.0, 0.0){\includegraphics[width=24.1\unitlength]{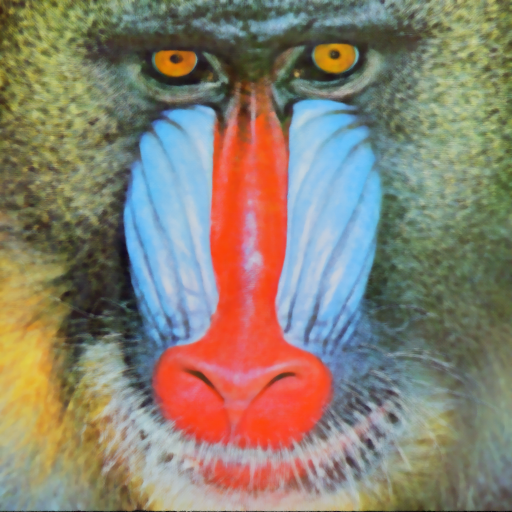}}
\put(25.3, 0.0){\includegraphics[width=24.1\unitlength]{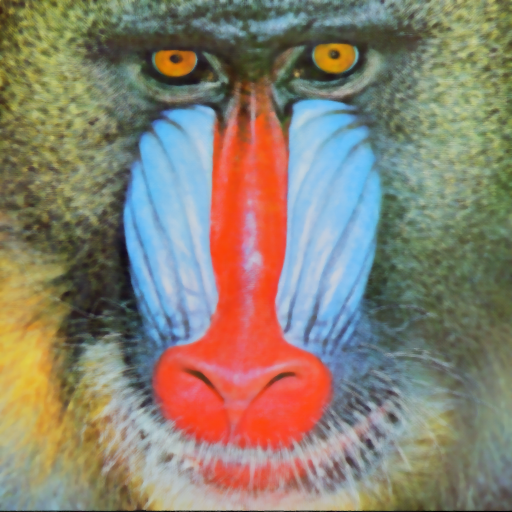}}
\put(50.6, 0.0){\includegraphics[width=24.1\unitlength]{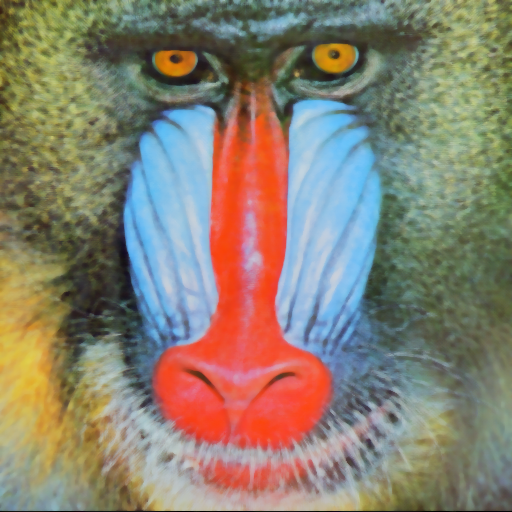}}
\put(75.9, 0.0){\includegraphics[width=24.1\unitlength]{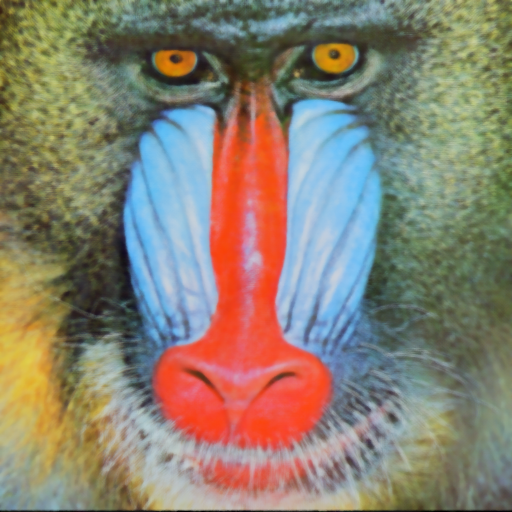}}
\put( 1.1, 2.0){\colorbox{white}{\rule{0pt}{.6em}\hbox to.6em{\kern.1em\smash{i}}}}
\put(26.4, 2.0){\colorbox{white}{\rule{0pt}{.6em}\hbox to.6em{\kern.1em\smash{j}}}}
\put(51.7, 2.0){\colorbox{white}{\rule{0pt}{.6em}\hbox to.6em{\kern.1em\smash{k}}}}
\put(77.0, 2.0){\colorbox{white}{\rule{0pt}{.6em}\hbox to.6em{\kern.1em\smash{l}}}}
\end{picture}
\caption{\label{fig-baboon}
Median filtering of the \emph{Baboon} test image 
using different multivariate medians and
disc-shaped structuring elements $D_\varrho$ of varying radius.
\textbf{Top row, left to right:}
\textbf{(a)} Original RGB image ($512\times512$ pixels). --
\textbf{(b)} Filtered using 2D Oja median with $\varrho=10$. --
\textbf{(c)} Filtered using $L^1$ median with $\varrho=10$. --
\textbf{(d)} Filtered using affine equivariant transformed
$L^1$ median with $\varrho=10$. --
\textbf{Middle row, left to right:}
\textbf{(e)} Filtered using 3D Oja median with $\varrho=5$. --
\textbf{(f)} Filtered using 2D Oja median with $\varrho=5$. --
\textbf{(g)} Filtered using $L^1$ median with $\varrho=5$. --
\textbf{(h)} Filtered using affine equivariant transformed
$L^1$ median with $\varrho=5$. --
\textbf{Bottom row, left to right:}
\textbf{(i)} Filtered using 3D Oja median with $\varrho=3$. --
\textbf{(j)} Filtered using 2D Oja median with $\varrho=3$. --
\textbf{(k)} Filtered using $L^1$ median with $\varrho=3$. --
\textbf{(l)} Filtered using affine equivariant transformed
$L^1$ median with $\varrho=3$.
}
\end{figure}
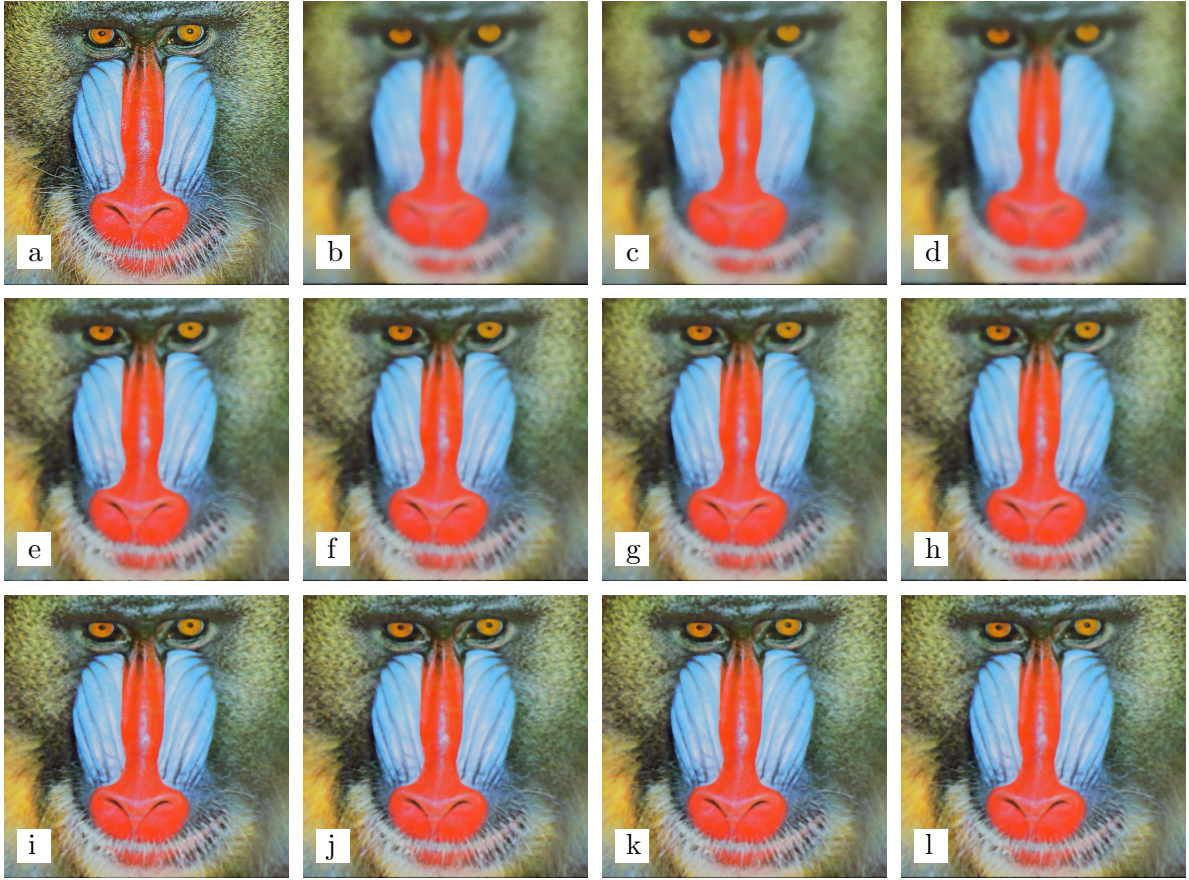

In this section, we consider the filtering of RGB colour image data.
The RGB colour space is used here for its simplicity. A comparison
with other colour spaces like HSV, HCL, YCbCr etc.\ is left to future
work, and will be important when evaluating the applicability of
multivariate median filters in, e.g., denoising applications.
It is worth noting, however, that common colour spaces are related
via differentiable transforms (with isolated singularities to be observed
in some cases). This means that \emph{locally} replacing one colour
space with another is just an affine transformation (given by the Jacobian
of the colour space transform). For affine equivariant median filters
applied in small neighbourhoods of smooth images, filtering results
can therefore be expected to be largely independent of the colour space
being used.

Application of the $L^1$ median to three-channel data is straightforward.
Regarding the Oja median filter, it is worth noting that 
a planar RGB image is a discretisation of a function
$\bm{u}: \mathbb{R}^2\supset\varOmega\to\mathbb{R}^3$, i.e.\ a parametrised
surface in $\mathbb{R}^3$.
The values of $\bm{u}$ (RGB triples) within a small patch of $\varOmega$,
such as the structuring element of a pixel, form a surface patch
in $\mathbb{R}^3$. 
For a noise-free image, the function $\bm{u}$ can be assumed to be smooth, 
resulting in almost planar surface patches.

One consequence of this is that the 3D Oja median applied to the
RGB triples from a structuring element will be the minimiser of a sum of
simplex volumes where virtually all of the simplices are almost degenerated.

On the other hand, the 2D Oja median, which minimises a sum of triangle
areas, can easily be applied to these data, which gives us a further option
for median filtering of planar RGB images that stands between the $L^1$
(thus, 1D Oja) and 3D Oja median,
\begin{equation}
\bm{m}_{\mathrm{Oja}(2,3)}(\mathcal{X}):=
\mathop{\operatorname{argmin}}\limits_{\bm{x}\in\bbbr^3}
\sum_{1\le i<j\le N}\!\!
\lvert[\bm{x},\bm{x}_i,\bm{x}_j]\rvert \;.
\label{mOja23}
\end{equation}
Of course, the 2D Oja median $\bm{m}_{\mathrm{Oja}(2,3)}$ for general 3D 
data is not equivariant
under affine transformations of $\mathbb{R}^3$. However, the 2D Oja median
of co-planar data from $\mathbb{R}^3$ is affine equivariant even
with respect to affine transformations of $\mathbb{R}^3$. Since the
RGB triples being filtered are almost co-planar, it can be expected
that a 2D Oja median filter for planar RGB images will display a good 
approximation to affine equivariance.
We include therefore in our experiments four filters based on
the standard $L^1$ median \eqref{mL1}, the 2D \eqref{mOja23} and 
3D \eqref{mOja} Oja median, and
the affine equivariant transformed $L^1$ median \eqref{mL1a}.

Figure~\ref{fig-baboon} shows results of filtering of an RGB image with
these three filters with disc-shaped structuring elements $D_\varrho$
of different size. (The combination of the full 3D Oja median filter
with a structuring element of radius $\varrho=10$ is beyond computational
possibilities with our simple algorithm and therefore omitted.)

Results indicate that the four median filter variants again give
very similar results. As the size of structuring elements increases,
the behaviour known from univariate median filters is observed: 
Small image details
are progressively smoothed out, whereas strong edges between homogeneous
regions are kept sharp even for larger structuring elements.

\begin{figure}[t!]
\unitlength0.0048\textwidth
\begin{picture}(208,126)
\put(  0.0,94){\includegraphics[width=32\unitlength]{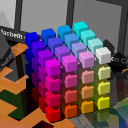}}
\put( 34.0,94){\includegraphics[width=32\unitlength]{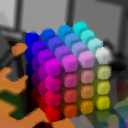}}
\put( 68.0,94){\includegraphics[width=32\unitlength]{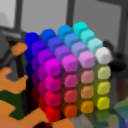}}
\put(  1.1,96){\colorbox{white}{\rule{0pt}{.6em}\hbox to.6em{\kern.1em\smash{a}}}}
\put( 35.1,96){\colorbox{white}{\rule{0pt}{.6em}\hbox to.6em{\kern.1em\smash{b}}}}
\put( 69.1,96){\colorbox{white}{\rule{0pt}{.6em}\hbox to.6em{\kern.1em\smash{c}}}}
\put(  0.0,60){\includegraphics[width=32\unitlength]{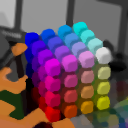}}
\put( 34.0,60){\includegraphics[width=32\unitlength]{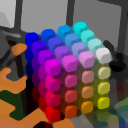}}
\put( 68.0,60){\includegraphics[width=32\unitlength]{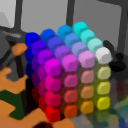}}
\put(  1.1,62){\colorbox{white}{\rule{0pt}{.6em}\hbox to.6em{\kern.1em\smash{d}}}}
\put( 35.1,62){\colorbox{white}{\rule{0pt}{.6em}\hbox to.6em{\kern.1em\smash{e}}}}
\put( 69.1,62){\colorbox{white}{\rule{0pt}{.6em}\hbox to.6em{\kern.1em\smash{f}}}}
\put(  0.0,26){\includegraphics[width=32\unitlength]{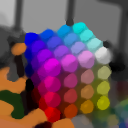}}
\put( 34.0,26){\includegraphics[width=32\unitlength]{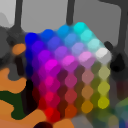}}
\put( 68.0,26){\includegraphics[width=32\unitlength]{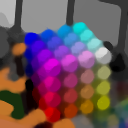}}
\put(  1.1,28){\colorbox{white}{\rule{0pt}{.6em}\hbox to.6em{\kern.1em\smash{g}}}}
\put( 35.1,28){\colorbox{white}{\rule{0pt}{.6em}\hbox to.6em{\kern.1em\smash{h}}}}
\put( 69.1,28){\colorbox{white}{\rule{0pt}{.6em}\hbox to.6em{\kern.1em\smash{i}}}}
\put(108.0,18){\parbox[b]{100\unitlength}{%
\caption{\label{fig-co01}
Median filtering of the \emph{Colors} test image
using different multivariate medians and
disc-shaped structuring elements $D_\varrho$ of varying radius.
The test image is a clipping from the image 
``Delta E'' from \texttt{http://brucelindbloom.com/}
(\copyright 2001--2015 Bruce Justin Lindbloom; 
research and non-commercial use permitted.)
\textbf{Top row, left to right:}
\textbf{(a)} Original RGB image ($128\times128$ pixels). --
\textbf{(b)} Filtered using 2D Oja median with $\varrho=3$, 
without regularisation. --
\textbf{(c)} Filtered using 3D Oja median with $\varrho=3$,
with input regularisation. --
\textbf{Middle row, left to right:}
\textbf{(d)} Filtered using 2D Oja median with $\varrho=3$,
with input regularisation. --
\textbf{(e)} Filtered using $L^1$~median with $\varrho=3$.~~--~~
\textbf{(f)} Filtered using 
}}}
\put(0,0){\parbox[b]{208\unitlength}{%
affine equivariant transformed
$L^1$ median with $\varrho=3$. --
\textbf{Bottom row, left to right:}
\textbf{(g)} Filtered using 2D Oja median with $\varrho=5$,
with input regularisation. --
\textbf{(h)} Filtered using $L^1$ median with $\varrho=5$. --
\textbf{(i)} Filtered using affine equivariant transformed
$L^1$ median with $\varrho=5$.
}}
\end{picture}
\end{figure}
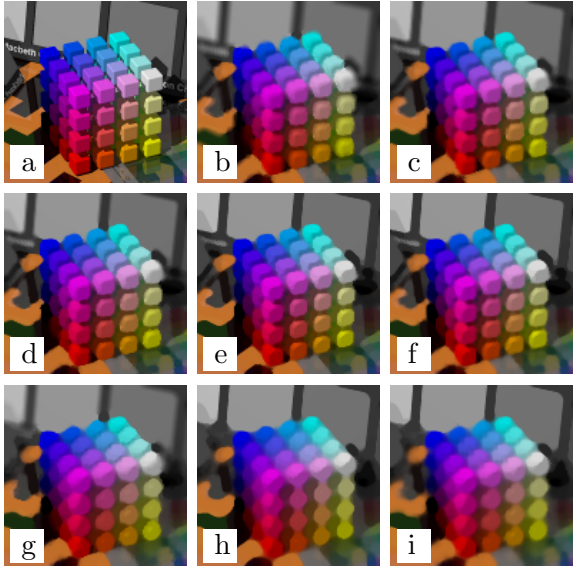

Whereas the test image used in Figure~\ref{fig-baboon} contains many 
fine-scale structures everywhere in the image, we consider in our
next experiment, Figure~\ref{fig-co01}, a test image which is 
dominated by smooth regions, some
even with constant colour values, separated by sharp boundaries. This image,
shown in Figure~\ref{fig-co01}(a),
is almost perfectly noise-free apart from the quantisation noise.
Therefore, the RGB triples found within a structuring element are often 
strictly
collinear such that the degeneracy of the Oja median energies becomes an
issue in computation. This is demonstrated in Figure~\ref{fig-co01}(b)
by the result of (2D) Oja median filtering without regularisation. Note
that most edges are substantially blurred. However, at some junctions
where values from sufficiently many regions within a small neighbourhood
create input data sets of sufficient dimensionality, edges stay sharp.
For the further 2D and 3D Oja median filtering in this experiment series,
we use therefore the input regularisation as described in 
Section~\ref{ssec-demo-num}, consisting of replacing points with 
quadruples of simplex corners and subsequent normalisation by principal
axis transform.
Even with this proceeding, a slight blur remains visible in
the Oja results, Figure~\ref{fig-co01}(c), (d) and (g), especially
for the larger structuring element (g).
Apart from this, the results of 3D (c) and 2D Oja
median filtering (d, g) as well as those of standard (e, h) and 
affine equivariant transformed $L^1$ median filtering (f, i) are again
largely comparable. They show the structure simplification and rounding 
of contours known from univariate median filters, whereas edges
are kept reasonably sharp.

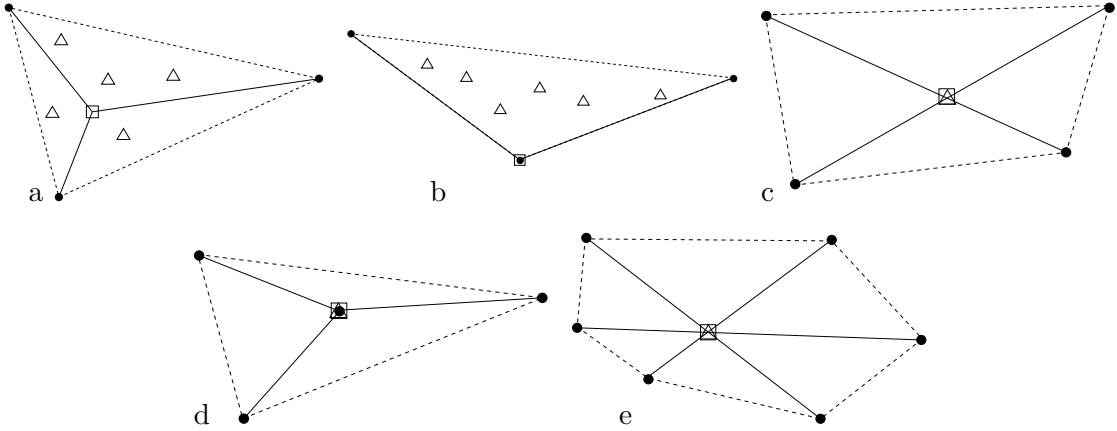
\begin{figure}[t]
\unitlength0.01\textwidth
\begin{picture}(100,36)
\put(2,19){\includegraphics[width=27\unitlength]{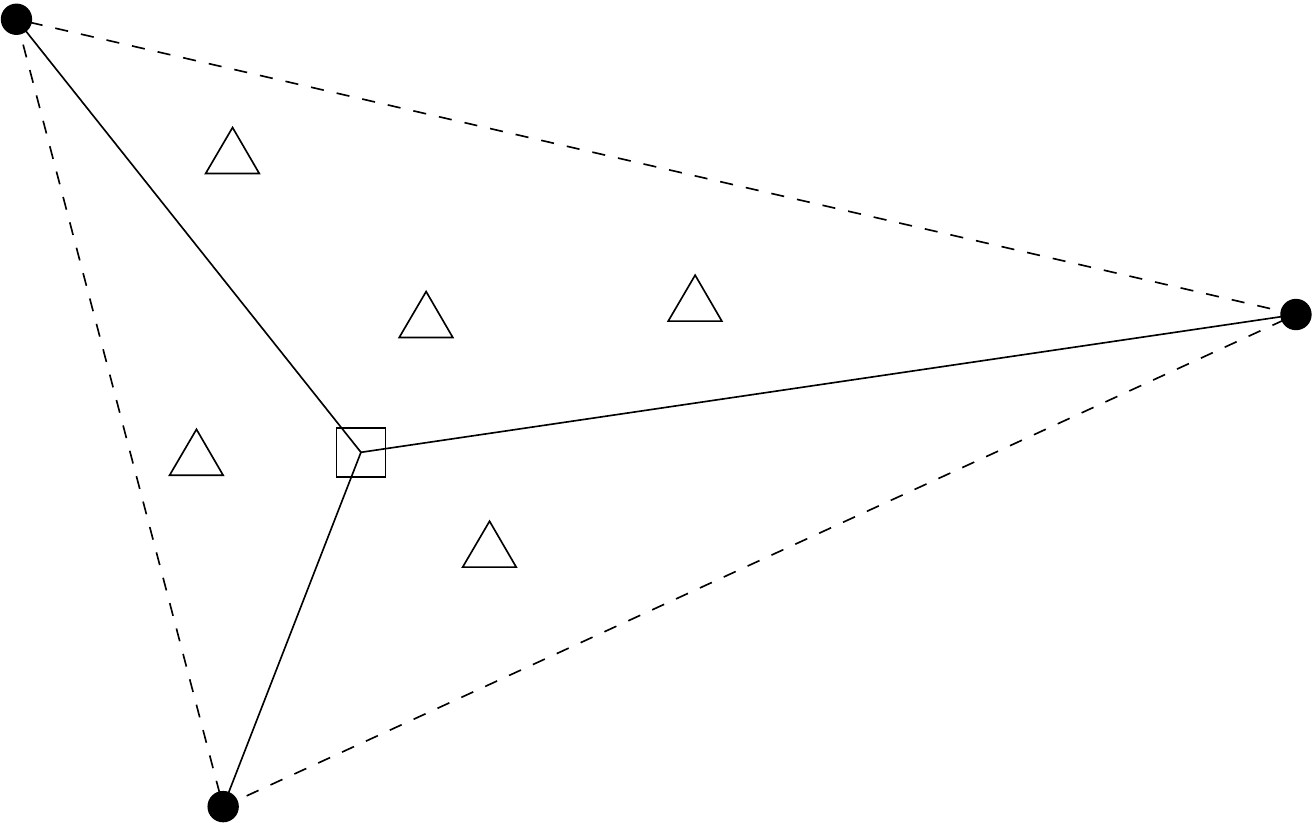}}
\put(31,22){\includegraphics[width=33\unitlength]{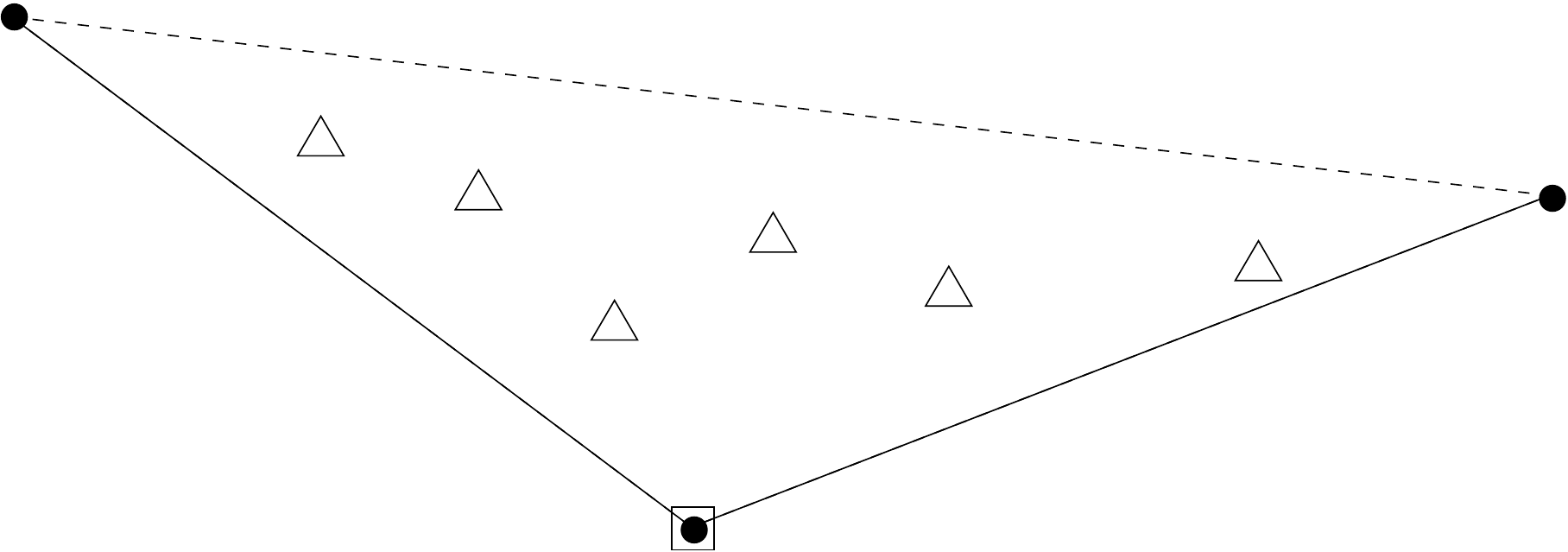}}
\put(66,20){\includegraphics[width=30\unitlength]{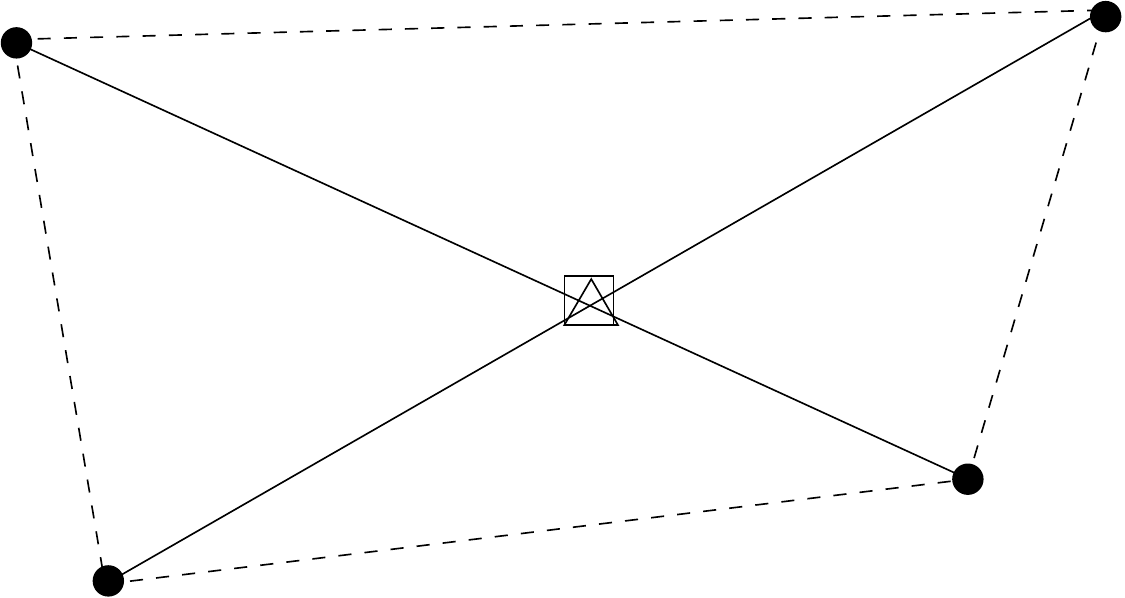}}
\put(18,0){\includegraphics[width=30\unitlength]{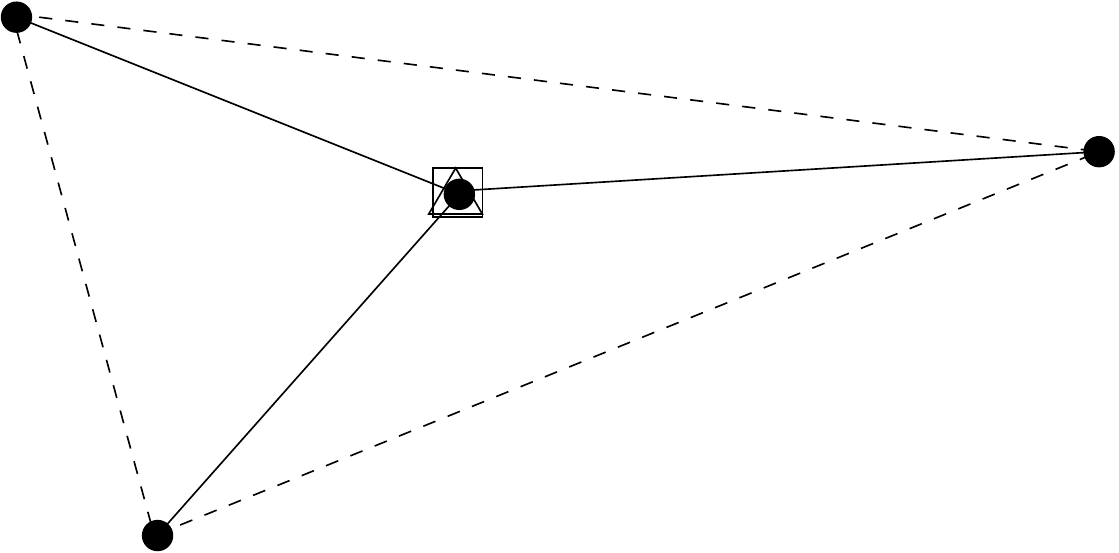}}
\put(50,0){\includegraphics[width=30\unitlength]{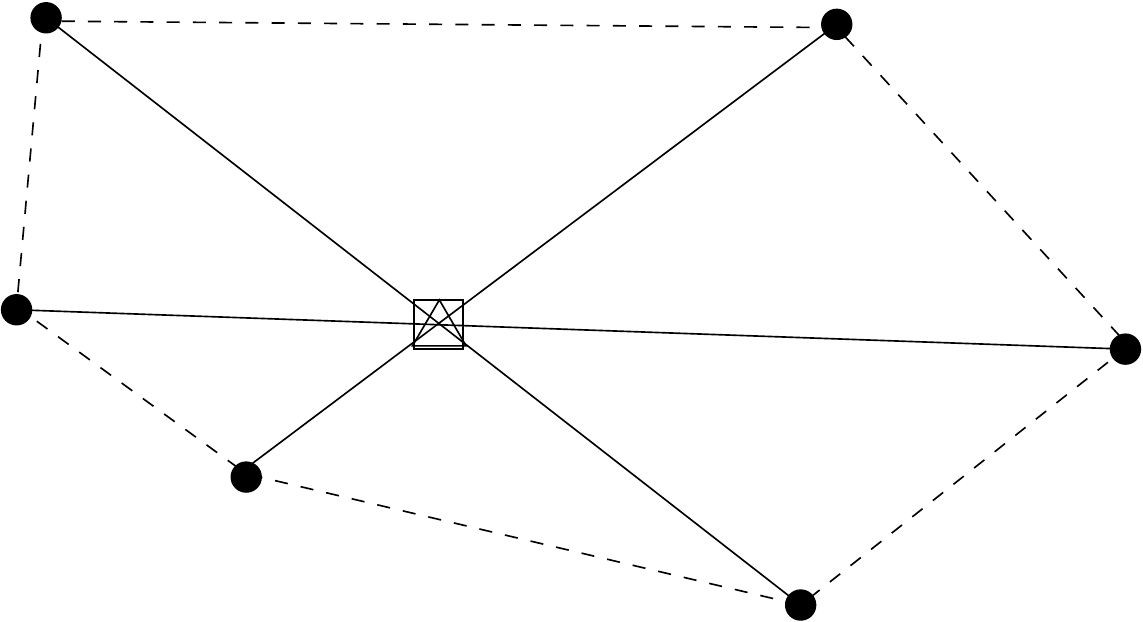}}
\put(4,19){a}
\put(38,19){b}
\put(66,19){c}
\put(18,0){d}
\put(54,0){e}
\end{picture}
\caption{\label{fig-geometry}
Simple configurations of input data points (solid points) with
their $L^1$ medians (squares) and Oja medians (triangles).
\textbf{(a)} Three points forming a triangle with all interior
angles less than $120$ degrees: The $L^1$ median is the
\emph{Steiner point}; any point within the triangle is an Oja
median. --
\textbf{(b)} Three points forming a triangle with an obtuse
angle of $120$ degrees or more: The obtuse corner is the $L^1$
median; still, all points within the triangle are Oja medians. --
\textbf{(c)} Four points forming a convex quadrangle:
the $L^1$ and Oja median coincide at the intersection of the
diagonals. --
\textbf{(d)} Four points whose convex hull is a triangle:
the $L^1$ and Oja median coincide at
the data point that is not a corner of the convex hull. --
\textbf{(e)} $2n$ points that form a convex $2n$-gon (hexagon
shown as example) in which all diagonals between opposing points
have a common intersection point: its $L^1$ and Oja median coincide
at this intersection point. --
From \cite{Welk-ssvm15}.
}
\end{figure}

\subsection{Geometric Facts about Bivariate $L^1$ and Oja Median}

To add some geometric intuition about the $L^1$ and Oja medians,
we consider small point sets in the plane and their medians.
The following statements can easily be inferred from standard elementary
geometry arguments such as the triangle
inequality (for the $L^1$ median) and multiplicities of covering of
the convex hull of input points by the triangles with input and median
points as corners (for the 2D Oja median).

In all cases, the $L^1$ and Oja medians 
will be located within the convex hull of the input data set
(if, in the case of the Oja median, this set is not collinear)
due to the convexity of the objective functions being minimised.

\begin{enumerate}
\item
For two points, the $L^1$ median criterion is fulfilled equally for
all points of their connecting line segment. The Oja median criterion
is even fulfilled by all points of the straight line through these
points since the Oja median definition degenerates for collinear
sets of points.
\item
For three points, the $L^1$ median depends on the sort of triangle
they span. If all of its interior angles are smaller than $120$
degrees, see Figure~\ref{fig-geometry}(a), the sum of distances
to the corners is minimised by a unique point known as \emph{Steiner point}
or \emph{Fermat-Torricelli point,} from which all sides of the
triangle are seen under $120$ degree angles. For a triangle with
an obtuse corner of at least $120$ degrees, this corner is the
$L^1$ median, see Figure~\ref{fig-geometry}(b).

In contrast, the
Oja median criterion is met in both cases by all points of the
triangle. This is consistent with the affine equivariance of
the Oja median that does not discriminate triangles by shape.
Besides, this configuration nicely illustrates how
in the Oja median definition 
simplices take the role of line segments from the
univariate median definition: the three-point case of the bivariate
Oja median is just the analogue of the two-point case of the univariate
median.
\item
For four points, $L^1$ and Oja median always coincide: If the
convex hull of the data points is a triangle, then the data point
that is not a corner of the convex hull is the median, see
Figure~\ref{fig-geometry}(d); if it is
a convex quadrangle, then the intersection point of its diagonals
is the median, see Figure~\ref{fig-geometry}(c).
\item
The coincidence between $L^1$ and Oja median continues also in
some configurations of more data points. A (non-generic) example
is shown in Figure~\ref{fig-geometry}(e): A convex $2n$-gon
in which all the diagonals that bisect the point set (i.e.\ those
that span $n$ sides) have a common intersection point, features this
point as $L^1$ and Oja median.
\end{enumerate}

We point out two facts that can be learned from these simple
configurations.
Firstly, bivariate medians, unlike their univariate counterpart,
cannot always be chosen from the input data set, but they happen
to be input data points in some generic configurations. Only in cases
when none of the input points lies sufficiently ``in the middle''
of the data, a new point is created. Secondly, despite their
different definitions, the $L^1$ and Oja median coincide in some
generic situations, or are not far apart from each other. This
adds plausibility to why the image filtering results in
Figure~\ref{fig-baboon} and Figure~\ref{fig-co01} are that similar.

We conclude this section by mentioning a result from
\cite{Barbara-MG01}: 
For a non-collinear point set in the plane, there exists always an Oja median
(i.e.\ a point minimising the relevant objective function)
that is the intersection of two lines, each of which goes through
two data points. Restricting the search for minimisers to the
finite set of these intersection points is one of the ingredients
in the efficient planar Oja median algorithm from \cite{Aloupis-CG03}.

\section{Asymptotic Analysis of Multivariate Median Filtering}\label{sec-pde}

The reformulation of a local image filter to a space-continuous
setting is straightforward. The main modification is that the
set of values that results from the selection step and is
processed in the aggregation step is now infinite and
equipped with a density. This density is induced from the uniform
distribution of function arguments in the structuring element in
the image domain via the Jacobian of the image function.

As proven in \cite{Guichard-sana97},
a univariate median filtering step of an image with disc-shaped
structuring element of radius $\varrho$ approximates for
$\varrho\to0$ a time step of size $\tau=\varrho^2/6$ of an explicit
scheme for the mean curvature motion PDE. 
In this section, we will derive PDEs that
are approximated in the same sense by 
multivariate median filters based on $L^1$ and Oja medians. 
We will consider images with two or three channels
over two- and three-dimensional domains.

Throughout this paper, the structuring element will be
a disc $D_\varrho$ of radius 
$\varrho$
for planar images, or a ball $B_\varrho$ of radius $\varrho$ for
volume images.

\subsection{Bivariate Planar Images}\label{ssec-pde22}

We start by considering the case of two-channel images 
over a planar domain $\varOmega$, 
as already studied in
\cite{Welk-ssvm15}.

\subsubsection{$L^1$ Median}\label{sssec-l1pde22}

In \cite{Welk-ssvm15} the result from \cite{Welk-Aiep14}
concerning the $L^1$ multivariate median filter for images
$\bm{u}:\bbbr^2\supset\varOmega\to\bbbr^n$ was simplified
to the bivariate case $n=2$. As the result from \cite{Welk-Aiep14}
needs to be corrected as stated in the Introduction,
the statement from \cite{Welk-ssvm15} is modified
as follows.

\begin{proposition}[from \cite{Welk-ssvm15}, corrected]
\label{prop-l1pde22}
Let a bivariate image
$\bm{u}:\bbbr^2\supset\varOmega\to\bbbr^2$, $(x,y)\mapsto(u,v)$,
be given. One step of $L^1$ median filtering
with the structuring element $D_\varrho$
approximates for $\varrho\to0$ an explicit time step of size
$\tau=\varrho^2/6$ of the PDE system
\begin{align}
\begin{pmatrix}u_t\\v_t\end{pmatrix} &=
\bm{S}(\mathrm{D}\bm{u}) 
\begin{pmatrix}u_{\bm{\eta}\bm{\eta}}\\v_{\bm{\eta}\bm{\eta}}\end{pmatrix}
+\bm{T}(\mathrm{D}\bm{u})
\begin{pmatrix}u_{\bm{\xi}\bm{\xi}}\\v_{\bm{\xi}\bm{\xi}}\end{pmatrix}
-2\,\bm{W}(\mathrm{D}\bm{u})
\begin{pmatrix}u_{\bm{\xi}\bm{\eta}}\\v_{\bm{\xi}\bm{\eta}}\end{pmatrix}
\label{l1pde22}
\end{align}
where $\bm{\eta}$ is the major, and $\bm{\xi}$ the minor eigenvector of the
structure tensor
$\bm{J}:=\bm{J}(\mathrm{D}\bm{u}):=
\bm{\nabla}u\bm{\nabla}u\transpose+\bm{\nabla}v\bm{\nabla}v\transpose=
\mathrm{D}\bm{u}\transpose\mathrm{D}\bm{u}$.
The coefficient matrices
$\bm{S}(\mathrm{D}\bm{u})$, $\bm{T}(\mathrm{D}\bm{u})$
and $\bm{W}(\mathrm{D}\bm{u})$
are given by
\begin{align}
\bm{S}(\mathrm{D}\bm{u})
&:= \bm{R}\,
\mathrm{diag}\left(
Q_1{\left(\frac{\lvert\partial_{\bm{\eta}}\bm{u}\rvert}
{\lvert\partial_{\bm{\xi}}\bm{u}\rvert}\right)},
Q_2{\left(\frac{\lvert\partial_{\bm{\eta}}\bm{u}\rvert}
{\lvert\partial_{\bm{\xi}}\bm{u}\rvert}\right)}
\right)
\bm{R}\transpose\;,
\label{l1pde22-S}
\\
\bm{T}(\mathrm{D}\bm{u})
&:= \bm{R}\,
\mathrm{diag}\left(
Q_2{\left(\frac{\lvert\partial_{\bm{\xi}}\bm{u}\rvert}
{\lvert\partial_{\bm{\eta}}\bm{u}\rvert}\right)},
Q_1{\left(\frac{\lvert\partial_{\bm{\xi}}\bm{u}\rvert}
{\lvert\partial_{\bm{\eta}}\bm{u}\rvert}\right)}\right)
\bm{R}\transpose\;,
\label{l1pde22-T}
\\
\bm{W}(\mathrm{D}\bm{u})
&:= \bm{R}\,
\begin{pmatrix}
0&
\frac{\lvert\partial_{\bm{\eta}}\bm{u}\rvert}
{\lvert\partial_{\bm{\xi}}\bm{u}\rvert}\,
Q_1{\left(\frac{\lvert\partial_{\bm{\eta}}\bm{u}\rvert}
{\lvert\partial_{\bm{\xi}}\bm{u}\rvert}\right)}
\\
\frac{\lvert\partial_{\bm{\xi}}\bm{u}\rvert}
{\lvert\partial_{\bm{\eta}}\bm{u}\rvert}\,
Q_1{\left(\frac{\lvert\partial_{\bm{\xi}}\bm{u}\rvert}
{\lvert\partial_{\bm{\eta}}\bm{u}\rvert}\right)}
&
0
\end{pmatrix}
\bm{R}\transpose\;,
\label{l1pde22-W}
\end{align}
where
$\bm{R}=
(\mathrm{D}\bm{u}^{-1})^{\mathrm{T}}\,
\bm{P}\,
\mathrm{diag}(
\lvert\partial_{\bm{\eta}}\bm{u}\rvert,
\lvert\partial_{\bm{\xi}}\bm{u}\rvert)
$
is a rotation matrix that depends on
the Jacobian $\mathrm{D}\bm{u}$ of $\bm{u}$
and
the eigenvector matrix $\bm{P}=\bigl(\bm{\eta}~|~\bm{\xi}\bigr)$
of $\bm{J}$.
The functions $Q_1,Q_2:[0,\infty]\to\bbbr$
are given by the quotients of elliptic integrals
\begin{align}
Q_1(\lambda) &= \frac{3 \iint_{D_1(\bm{0})} 
s^2t^2/(s^2+\lambda^2t^2)^{3/2}\,\mathrm{d}s\,\mathrm{d}t}
{\iint_{D_1(\bm{0})}
s^2/(s^2+\lambda^2t^2)^{3/2}\,\mathrm{d}s\,\mathrm{d}t}\;,\\
Q_2(\lambda) &= \frac{3 \iint_{D_1(\bm{0})} 
t^4/(s^2+\lambda^2t^2)^{3/2}\,\mathrm{d}s\,\mathrm{d}t}
{\iint_{D_1(\bm{0})}
t^2/(s^2+\lambda^2t^2)^{3/2}\,\mathrm{d}s\,\mathrm{d}t}
\end{align}
for $\lambda\in(0,\infty)$, together with the limits $Q_1(0)=Q_2(0)=1$,
$Q_1(\infty)=Q_2(\infty)=0$.
\end{proposition}

\begin{remark}
The vectors $\bm{\eta}$ and $\bm{\xi}$ used in
\eqref{l1pde22}--\eqref{l1pde22-W} are the directions of
greatest and least change of the bivariate function $\bm{u}$,
thus the closest analoga to gradient and level line
directions of univariate images, see \cite{Chung-SPL00}.
The use of these image-adaptive local coordinates characterises
\eqref{l1pde22} as a curvature-based PDE remotely similar to
the (mean) curvature motion PDE approximated by univariate
median filtering.
\end{remark}

The proof of the proposition
relies on the following statement which is corrected
from \cite{Welk-Aiep14}.

\begin{lemma}[from \cite{Welk-Aiep14}, corrected]\label{lem-l1pde22-aligned}
Let $\bm{u}$ be given as in Proposition~\ref{prop-l1pde22}, and the origin
$\bm{0}=(0,0)$ be an inner point of $\varOmega$. Assume that
the Jacobian $\mathrm{D}\bm{u}(\bm{0})$ is diagonal,
i.e.\ $u_y=v_x=0$, and $u_x\ge v_y>0$.
Then one step of $L^1$ median filtering with the structuring element
$D_{\varrho}$ at $\bm{0}$ approximates for $\varrho\to0$ an explicit time
step of size $\tau=\varrho^2/6$ of the PDE system
\begin{align}
u_t &= 
Q_1{\left(\frac{u_x}{v_y}\right)} u_{xx} 
+ Q_2{\left(\frac{v_y}{u_x}\right)} u_{yy}
- \frac{2u_x}{v_y}\,Q_1{\left(\frac{u_x}{v_y}\right)} v_{xy}
\label{l1pde22-aligned-u}
\;,\\
v_t &= 
Q_2{\left(\frac{u_x}{v_y}\right)} v_{xx} 
+ Q_1{\left(\frac{v_y}{u_x}\right)} v_{yy} 
- \frac{2v_y}{u_x}\,Q_1{\left(\frac{v_y}{u_x}\right)} u_{xy}
\label{l1pde22-aligned-v}
\;,
\end{align}
with the coefficient functions $Q_1$, $Q_2$ as stated in
Proposition~\ref{prop-l1pde22}.
\end{lemma}

\begin{remark}\label{rm-q2q1i}
The elliptic integrals in the
coefficient expressions $Q_1(\lambda)$ and $Q_2(\lambda)$
can in general not be evaluated in closed form. However, they
are connected by 
\begin{equation}
Q_2(\lambda)=1-Q_1(\lambda^{-1})\;.
\label{l1-22-q1q2}
\end{equation}
\end{remark}

\begin{remark}\label{rm-l1pde22-I}
In the case $u_x=1$, $v_y=1$, the coefficients of \eqref{l1pde22-aligned-u},
\eqref{l1pde22-aligned-v} simplify via $Q_1(1)=1/4$, $Q_2(1)=3/4$ such that
one obtains
\begin{align}
u_t &= \tfrac14u_{xx} + \tfrac34u_{yy} - \tfrac12v_{xy}\;,
\label{l1pde22-I-1}
\\
v_t &= \tfrac34v_{xx} + \tfrac14v_{yy} - \tfrac12u_{xy}\;.
\label{l1pde22-I-2}
\end{align}
\end{remark}

\begin{remark}\label{rm-l1pde22-infty}
Note that for $\lambda\to\infty$, $\lambda\,Q_1(\lambda)$ goes to zero such
that the coefficients for $v_{xy}$ in \eqref{l1pde22-aligned-u}
and for $u_{xy}$ in \eqref{l1pde22-aligned-v} are globally bounded
for arbitrary $u_x$, $v_y$, and in the limit case $v_y=0$ one has the
decoupled PDEs $u_t=u_{yy}$, $v_t=v_{xx}$.
\end{remark}

\begin{remark}\label{rm-guichardmorellimit}
Univariate median filtering is contained in the statement of
Lemma~\ref{lem-l1pde22-aligned} when $v_y$ is sent to $0$. In this case,
the first PDE \eqref{l1pde22-aligned-u} becomes $u_t=u_{yy}$ by virtue
of $Q_1(\infty)=0$, $Q_2(0)=1$, and the previous remark.
This translates to $u_t=u_{\bm{\xi\xi}}$ in the general setting of
Proposition~\ref{prop-l1pde22}, i.e.\ the (mean) curvature motion
equation, thus reproducing exactly the result of \cite{Guichard-sana97}.
\end{remark}

\begin{proof}[Proof of Proposition~\ref{prop-l1pde22}]
Consider an arbitrary fixed location $(x^*,y^*)$.
By applying rotations with $\bm{P}$ in the $x$-$y$ plane and
with $\bm{R}$ in the $u$-$v$ plane,
$x$, $y$ can be aligned with the
(orthogonal) major and minor
eigenvector directions $\bm{\eta}$ and $\bm{\xi}$
of the structure tensor 
$\bm{J}(\bm{\nabla}u,\bm{\nabla}v)$
at $(x^*,y^*)$, 
and $u$, $v$ with the corresponding derivatives
$\partial_{\bm{\eta}}\bm{u}$, $\partial_{\bm{\xi}}\bm{u}$.
Then Lemma~\ref{lem-l1pde22-aligned} can be applied.
Reverting the rotations in the $x$-$y$ and $u$-$v$ planes,
the PDE system \eqref{l1pde22-aligned-u}--\eqref{l1pde22-aligned-v}
 turns into the system
\eqref{l1pde22}--\eqref{l1pde22-W} of the proposition.
\end{proof}

\begin{remark}
Equivariance of the PDE \eqref{l1pde22}
with regard to Euclidean transformations of the $u$-$v$ plane 
follows immediately from its derivation for a special
case and transfer to the general case by a Euclidean
transformation.
\end{remark}

\subsubsection{Oja Median}\label{sssec-ojapde22}

Next we turn to the Oja median, which in the bivariate case under 
consideration is defined as the minimiser of the total area of 
triangles each formed by two given data points and the median point.
The following result was proven in \cite{Welk-ssvm15}.

\begin{theorem}[from \cite{Welk-ssvm15}]\label{thm-ojapde22}
Let a bivariate image
$\bm{u}:\bbbr^2\supset\varOmega\to\bbbr^2$, $(x,y)\mapsto(u,v)$,
be given. At any location where $\mathrm{det}\,\mathrm{D}\bm{u}\ne0$,
one step of Oja median filtering of $\bm{u}$
with the structuring element $D_\varrho$
approximates for $\varrho\to0$ an explicit time step of size
$\tau=\varrho^2/24$ of the PDE system
\begin{align}
\begin{pmatrix}u_t\\v_t\end{pmatrix} =
2\begin{pmatrix}u_{xx}\!+\!u_{yy}\\v_{xx}\!+\!v_{yy}\end{pmatrix}
&-
\bm{A}(\mathrm{D}\bm{u})
\begin{pmatrix}u_{xx}\!-\!u_{yy}\\v_{yy}\!-\!v_{xx}\end{pmatrix}
-
\bm{B}(\mathrm{D}\bm{u})
\begin{pmatrix}u_{xy}\\v_{xy}\end{pmatrix}
\label{ojapde22}
\end{align}
with the coefficient matrices
\begin{align}
\bm{A}(\mathrm{D}\bm{u})
&:=
\frac{1}{u_xv_y-u_yv_x} 
\begin{pmatrix}u_xv_y+u_yv_x&2u_xu_y\\2v_xv_y&u_xv_y+u_yv_x\end{pmatrix}
\;,
\label{ojapde22-A}
\\
\bm{B}(\mathrm{D}\bm{u})
&:=
\frac{2}{u_xv_y-u_yv_x} 
\begin{pmatrix}-u_xv_x+u_yv_y&u_x^2-u_y^2\\-v_x^2+v_y^2&u_xv_x-u_yv_y
\end{pmatrix}
\;.
\label{ojapde22-B}
\end{align}
\end{theorem}

The proof of this theorem relies on the following lemma.

\begin{lemma}[from \cite{Welk-ssvm15}]\label{lem-ojapde22-I}
Let $\bm{u}$ be given as in Theorem~\ref{thm-ojapde22}, and 
$\bm{0}=(0,0)$ be an inner point of $\varOmega$.
Assume that $\mathrm{D}\bm{u}(\bm{0})$ is the $2\times2$ unit matrix $\bm{I}$.
At $\bm{x}=\bm{0}$, one step of Oja median filtering of $\bm{u}$
with the structuring element $D_\varrho$ then
approximates for $\varrho\to0$ an explicit time step of size
$\tau=\varrho^2/24$ of the PDE system
\begin{align}
u_t &= u_{xx}+3u_{yy}-2v_{xy}\;, \label{ojapde22-I1} \\
v_t &= 3v_{xx}+v_{yy}-2u_{xy}\;. \label{ojapde22-I2}
\end{align}
\end{lemma}

\begin{remark}\label{rm-l1oja22-I}
Note that the PDE system \eqref{ojapde22-I1}, \eqref{ojapde22-I2}
coincides exactly with \eqref{l1pde22-I-1}, \eqref{l1pde22-I-2},
the $L^1$ result for the same case $\mathrm{D}\bm{u}(\bm{0})=\bm{I}$,
except for a rescaling of the time $t$ by a factor $4$ 
in compensation for the different time step size $\varrho^2/24$ in
Lemma~\ref{lem-ojapde22-I} opposed to $\varrho^2/6$ in 
Lemma~\ref{lem-l1pde22-aligned}.
\end{remark}

This lemma is proven in the appendix in two slightly different ways. 
The first proof, in Appendix~\ref{app-22proof1}, goes back to
\cite{Welk-ssvm15} and is presented here in slightly more detail.
The new proof in Appendix~\ref{app-22proof2} is more straightforward.
The reason why the first proof is also kept in this paper is that
it is the blueprint for subsequent proofs in this paper, whereas the
approach of the second proof would be more cumbersome to extend to
these cases.

Both proofs start from a Taylor expansion of $(u,v)\transpose$
within the structuring element, and express the gradient of the objective
function minimised by the Oja median in terms of the Taylor coefficients.
The median value is found as the point in the $u$-$v$ plane
for which this gradient vanishes. In both cases, the gradient itself
is linearised w.r.t.\ the Taylor coefficients.

In the first proof, Appendix~\ref{app-22proof1},
the calculation of the gradient is organised by
integration over directions in the $u$-$v$ plane, and the influences of
the individual Taylor coefficients are estimated separately by integrals
over the respective deformed structuring elements.

In contrast, the second proof in Appendix~\ref{app-22proof2}
calculates the gradient by
integration in the $x$-$y$ plane. The main idea here is to find
for each point $(x_1,y_1)$ a splitting of the 
structuring element into two regions: one region contains all points
$(x_2,y_2)$ for which the median candidate point and the images of
$(x_1,y_1)$ and $(x_2,y_2)$ form in this order a positively oriented
triangle in the $u$-$v$ plane whereas for $(x_2,y_2)$ in the other region
this triangle has negative orientation.
This approach allows to calculate the entire gradient with its dependencies
on all Taylor coefficients at once.

\begin{proof}[Proof of Theorem~\ref{thm-ojapde22}]
To prove the Theorem,
we consider the median of the values $\bm{u}(x,y)$ within the
Euclidean $\varrho$-neigh\-bour\-hood of $(0,0)$, and assume
that the Jacobian $\bm{D}:=\mathrm{D}\bm{u}(\bm{0})$ is regular as requested
by the hypothesis of the Theorem.

We transform the $u$-$v$ plane to variables 
$\hat{\bm{u}}$ via
\begin{equation}
\hat{\bm{u}} = \bm{D}^{-1}\bm{u} \;.
\label{ojapde22-Itfm}
\end{equation}
The affine equivariance of Oja's simplex median ensures that
also the median $\bm{u}^*$ of the values $\bm{u}$ within the structuring
element follows this transform.
The transformed data $\hat{\bm{u}}$ satisfy the hypothesis
$\mathrm{D}\hat{\bm{u}}(\bm{0})=\bm{I}$ of Lemma~\ref{lem-ojapde22-I},
thus the median filtering step for these values approximates the PDE
system \eqref{ojapde22-I1}, \eqref{ojapde22-I2}.

We transfer the result to the general situation of the theorem
by the inverse transform of \eqref{ojapde22-Itfm}.
Rewriting \eqref{ojapde22-I1}, \eqref{ojapde22-I2} as
\begin{align}
\hat{\bm{u}}_t 
&= 
\begin{pmatrix}
\hat{u}_{xx}+3\hat{u}_{yy}-2\hat{v}_{xy}\\
3\hat{v}_{xx}+\hat{v}_{yy}-2\hat{u}_{xy}
\end{pmatrix}
\notag\\*
&= 
2(\hat{\bm{u}}_{xx}+\hat{\bm{u}}_{yy}) 
+ \begin{pmatrix}1&0\\0&-1\end{pmatrix}
(\hat{\bm{u}}_{yy}-\hat{\bm{u}}_{xx})
-2 \begin{pmatrix}0&1\\1&0\end{pmatrix}\hat{\bm{u}}_{xy}
\notag\\*
&= 
2\,\bm{D}^{-1}
(\bm{u}_{xx}+\bm{u}_{yy})
+ \begin{pmatrix}1&0\\0&-1\end{pmatrix}
\bm{D}^{-1}
(\bm{u}_{yy}-\bm{u}_{xx})
-2 \begin{pmatrix}0&1\\1&0\end{pmatrix}
\bm{D}^{-1}
\bm{u}_{xy}
\end{align}
we obtain
\begin{align}
\bm{u}_t 
&= 
2\,\bm{D}\,\bm{D}^{-1}
(\bm{u}_{xx}+\bm{u}_{yy})
+ \bm{D}\,\begin{pmatrix}1&0\\0&-1\end{pmatrix} \bm{D}^{-1}
(\bm{u}_{yy}-\bm{u}_{xx})
-2\,\bm{D}\,\begin{pmatrix}0&1\\1&0\end{pmatrix} \bm{D}^{-1}
\bm{u}_{xy}
\end{align}
which expands to the PDE system
\eqref{ojapde22} with coefficient matrices
\eqref{ojapde22-A}, \eqref{ojapde22-B}
as stated in the theorem.
\end{proof}

\begin{remark}
The derivation of the PDE of Theorem~\ref{thm-ojapde22}
by affine transformation immediately implies its affine equivariance.
The final PDE itself is even equivariant under affine transformations of
the $x$-$y$ plane. Regarding the approximation of Oja median 
filtering, however, the Euclidean disc-shaped structuring element
allows only for Euclidean transformations of the $x$-$y$ plane.
\end{remark}

\subsubsection{Interpretation of Bivariate Median Filter PDEs}
\label{sssec-interpret22}

The geometric meaning of the PDE systems
from Sections~\ref{sssec-l1pde22} and~\ref{sssec-ojapde22} is best 
discussed by considering the principal components of the local variation
of the data. In the general setting of Proposition~\ref{prop-l1pde22}
and Theorem~\ref{thm-ojapde22} the channelwise evolutions $u_t$, $v_t$
are mixtures of these principal components, which obscures their
geometric significance. In the case of diagonal Jacobian $\mathrm{D}\bm{u}$ as 
in the hypothesis of Lemma~\ref{lem-l1pde22-aligned} the channels
are decorrelated and represent these principal components. 

We base our discussion therefore on the PDE system 
\eqref{l1pde22-aligned-u}, \eqref{l1pde22-aligned-v} from 
Lemma~\ref{lem-l1pde22-aligned} for the $L^1$ median, and 
\begin{align}
u_t &= u_{xx} + 3u_{yy} -2u_xv_{xy}/v_y\;,
\label{ojapde22-aligned-1}
\\
v_t &= 3v_{xx} + v_{yy} -2v_yu_{xy}/u_x
\label{ojapde22-aligned-2}
\end{align}
for the Oja median which is the straightforward adaptation of
the PDE system \eqref{ojapde22-I1}, \eqref{ojapde22-I2} from
Lemma~\ref{lem-ojapde22-I} to the situation of a general
diagonal Jacobian.

Comparing the two PDE systems, we see that in each of them
an isotropic linear diffusion contribution 
$(u_{xx}+u_{yy},v_{xx}+v_{yy})\transpose$
is combined with an additional directional diffusion 
$(u_{yy},v_{xx})\transpose$ and
a cross-effect contribution $(u_xv_{xy}/v_y,v_yu_{xy}/u_x)\transpose$ 
with some weights.

For the directional diffusion term it is worth noticing that the $y$ 
direction for $u$, and $x$ direction for $v$ are the level-line directions
of the individual components, i.e.\ this term represents 
independent (mean) curvature motion evolutions for the two principal
components.

The mixed second derivatives of the third term express the torsion of the 
graphs of the two principal components, and are multiplied with scaling 
factors that adapt between the componentwise gradients $u_x$ and $v_y$.

In the Oja median PDE, the weights of these terms are constant.
The first two terms act independent in the two components such that 
the torsion-based cross-effect term constitutes the only coupling between
principal components. 

In contrast, the coefficient functions $Q_1$ and $Q_2$ in the
$L^1$ case modulate also the diffusion and curvature terms and create
additional cross-effects between the principal components. 
This is due to the more rigid
Euclidean structure underlying the $L^1$ median definition, and
also makes it sensible to write the PDE for the general case using the
eigenvector directions $\bm{\eta}$ and $\bm{\xi}$ of the
structure tensor as done in Proposition~\ref{prop-l1pde22}. In the
decoupled setting of Theorem~\ref{thm-ojapde22} these directions have
no meaning. This is plausible because these eigenvectors are
strongly related with a Euclidean geometry concept of the $u$-$v$
plane, and are thereby inappropriate for an affine equivariant
process like Oja median filtering.

In detail, the effect of the coefficient functions $Q_1$ and $Q_2$ 
is steered by the relative weight of the principal components,
namely $u_x$ and $v_y$ in the aligned case under consideration.
Denoting the principal component with stronger gradient as dominant
component, and the other as non-dominant component, one sees that
the more pronounced the dominance of the first principal component is, 
the more does it steer the evolution also of the other principal component
(as the joint pseudo-gradient vector $\bm{\eta}$ aligns more and more
with the gradient vector of the dominant component).

\subsubsection{Discussion of the Degenerate Case 
$\mathrm{det}\,\mathrm{D}\bm{u}=0$}
\label{sssec-ojapde22-deg}

The right-hand side of equation \eqref{ojapde22} is undefined 
at locations where $\mathrm{det}\,\mathrm{D}\bm{u}=0$.
While the weights for the second derivatives
$\bm{u}_{xx}$ and $\bm{u}_{yy}$ remain bounded when
$\mathrm{det}\,\mathrm{D}\bm{u}$ goes to zero, the weights
of the mixed terms $u_{xy}$ and $v_{xy}$ can take arbitrarily large
values in this case. To see more precisely what is going on,
let us consider once more the case of a diagonal Jacobian 
$\mathrm{D}\bm{u}$, and keep $u_x=1$ fixed while
$v_y$ goes to zero. Then the weight of $v_{xy}$ in the PDE
\eqref{ojapde22-I1} for $u_t$ goes to infinity with $1/v_y$ whereas the
weight of $u_{xy}$ in the PDE \eqref{ojapde22-I2} for $v_t$ goes to
zero. This is different from the situation for the $L^1$ median where
the coefficients of the mixed terms $u_{xy}$ and $v_{xy}$ were bounded 
for all values of the gradient. However, it is easy to see that
for an affine equivariant median there is basically no way out:
As soon as there is a non-zero influence of $v_{xy}$ on $u_t$, it must
scale in this way by affine equivariance.

Keeping in mind, however, that $u_x$ and $v_y$ for diagonal $\mathrm{D}\bm{u}$
are the channelwise gradient directions, it becomes evident that 
divergent behaviour, such as $v_y$ going to zero while $u_xv_{xy}$ in the
numerator is nonzero, can affect only isolated points in the plane,
and can thereby be cured by using the concept of viscosity solutions.
Vanishing of $v_y$ in an extended region is only possible if the 
function $v$ is constant in this region such that also $v_{xy}$ vanishes,
allowing to fill this definition gap in the term $u_xv_{xy}/v_y$ with zero.

This is also in harmony with the behaviour of the median filter itself.
As the median of a set of data values is restricted to the convex hull of the
input data, infinite amplification of variations from the $v$ to the $u$
component and vice versa is impossible.
As the PDE is only approximated in the limit $\varrho\to0$, it can moreover
be expected that for positive $\varrho$, the sensitivity of the $u$ component
of the median filtering result to $v_{xy}$ will be dampened nonlinearly
which would be reflected in higher order terms neglected in the PDE derivation.

Structuring elements with varying radius $\varrho$ can be translated
to fixed radius by scaling the second-order Taylor coefficients
of the bivariate function, i.e.\ $u_{xx}(\bm{0})$, etc.,
with $\varrho$. Deviations from the PDE behaviour for positive $\varrho$ 
can therefore be studied equivalently by investigating 
nonlinearities in the response of the median to increasing values of
the derivatives $u_{xx}$ within a fixed structuring element.
In Section \ref{sssec-cr-degen} we will demonstrate this dampening
by a numerical experiment.

\subsubsection{Affine Equivariant Transformed $L^1$ Median}
\label{sssec-l1apde22}

As pointed out in Remark~\ref{rm-l1oja22-I}, 
the PDEs approximated by the bivariate $L^1$ and Oja median filters 
coincide when the Jacobian of the image being filtered is the
unit matrix. The difference between the $L^1$ case in 
Proposition~\ref{prop-l1pde22} and the Oja case in
Theorem~\ref{thm-ojapde22} is that the affine equivariance of
the latter allows to derive the general case by affine transformations
from the special case $\mathrm{D}\bm{u}=\bm{I}$, whereas
the $L^1$ median admits only Euclidean transformations such
that its general case needs to be derived from the wider setting
of Lemma~\ref{lem-l1pde22-aligned} where $\mathrm{D}\bm{u}$ can be
arbitrary diagonal. This is where the complicated coefficient
functions of Proposition~\ref{prop-l1pde22} have their origin.

On the other hand, one can combine the minimisation
principle of the $L^1$ median with the affine transformation
concept from the proof of Theorem~\ref{thm-ojapde22} to design
a bivariate space-continuous image filter as follows.

\begin{definition}
[Space-continuous affine equivariant transformed $L^1$ median filter.]
\label{def-l1a22}
Let a function $\bm{u}:\mathbb{R}^2\supset\varOmega\to\mathbb{R}^2$
and the structuring element $D_\varrho$ be given.
For each location $\bm{x}_0\in\varOmega$
with $\mathrm{det}\,\mathrm{D}\bm{u}(\bm{x}_0)\ne0$,
transform the function values $\bm{u}(\bm{x})$ for
$\bm{x}\in\bm{x}_0+D_\varrho$ via
$\hat{\bm{u}} = \mathrm{D}\bm{u}(\bm{x}_0)^{-1}\bm{u}$.
Determine the $L^1$ median $\hat{\bm{u}}^*$ of the data $\hat{\bm{u}}$.
Transform $\hat{\bm{u}}^*$ back to 
$\bm{u}^*(\bm{x}_0)=\mathrm{D}\bm{u}(\bm{x}_0)\hat{\bm{u}}^*$.
The image filter that transfers the input function 
$\bm{u}:\varOmega\to\mathbb{R}^2$ to the
function $\bm{u}^*:\varOmega\to\mathbb{R}^2$ is called affine equivariant
transformed $L^1$ median filter.
\end{definition}

Affine equivariance of this image filter is clear by construction.
By inheritance from the underlying $L^1$ median it approximates
in the case $\mathrm{D}\bm{u}=\bm{I}$ the same
PDEs for $\varrho\to0$ as the $L^1$ and Oja median filters.
Due to its construction from this special case via the affine transform
with $\mathrm{D}\bm{u}$ it finally approximates in the general
(non-degenerate) situation the same PDEs as the Oja median filter.
We have thus the following corollary.

\begin{corollary}
\label{cor-l1apde22}
Let a bivariate image
$\bm{u}:\bbbr^2\supset\varOmega\to\bbbr^2$, $(x,y)\mapsto(u,v)$,
be given. At any location where $\mathrm{det}\,\mathrm{D}\bm{u}\ne0$,
one step of affine equivariant transformed $L^1$ median filtering of $\bm{u}$
with the structuring element $D_\varrho$
approximates for $\varrho\to0$ an explicit time step of size
$\tau=\varrho^2/24$ of the PDE system \eqref{ojapde22} 
from Theorem~\ref{thm-ojapde22}.
\end{corollary}

Using this approach for practical, i.e.\ discrete image filtering,
requires to estimate from the discrete
image data within a structuring element the Jacobian $\mathrm{D}\bm{u}$.
But the space-continuous
data within $\bm{x}_0+D_\varrho$ represent a distribution whose
covariance matrix asymptotically approaches $\mathrm{D}\bm{u}$ as
$\varrho\to0$.
Thus, estimation of this covariance matrix from sampled data
as used in the transformation--retransformation $L^1$ median approaches
\cite{Chakraborty-PAMS96,Hettmansperger-Biomet02,Rao-Sankhya88} 
and as used in our experimental demonstration in Section~\ref{ssec-demo-rgb}
is exactly what is needed here. Hence, the filter from 
Definition~\ref{def-l1a22} is a space-continuous version of the discrete 
transformation--retransformation $L^1$ median filter.

Corollary~\ref{cor-l1apde22} therefore states that
as bivariate image filters, \emph{the 
affine equivariant transformed $L^1$ median is
asymptotically equivalent to the Oja median.}
Further analysis in this section as well as numerical evidence in
Section~\ref{ssec-cr} will reveal that this asymptotic equivalence
generalises beyond the bivariate case.

\subsection{Three-Channel Volume Images}\label{ssec-l1ojapde33}

As the next step in our theoretical investigation, we increase the
dimensions of image and value domain equally to three, thus arriving
at three-channel volume images. A possible application would be given
by 3D deformation fields as they arise in elastic registration of 
medical 3D data sets. We do, however, not aim at applications of this
setting within this work, and include it primarily for the
theoretical completeness. Our focus in this context will be on
affine equivariant median filters.

\subsubsection{Oja Median}\label{sssec-ojapde33}

The first three-channel volume filter we consider will be based on the
3D Oja median in the sense of \eqref{mOja} minimising a sum of volumes of
tetrahedra.

\begin{theorem}\label{thm-ojapde33}
Let a three-channel volume image
$\bm{u}:\bbbr^3\supset\varOmega\to\bbbr^3$, $(x,y,z)\mapsto(u,v,w)$,
be given. At any location where $\mathrm{det}\,\mathrm{D}\bm{u}\ne0$,
one step of Oja median filtering of $\bm{u}$
with the structuring element $B_\varrho$
approximates for $\varrho\to0$ an explicit time step of size
$\tau=\varrho^2/{60}$ of the PDE system
\begin{align}
\begin{pmatrix}u_t\\v_t\\w_t\end{pmatrix} &=
5 \begin{pmatrix}
u_{xx}+u_{yy}+u_{zz}\\v_{xx}+v_{yy}+v_{zz}\\w_{xx}+w_{yy}+w_{zz}
\end{pmatrix}
+ \bm{A}_1(\mathrm{D}\bm{u}) \begin{pmatrix}
u_{yy}\!-\!u_{xx}\\v_{yy}\!-\!v_{xx}\\w_{yy}\!-\!w_{xx}
\end{pmatrix}
+ \bm{A}_2(\mathrm{D}\bm{u}) \begin{pmatrix}
u_{zz}\!-\!u_{xx}\\v_{zz}\!-\!v_{xx}\\w_{zz}\!-\!w_{xx}
\end{pmatrix}
\notag\\*&\quad{}
-3\, \bm{B}_1(\mathrm{D}\bm{u}) \begin{pmatrix}
u_{xy}\\v_{xy}\\w_{xy}
\end{pmatrix}
-3\, \bm{B}_2(\mathrm{D}\bm{u}) \begin{pmatrix}
u_{xz}\\v_{xz}\\w_{xz}
\end{pmatrix}
-3\, \bm{B}_3(\mathrm{D}\bm{u}) \begin{pmatrix}
u_{yz}\\v_{yz}\\w_{yz}
\end{pmatrix}
\label{ojapde33}
\end{align}
where for $\bm{D}:=\mathrm{D}\bm{u}$ the coefficient matrices are
given by
\begin{align}
\bm{A}_1(\bm{D}) &:= 
\bm{I} - 3\,\bm{D}\,\mathrm{diag}(0,1,0)\,\bm{D}^{-1}\;,
\label{ojapde33-A1}
\\
\bm{A}_2(\bm{D}) &:= 
\bm{I} - 3\,\bm{D}\,\mathrm{diag}(0,0,1)\,\bm{D}^{-1}\;,
\label{ojapde33-A2}
\\
\bm{B}_1(\bm{D}) &:= 
\bm{D}\,\begin{pmatrix}0&1&0\\1&0&0\\0&0&0\end{pmatrix}\,\bm{D}^{-1}\;,
\label{ojapde33-B1}
\\
\bm{B}_2(\bm{D}) &:= 
\bm{D}\,\begin{pmatrix}0&0&1\\0&0&0\\1&0&0\end{pmatrix}\,\bm{D}^{-1}\;,
\label{ojapde33-B2}
\\
\bm{B}_3(\bm{D}) &:= 
\bm{D}\,\begin{pmatrix}0&0&0\\0&0&1\\0&1&0\end{pmatrix}\,\bm{D}^{-1}
\label{ojapde33-B3}
\;.
\end{align}
\end{theorem}

The proof of this theorem proceeds analogously to the proof
of Theorem~\ref{thm-ojapde22}, with the use of the following
lemma that is analogous to Lemma~\ref{lem-ojapde22-I}.

\begin{lemma}\label{lem-ojapde33-I}
Let $\bm{u}$ be given as in Theorem~\ref{thm-ojapde33}, with
$\bm{0}=(0,0,0)$ being in the interior of $\varOmega$.
Assume that $\mathrm{D}\bm{u}(\bm{0})$ is the $3\times3$ unit matrix $\bm{I}$.
At $\bm{x}=\bm{0}$, one step of Oja median filtering of $\bm{u}$
with the structuring element $B_\varrho$ then
approximates for $\varrho\to0$ an explicit time step of size
$\tau=\varrho^2/{20}$ of the PDE system
\begin{align}
u_t&=u_{xx}+2(u_{yy}+u_{zz})-(v_{xy}+w_{xz}) \label{ojapde33-I-1}\\
v_t&=v_{yy}+2(v_{xx}+v_{zz})-(u_{xy}+w_{yz}) \label{ojapde33-I-2}\\
w_t&=w_{zz}+2(w_{xx}+w_{yy})-(u_{xz}+v_{yz}) \label{ojapde33-I-3}\;.
\end{align}
\end{lemma}

The proof of this lemma extends the first proof of
Lemma~\ref{lem-ojapde22-I} and is given in Appendix~\ref{app-33proof}.

\begin{remark}
In full analogy with the bivariate case, see Section~\ref{sssec-interpret22}, 
the PDE system can be interpreted in terms of the principal components of local
data variation, which appear decorrelated in Lemma~\ref{lem-ojapde33-I}.
Again, the PDEs combine isotropic diffusion with componentwise
mean curvature motion given by 
$(u_{yy}+u_{zz},v_{xx}+v_{zz},w_{xx}+w_{zz})\transpose$
and cross-effect terms. The latter couple each pair of principal components
by mutual influence based on the torsion of these components
in the plane spanned by both.
\end{remark}

\subsubsection{Affine Equivariant Transformed $L^1$ Median}
\label{sssec-l1apde33}

Definition~\ref{def-l1a22} can be applied verbatim to define an
affine equivariant transformed $L^1$ median filter for
functions $\bm{u}:\mathbb{R}^3\supset\varOmega\to\mathbb{R}^3$, which
we will consider now.

\begin{proposition}\label{prop-l1apde33}
Let a three-channel volume image
$\bm{u}:\bbbr^3\supset\varOmega\to\bbbr^3$, $(x,y,z)\mapsto(u,v,w)$,
be given. At any location where $\mathrm{det}\,\mathrm{D}\bm{u}\ne0$,
one step of affine equivariant transformed $L^1$ median filtering of $\bm{u}$
with the structuring element $B_\varrho$
approximates for $\varrho\to0$ an explicit time step of size
$\tau=\varrho^2/{60}$ of the PDE system
\eqref{ojapde33} with the coefficient matrices 
\eqref{ojapde33-A1}--\eqref{ojapde33-B3} as stated in 
Theorem~\ref{thm-ojapde33}.
\end{proposition}

This proposition is a consequence of the following lemma.

\begin{lemma}\label{lem-l1apde33-I}
Let $\bm{u}$ be given as in Proposition~\ref{prop-l1apde33}.
Assume that $\mathrm{D}\bm{u}(\bm{0})$ is the $3\times3$ unit matrix $\bm{I}$.
At $\bm{x}=\bm{0}$, one step of $L^1$ median filtering of $\bm{u}$
with the structuring element $B_\varrho$ then
approximates for $\varrho\to0$ an explicit time step of size
$\tau=\varrho^2/{20}$ of the PDE system
\eqref{ojapde33-I-1}--\eqref{ojapde33-I-3} from Lemma~\ref{lem-ojapde33-I}.
\end{lemma}

The proof of this lemma is based on the same principle as the proof
of Lemma~\ref{lem-l1pde22-aligned} in \cite{Welk-Aiep14}, extended from
two to three dimensions but at the same time simplified by restricting
the Jacobian $\mathrm{D}\bm{u}$ to the unit matrix.
It is detailed in Appendix~\ref{app-l1a33proof}.

We remark that there is no serious technical obstacle to generalising
this proof even to arbitrary diagonal Jacobians, which would yield a PDE 
approximation result for the standard $L^1$ median in three dimensions. 
As in Proposition~\ref{prop-l1pde22}, quotients of elliptic integrals 
would appear as coefficient functions. With our focus on affine 
equivariant filters, we do not need this generality here.

\subsection{Three-Channel Planar Images}\label{ssec-ojapde23}

So far we have considered settings in which the number of dimensions of
the image domain $\varOmega$ equalled that of the data space.
There are, however, important classes of image data for which this
is not the case, with RGB colour images over planar domains being
the most prominent example. As our last dimensional setting, we
will therefore consider three-channel images over planar domains.
The Jacobian of such an image is a $3\times2$ matrix field.
The generic case is therefore no longer given by an invertible
Jacobian but just by the rank of the Jacobian being maximal ($2$),
which requires adjusting several arguments.
Our focus will again be on affine equivariant filters.

\subsubsection{2D Oja Median}\label{sssec-ojapde23}

With regard to the degeneracy of the 3D Oja median in the sense of
\eqref{mOja} in the case of
three-channel data over a planar domain that has already been 
discussed in Section~\ref{ssec-demo-rgb} we choose the 2D Oja median
in the sense of \eqref{mOja23}
for our theoretical analysis. 

\begin{theorem}\label{thm-ojapde23}
Let a three-channel planar image
$\bm{u}:\bbbr^2\supset\varOmega\to\bbbr^3$, $(x,y)\mapsto(u,v,w)$,
be given. At any location where the $3\times2$ matrix
$\mathrm{D}\bm{u}$ has rank 2,
one step of 2D Oja median filtering of $\bm{u}$
with the structuring element $D_\varrho$
approximates for $\varrho\to0$ an explicit time step of size
$\tau=\varrho^2/{24}$ of the PDE system
\begin{align}
\begin{pmatrix}u_t\\v_t\\w_t\end{pmatrix} &=
2 \begin{pmatrix}
u_{xx}+u_{yy}\\v_{xx}+v_{yy}\\w_{xx}+w_{yy}
\end{pmatrix}
+ \bm{A}(\mathrm{D}\bm{u}) \begin{pmatrix}
u_{yy}\!-\!u_{xx}\\v_{yy}\!-\!v_{xx}\\w_{yy}\!-\!w_{xx}
\end{pmatrix}
-2\, \bm{B}(\mathrm{D}\bm{u}) \begin{pmatrix}
u_{xy}\\v_{xy}\\w_{xy}
\end{pmatrix}
\label{ojapde23}
\end{align}
where for $\bm{D}:=\mathrm{D}\bm{u}=(\partial_x \bm{u}~|~\partial_y \bm{u})$ 
and 
\begin{equation}
\bm{D}_3:= \left(\partial_x \bm{u}~\Big|~\partial_y \bm{u}~\Big|~
\partial_x \bm{u}\times\partial_y\bm{u}\right)
\label{D3}
\end{equation}
the coefficient matrices are given by
\begin{align}
\bm{A}(\bm{D}) &:= 
\bm{D}_3\,\mathrm{diag}(1,-1,0)\,\bm{D}_3^{-1}\;,
\label{ojapde23-A}
\\
\bm{B}(\bm{D}) &:= 
\bm{D}_3\,\begin{pmatrix}0&1&0\\1&0&0\\0&0&0\end{pmatrix}\,\bm{D}_3^{-1}
\label{ojapde23-B}
\;.
\end{align}
\end{theorem}

\begin{remark}\label{rm-oja23-D3}
Note that $\bm{D}_3$, the $3\times3$ matrix obtained by enlarging the
$2\times3$ Jacobian $\mathrm{D}\bm{u}$ with a third column orthogonal to the
first two ones, is regular if and only if $\mathrm{D}\bm{u}$ has rank 2
as required in the hypothesis of the theorem. The transformed
variables $\hat{\bm{u}}:=\bm{D}_3^{-1}\bm{u}$ have the Jacobian 
$\begin{pmatrix}1&0\\0&1\\0&0\end{pmatrix}$. Any scaling of the third
column of $\bm{D}_3$ is actually irrelevant for the statement
and proof of the theorem; 
it cancels out in the evaluation of \eqref{ojapde23-A} and \eqref{ojapde23-B}. 
It may, however, affect the scaling of deviations
from the PDE that occur for positive structuring element radius $\varrho$.
\end{remark}

With the coordinate transform $\bm{D}_3$, the proof of the theorem proceeds 
analogously to the proof of Theorem~\ref{thm-ojapde22} and relies on the 
following lemma.

\begin{lemma}\label{lem-ojapde23-I}
Let $\bm{u}$ be given as in Theorem~\ref{thm-ojapde23}, with the image
domain $\varOmega$ containing the origin $\bm{0}=(0,0)$ in its interior.
Assume that 
$\mathrm{D}\bm{u}(\bm{0}) = \begin{pmatrix}1&0\\0&1\\0&0\end{pmatrix}$.
At $\bm{x}=\bm{0}$, one step of Oja median filtering of $\bm{u}$
with the structuring element $D_\varrho$ then
approximates for $\varrho\to0$ an explicit time step of size
$\tau=\varrho^2/{24}$ of the PDE system
\begin{align}
u_t&=u_{xx}+3u_{yy}-2v_{xy} \label{ojapde23-I-1} \\
v_t&=3v_{xx}+v_{yy}-2u_{xy} \label{ojapde23-I-2} \\
w_t&=2w_{xx}+2w_{yy}        \label{ojapde23-I-3} \;.
\end{align}
\end{lemma}

The proof of this lemma is given in Appendix~\ref{app-23proof}.
It is based on the result from Lemma~\ref{lem-ojapde22-I} for the
bivariate Oja median
and extends it with a calculation of the $w_t$ component.
In doing the latter, we reformulate the Oja median
function into a weighted $L^1$ median for the feet of altitudes
in the triangles, and proceed then
analogously to the proof of the three-channel $L^1$ median result,
Lemma~\ref{lem-l1apde33-I},
where the minimisation condition was evaluated by explicit integration
over the structuring element. 
This approach has been avoided in the other proofs for the Oja
median results because in the general Oja median case it turns out
extremely tedious, but in the special case considered here it becomes
feasible by exploiting a rotational symmetry argument in combination
with an integration in polar coordinates similar to the first proof
of Lemma~\ref{lem-ojapde22-I}.

\subsubsection{Affine Equivariant Transformed $L^1$ Median}
\label{sssec-l1apde23}

Turning to the affine equivariant transformed $L^1$ median filter,
Definition~\ref{def-l1a22} of its space-continuous variant
in the bivariate case does not transfer straightforwardly to the
situation of three-channel planar images as it uses the inverse of
the Jacobian of the input function. 
For our analysis, we adopt the proceeding from
Section~\ref{sssec-ojapde23} and use the enlarged Jacobian
$\bm{D}_3$ with the same scaling convention as in \eqref{D3}.
We can then define the filter to be analysed as follows.

\begin{definition}[Space-continuous affine equivariant transformed 
$L^1$ median filter for three-channel planar images.]
\label{def-l1a23}
Let a function $\bm{u}:\mathbb{R}^2\supset\varOmega\to\mathbb{R}^3$
and the structuring element $D_\varrho$ be given.
For each location $\bm{x}_0\in\varOmega$
where $\mathrm{D}\bm{u}(\bm{x}_0)$ has rank $2$,
let $\bm{D}_3=\bm{D}_3(\bm{x}_0)$ be given as in \eqref{D3}, and
transform the function values $\bm{u}(\bm{x})$ for
$\bm{x}\in\bm{x}_0+D_\varrho$ via
$\hat{\bm{u}} = \bm{D}_3(\bm{x}_0)^{-1}\bm{u}$.
Determine the $L^1$ median $\hat{\bm{u}}^*$ of the data $\hat{\bm{u}}$.
Transform $\hat{\bm{u}}^*$ back to 
$\bm{u}^*(\bm{x}_0)=\bm{D}_3(\bm{x}_0)\hat{\bm{u}}^*$.
The image filter that transfers the input function 
$\bm{u}:\varOmega\to\mathbb{R}^3$ to the
function $\bm{u}^*:\varOmega\to\mathbb{R}^3$ is called affine equivariant
transformed $L^1$ median filter.
\end{definition}

\begin{remark}
As in the case of the Oja median, any possible scaling of the third
column of $\bm{D}_3$ will be irrelevant for the asymptotic analysis carried
out in the following. A caveat arises, however, when a discrete
filter based on the transformation--retransformation
approach \cite{Chakraborty-PAMS96,Hettmansperger-Biomet02,Rao-Sankhya88}
or as implemented in Section~\ref{ssec-demo-rgb} is used as discrete
approximation for variable positive structuring element radius $\varrho$.
As this discrete procedure just takes the $\mathbb{R}^3$ input data as samples 
of a 3D distribution and tries to normalise this distribution, it might
introduce a scaling factor that changes with $\varrho$. We leave analysis
of this difficulty for future work but remark that the results of our
numerical experiments in Section~\ref{sssec-cr23} support the validity of
the analysis given here.
\end{remark}

\begin{proposition}\label{prop-l1apde23}
Let a three-channel planar image
$\bm{u}:\bbbr^2\supset\varOmega\to\bbbr^3$, $(x,y)\mapsto(u,v,w)$,
be given. At any location where the $3\times2$ matrix
$\mathrm{D}\bm{u}$ has rank 2,
one step of affine equivariant transformed $L^1$ median filtering of $\bm{u}$
with the structuring element $D_\varrho$
approximates for $\varrho\to0$ an explicit time step of size
$\tau=\varrho^2/{24}$ of the PDE system
\eqref{ojapde23} with the coefficient matrices 
\eqref{ojapde23-A}--\eqref{ojapde23-B} as stated in Theorem~\ref{thm-ojapde23}.
\end{proposition}

The proof of the proposition is analogous to the proof of
Theorem~\ref{thm-ojapde23}, using the special case $u_x=v_y=1$ of the
following lemma. The lemma itself is corrected from \cite{Welk-Aiep14}
and rewritten for the three-channel case.

\begin{lemma}[from \cite{Welk-Aiep14}, corrected]\label{lem-l1pde23-aligned}
Let $\bm{u}$ be given as in Proposition~\ref{prop-l1apde23}, with the image
domain $\varOmega$ containing the origin $\bm{0}=(0,0)$ in its interior.
Assume that the Jacobian at $\bm{0}$ is of the form 
$\mathrm{D}\bm{u}(\bm{0}) = \begin{pmatrix}u_x&0\\0&v_y\\0&0\end{pmatrix}$
with $u_x\ge v_y>0$.
Then one step of $L^1$ median filtering with the structuring element
$D_{\varrho}$ approximates for $\varrho\to0$ at $(x,y)$ an explicit time
step of size $\tau=\varrho^2/6$ of the PDE system consisting of the
equations \eqref{l1pde22-aligned-u}, \eqref{l1pde22-aligned-v}, and
\begin{align}
w_t &= 
Q_3\left(\frac{u_x}{v_y}\right) w_{xx} 
+ Q_3\left(\frac{v_y}{u_x}\right) w_{yy}
\label{l1pde23-aligned-w}
\;,
\end{align}
where the coefficient function $Q_3$ is given by
\begin{align}
Q_3(\lambda) &= \frac{3 \iint_{D_1(\bm{0})} 
t^2/(s^2+\lambda^2t^2)^{3/2}\,\mathrm{d}s\,\mathrm{d}t}
{\iint_{D_1(\bm{0})}
1/(s^2+\lambda^2t^2)^{3/2}\,\mathrm{d}s\,\mathrm{d}t}\;.
\end{align}
\end{lemma}

\begin{remark}\label{rm-l1pde23-I}
In the case $u_x=1$, $v_y=1$, the coefficients of \eqref{l1pde23-aligned-w}
simplify via $Q_3(1)=1/2$ such that
one obtains
\begin{align}
w_t &= \tfrac12w_{xx} + \tfrac12w_{yy}\;,
\end{align}
which together with \eqref{l1pde22-I-1} and \eqref{l1pde22-I-2} and
after rescaling the time variable by $4$
yields \eqref{ojapde23-I-1}--\eqref{ojapde23-I-3}. This is the relevant
case for the proof of Proposition~\ref{prop-l1apde23}.
\end{remark}

\section{Experimental Validation of the PDEs for Multivariate 
Median Filtering}
\label{sec-ex}

This section is focussed at validating the PDE approximation
results from Section~\ref{sec-pde} by numerical experiments.

\subsection{Simple Example Functions}
\label{ssec-cr}

We start with several experiments on simple example functions
in which we compare the individual coefficients of the
PDEs for the different variants of multivariate median filters
derived in Section~\ref{sec-pde} with median filtering results
for the function values sampled at high resolution.

\subsubsection{Bivariate Filters, Case $\mathrm{D}\bm{u}=\bm{I}$}
\label{sssec-cr22}

\begin{table}[b!]
\caption{\label{tab-l1oja22-coeffcomp}
Validation of the PDE approximation of bivariate $L^1$ and Oja median
filtering in the case $u_x=v_y=1$, $u_y=v_x=0$, see 
Remark~\ref{rm-l1pde22-I} and Lemma~\ref{lem-ojapde22-I}.
Median values $(u^*,v^*)$ 
computed from functions sampled with resolution $0.01$
in a disc-shaped structuring element of radius $1$ are juxtaposed with
the time steps of size $\tau=1/24$ of the corresponding PDE system
\eqref{ojapde22-I1}, \eqref{ojapde22-I2}.
Medians and time steps are scaled by $10^6$ for more compact 
representation.
}
\medskip
\centering
\begin{tabular}
{lllllrrrrrr}
\hline
\multicolumn{2}{c}{\rule{0pt}{1.2em}Function}&
\multicolumn{3}{c}{Derivatives}&
\multicolumn{2}{c}{$L^1$ median}&
\multicolumn{2}{c}{Oja median}&
\multicolumn{2}{c}{PDE time step}\\
$u$&$v$&$u_{xx}$&$u_{xy}$&$u_{yy}$
&\multicolumn{1}{c}{$10^6u^*$}&\multicolumn{1}{c}{$10^6v^*$}
&\multicolumn{1}{c}{$10^6u^*$}&\multicolumn{1}{c}{$10^6v^*$}
&\multicolumn{1}{c}{$10^6\tau u_t$}&\multicolumn{1}{c}{$10^6\tau v_t$}\\\hline
\rule{0pt}{1.2em}%
$x+0.05x^2$&$y$        &$0.1$&$0$&$0$
&$     4\,167$&$          0$
&$     4\,181$&$          0$
&$     4\,167$&$          0$
\\
$x+0.1xy$  &$y$        &$0$&$0.1$&$0$
&$          0$&$-    8\,364$
&$          0$&$-    8\,372$
&$          0$&$-    8\,333$
\\
$x+0.05y^2$&$y$        &$0$&$0$&$0.1$
&$    12\,479$&$          0$
&$    12\,479$&$          0$
&$    12\,500$&$          0$
\\
\hline
\end{tabular}
\end{table}

To verify the results on bivariate $L^1$ and Oja median filtering,
we focus first on the case $\mathrm{D}\bm{u}(\bm{0})=\bm{I}$, see
Remark~\ref{rm-l1pde22-I} and Lemma~\ref{lem-ojapde22-I}.

We discretise sample functions $u(x,y)$ and $v(x,y)$ in the
structuring element $D_1(\bm{0})$, i.e.\ the disc of radius $1$
around the origin, with a grid resolution of $0.01$ in $x$ and
$y$ direction, which yields $31\,417$ sample points. 
For symmetry reasons, the sample functions are chosen to test
only the weights of $u_{xx}$, $u_{xy}$ and $u_{yy}$ while leaving
$v(x,y)\equiv y$. 
For these input data we compute the $L^1$ and Oja medians and compare 
these with the theoretical values given by the right-hand side of
\eqref{ojapde22-I1}, \eqref{ojapde22-I2} with the time
step size $1/24$. The results can be found in 
Table~\ref{tab-l1oja22-coeffcomp}. 
The observed deviations in the range of $1.5\times10^{-4}$
between the computed medians and PDE time steps are expectable given the
grid resolution.

\subsubsection{Three-Channel Planar Image Filters, Isotropic Case}
\label{sssec-cr23}

The tests for bivariate filters in the case $\mathrm{D}\bm{u}=\bm{I}$
from Section~\ref{sssec-cr22} can easily be extended to three-channel
planar 
image filtering by $L^1$ and 2D Oja medians. 
Using the same structuring element and sampling grid
as before, we sample now $\mathbb{R}^3$-valued functions $(x,y)\mapsto(u,v,w)$
with $u_x=v_y=1$, $u_y=v_x=w_x=w_y=0$, where we vary single second-order
Taylor coefficients away from zero.

The results shown in Table~\ref{tab-l1oja23-coeffcomp} indicate an
accuracy of approximation comparable to the previous case, and thereby
confirm the validity of the approximation results from 
Lemma~\ref{lem-ojapde23-I} and Lemma~\ref{lem-l1pde23-aligned} (with
$u_x=v_y=1$).
We have omitted test cases where only second derivatives of $u$ and $v$
were varied, because in these cases results were identical to the
pure bivariate case.

\begin{table}[b!]
\caption{\label{tab-l1oja23-coeffcomp}
Validation of the PDE approximation of $L^1$ and 2D Oja median
filtering of three-channel planar image data
in the case $u_x=v_y=1$, $u_y=v_x=w_x=w_y=0$.
Median values $(u^*,v^*,w^*)$
computed from functions sampled with resolution $0.01$
in a disc-shaped structuring element of radius $1$ are juxtaposed with
the time steps of size $\tau=1/24$ of the corresponding PDE system
\eqref{ojapde23-I-1}--\eqref{ojapde23-I-3}.
Medians and time steps are scaled by $10^6$ for more compact 
representation.
}
\medskip
\centering
\small
\begin{tabular}
{l@{~}l@{~}l@{~~~}cr@{~}r@{~}r@{~~~}r@{~}r@{~}r@{~~~}r@{~}r@{~}r}
\hline
\multicolumn{3}{c}{\rule{0pt}{1.2em}Function}&
\multicolumn{1}{c}{\kern-1emNonzero second\kern-1em}&
\multicolumn{3}{c}{$L^1$ median}&
\multicolumn{3}{c}{Oja median}&
\multicolumn{3}{c}{PDE time step}\\
$u$&$v$&$w$
&\multicolumn{1}{c}{derivatives}
&\multicolumn{1}{@{}c@{}}{$10^6u^*$}
&\multicolumn{1}{@{}c@{}}{$10^6v^*$}
&\multicolumn{1}{@{}c@{~~~}}{$10^6w^*$}
&\multicolumn{1}{@{}c@{}}{$10^6u^*$}
&\multicolumn{1}{@{}c@{}}{$10^6v^*$}
&\multicolumn{1}{@{}c@{~~~}}{$10^6w^*$}
&\multicolumn{1}{@{}c@{~}}{$10^6\tau u_t$}
&\multicolumn{1}{@{}c@{~}}{$10^6\tau v_t$}
&\multicolumn{1}{@{}c@{}}{$10^6\tau w_t$}\\\hline
\rule{0pt}{1.2em}%
$x$        &$y$        &$0.05x^2$
&$w_{xx}=0.1$
&$          0$&$          0$&$     8\,401$
&$          0$&$          0$&$     8\,390$
&$          0$&$          0$&$     8\,333$
\\
$x$        &$y$        &$0.1xy$
&$w_{xy}=0.1$
&$          0$&$          0$&$          0$
&$          0$&$          0$&$          0$
&$          0$&$          0$&$          0$
\\
$x$        &$y$        &$0.05y^2$
&$w_{yy}=0.1$
&$          0$&$          0$&$     8\,401$
&$          0$&$          0$&$     8\,390$
&$          0$&$          0$&$     8\,333$
\\
$x+0.05x^2$&$y$        &$0.05x^2$
&$\left\{\begin{array}{@{}c@{}}u_{xx}=0.1\\w_{xx}=0.1\end{array}\right\}$
&$     4\,180$&$          0$&$     8\,405$
&$     4\,197$&$          0$&$     8\,391$
&$     4\,167$&$          0$&$     8\,333$
\\
$x+0.05x^2$&$y$        &$0.1xy$
&$\left\{\begin{array}{@{}c@{}}u_{xx}=0.1\\w_{xy}=0.1\end{array}\right\}$
&$     4\,172$&$          0$&$          0$
&$     4\,150$&$          0$&$          0$
&$     4\,167$&$          0$&$          0$
\\
$x+0.05x^2$&$y$        &$0.05y^2$
&$\left\{\begin{array}{@{}c@{}}u_{xx}=0.1\\w_{yy}=0.1\end{array}\right\}$
&$     4\,175$&$          0$&$     8\,401$
&$     4\,195$&$          0$&$     8\,389$
&$     4\,167$&$          0$&$     8\,333$
\\
\hline
\end{tabular}
\end{table}

\subsubsection{Nonlinear Dampening}
\label{sssec-cr-degen}

Referring to our discussion in Section~\ref{sssec-ojapde22-deg}
regarding the behaviour of multivariate median filters and
the corresponding PDEs for structuring elements $D_\varrho$
of nonvanishing radius $\varrho$, we turn to check by an
additional numerical experiment how an increase of the
second derivatives away from zero effects the median.

To this end, we compute bivariate
Oja medians, again with a structuring
element of radius $\varrho=1$ discretised with grid resolution $0.01$,
for functions with 
increasing values of the three second partial derivatives
$u_{xx}$, $u_{yy}$, $v_{xy}$ occurring on the right-hand side
of \eqref{ojapde22-I1}. The underlying functions are 
\begin{itemize}
\item $u=x+\frac12sx^2$, $v=y$ for the test of $u_{xx}$,
\item $u=x+\frac12sy^2$, $v=y$ for the test of $u_{yy}$, and
\item $u=x$, $v=y-sxy$ for the test of $v_{xy}$,
\end{itemize}
where $s$ is varied from $0$ to $2.5$.

Figure~\ref{fig-nonlinresponse} shows the $u$ components of the resulting
Oja medians dependent on the values of $s$.
For $s$ close to zero they increase linearly with
the ascents predicted by \eqref{ojapde22-I1}. Regarding $u_{xx}$ and
$u_{yy}$, the median values follow this linear ascent closely, 
within $10\,\%$ tolerance, up to $s\approx1$,
after which the values grow rapidly faster in the $u_{xx}$ case, and
are dampened in the $u_{yy}$ case.

In contrast, in the $v_{xy}$ case the deviation from linear behaviour
starts much earlier, leading to more than $10\,\%$ deviation already
for $s\approx0.6$, with the growth of the median rapidly being dampened
above this level. For large $s$, the effect of the coefficient $v_{xy}$
on the median even starts to decrease. The response of the median value
to $v_{xy}$ confirms the inherent nonlinear dampening effect of the
median filter procedure.

\begin{figure}[t!]
\unitlength0.0125\textwidth
\begin{picture}(80,31)
\put( 0.0, 0.0){\includegraphics[width=48\unitlength]{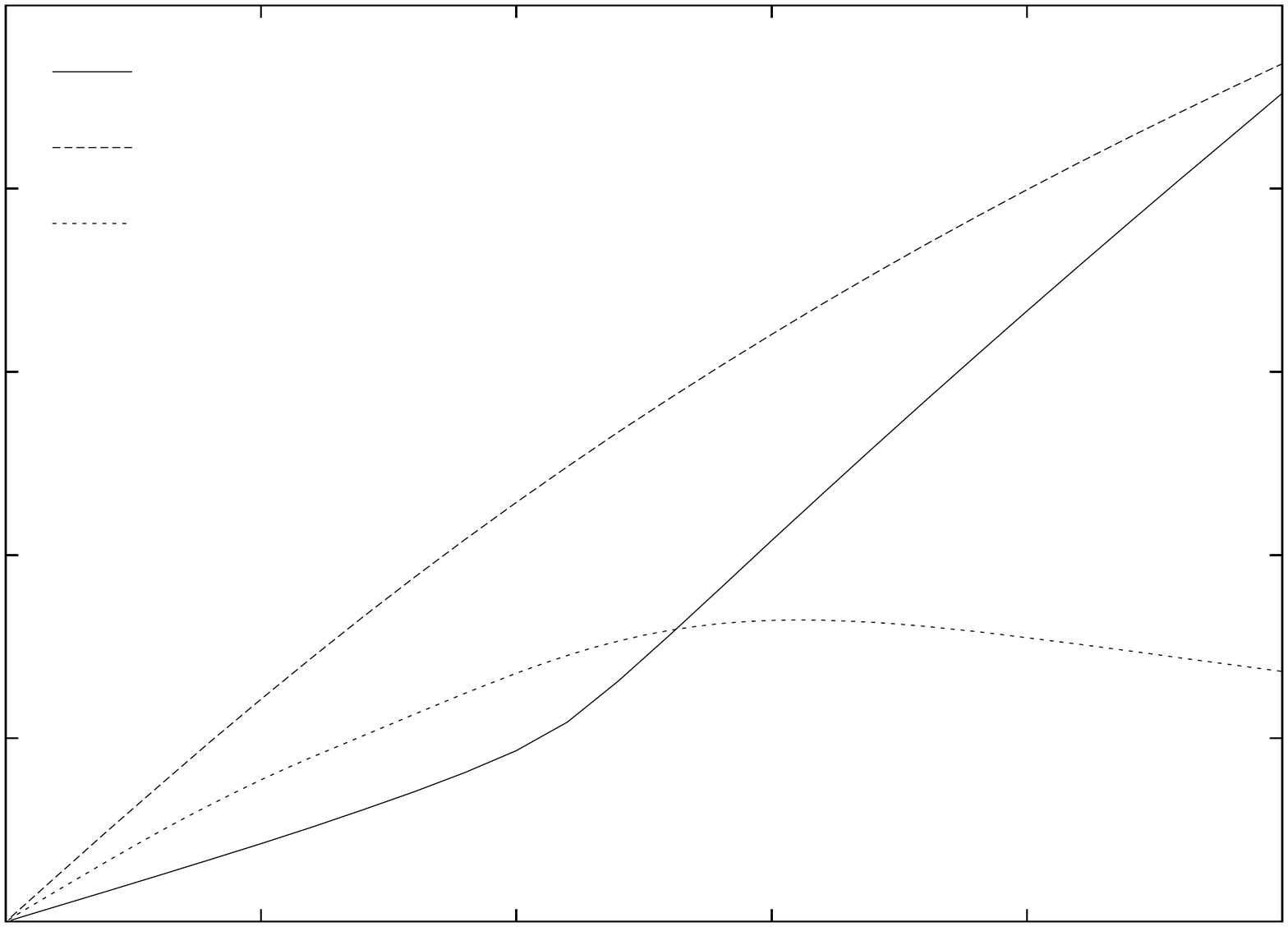}}
\put( 3.3, 3.6){\small$0.0$}
\put( 3.3, 8.6){\small$0.1$}
\put( 3.3,13.9){\small$0.2$}
\put( 3.3,19.1){\small$0.3$}
\put( 3.3,24.3){\small$0.4$}
\put( 3.3,29.5){\small$0.5$}
\put( 1.0, 6.0){\rotatebox{90}{Oja median ($u$ component)}}
\put( 5.6, 2.0){\small$0.0$}
\put(12.2, 2.0){\small$0.5$}
\put(19.5, 2.0){\small$1.0$}
\put(26.8, 2.0){\small$1.5$}
\put(34.1, 2.0){\small$2.0$}
\put(41.4, 2.0){\small$2.5$}
\put(18.0, 0.0){\small$u_{xx}$, $u_{yy}$, $-v_{xy}$}
\put(10.0,27.8){\small$u_{xx}$}
\put(10.0,25.7){\small$u_{yy}$}
\put(10.0,23.6){\small$v_{xy}$}
\put(48,0){\parbox[b]{32\unitlength}{%
\caption{\label{fig-nonlinresponse}Nonlinear response
of bivariate Oja median filter for fixed structuring element
radius $\varrho=1$ to Taylor coefficients $u_{xx}$, $u_{yy}$,
$v_{xy}$, where $\mathrm{D}\bm{u}=\bm{I}$.}}}
\end{picture}
\end{figure}

\subsubsection{A More Complex Bivariate Example}
\label{sssec-cr-exfn}

To demonstrate the validity of the PDE approximation results
of Proposition~\ref{prop-l1pde22}, Theorem~\ref{thm-ojapde22}, 
and Corollary~\ref{cor-l1apde22} also away from the
special case $\mathrm{D}\bm{u}=\bm{I}$, 
we consider a simple bivariate example function given by
\begin{align}
u(x,y)&= x^2\;, & v(x,y) &= \sqrt{x^2+y^2}\;.
\end{align}
Level sets of $u$ and $v$ for this function in the range
$[0,1]\times[0,1]$ are depicted in Figure~\ref{fig-exfn2}.
In this figure, also seven test locations a--g are depicted
together with structuring elements of radius $\varrho=0.1$
for which we compare in the following median filter values
with time steps $(\tau u_1,\tau v_t)$ of the respective
PDE counterparts.

We start with the $L^1$ median and the PDE 
\eqref{l1pde22}--\eqref{l1pde22-W} from 
Proposition~\ref{prop-l1pde22}.
Time steps $(\tau u_t,\tau v_t)$ of the PDE \eqref{l1pde22}
at the locations a--g were computed analytically, using
numeric integration for the integral values $Q_1(\lambda)$ and
$Q_2(\lambda)$. The time step size for \eqref{l1pde22} was
$\tau=\varrho^2/6=0.001\,667$.
For the computation of $L^1$ medians,
the structuring elements of radius $\varrho$ around
locations a--g were sampled at grid resolution $0.001$
resulting in approx.~$31\,000$ sample points for each location.
From their function values $(u,v)$ the $L^1$ median $(u^*,v^*)$
was computed by the gradient descent method. 
For comparison with the PDE time step the input
function value of the midpoint was subtracted. 
Table~\ref{tab-l1demo} shows PDE time steps, the corresponding
median filter updates $u^*-u$, $v^*-v$ and the relative errors
(in Euclidean norm) with respect to the PDE time steps, i.e.\
$\lvert(u^*-u-\tau u_t,v^*-v-\tau v_t)\rvert/\lvert\tau u_t,\tau v_t\rvert$.

For the Oja median, we proceed analogously, with the 
analytically computed PDE time steps of \eqref{ojapde22}, 
Oja median filter updates and their relative errors being shown in 
Table~\ref{tab-ojademo}.
The time step size for \eqref{ojapde22} was
$\tau=\varrho^2/24=0.000\,417$.
The Oja median values were computed using the gradient descent
method.
Moreover, Table~\ref{tab-ojademo} contains results of the affine
equivariant transformed $L^1$ median which according to
Corollary~\ref{cor-l1apde22} approximates the same PDE.

In Table~\ref{tab-l1demo}, the results for locations c--g
show relative errors 
below $3\,\%$, 
which are reasonable given the structuring element radius
$\varrho=0.1$ and the grid resolution.
The approximation at locations a and b is
less accurate. At these locations, the gradients of $u$ and $v$ are almost
aligned and not close to zero, making the Jacobian $\mathrm{D}\bm{u}$ 
ill-conditioned. Locations e and f where the gradient $\bm{\nabla}u$
is small and $\mathrm{D}\bm{u}$ therefore also ill-conditioned, create less
problems for the approximation. 

The results in Table~\ref{tab-ojademo} show that the approximation of
the PDE \eqref{ojapde22} by both the Oja median and the transformed
$L^1$ median is fairly accurate, with relative errors of less than $2\,\%$,
at locations b, c, d and g
where $\mathrm{det}\,\mathrm{D}\bm{u}$ is sufficiently different from
zero. 
Larger discrepancies are observed for locations a, e, and f
which are closer to the coordinate axes. Note that on the $x$
axis, $\mathrm{D}\bm{u}$ becomes singular due to
coinciding gradient directions for $u$ and $v$, while on the
$y$ axis it does so due to the vanishing of $\bm{\nabla} u$.

A comparison of Tables~\ref{tab-l1demo} and~\ref{tab-ojademo}
underlines that the standard $L^1$ median on one hand
and the Oja median and transformed $L^1$ median on the other hand
indeed differ substantially. For their very similar
results in the tests of Section~\ref{sssec-cr22} it was decisive that
the case $\mathrm{D}\bm{u}=\bm{I}$ was tested there. In contrast,
for our test function here the Jacobian is far away from the
unit matrix, not only in locations a, e, f where it is near the 
degenerate case but also in the fairly regular locations b--d and g.

A close look at Table~\ref{tab-ojademo} also makes clear that,
although they approximate the same PDE, the Oja median filter and
affine equivariant transformed $L^1$ median filter are not identical.
An analysis of the higher order terms neglected in the PDE approximation
analysis of Section~\ref{sec-pde} could shed more light on these
differences.

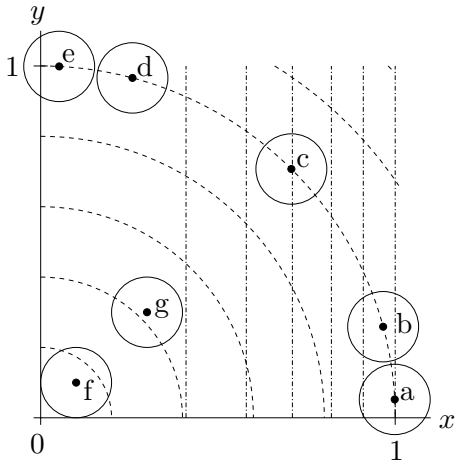
\begin{figure}[t!]
\unitlength0.003\textwidth
\centering
\begin{picture}(333,130)
\put(0,2){\includegraphics[width=120\unitlength]{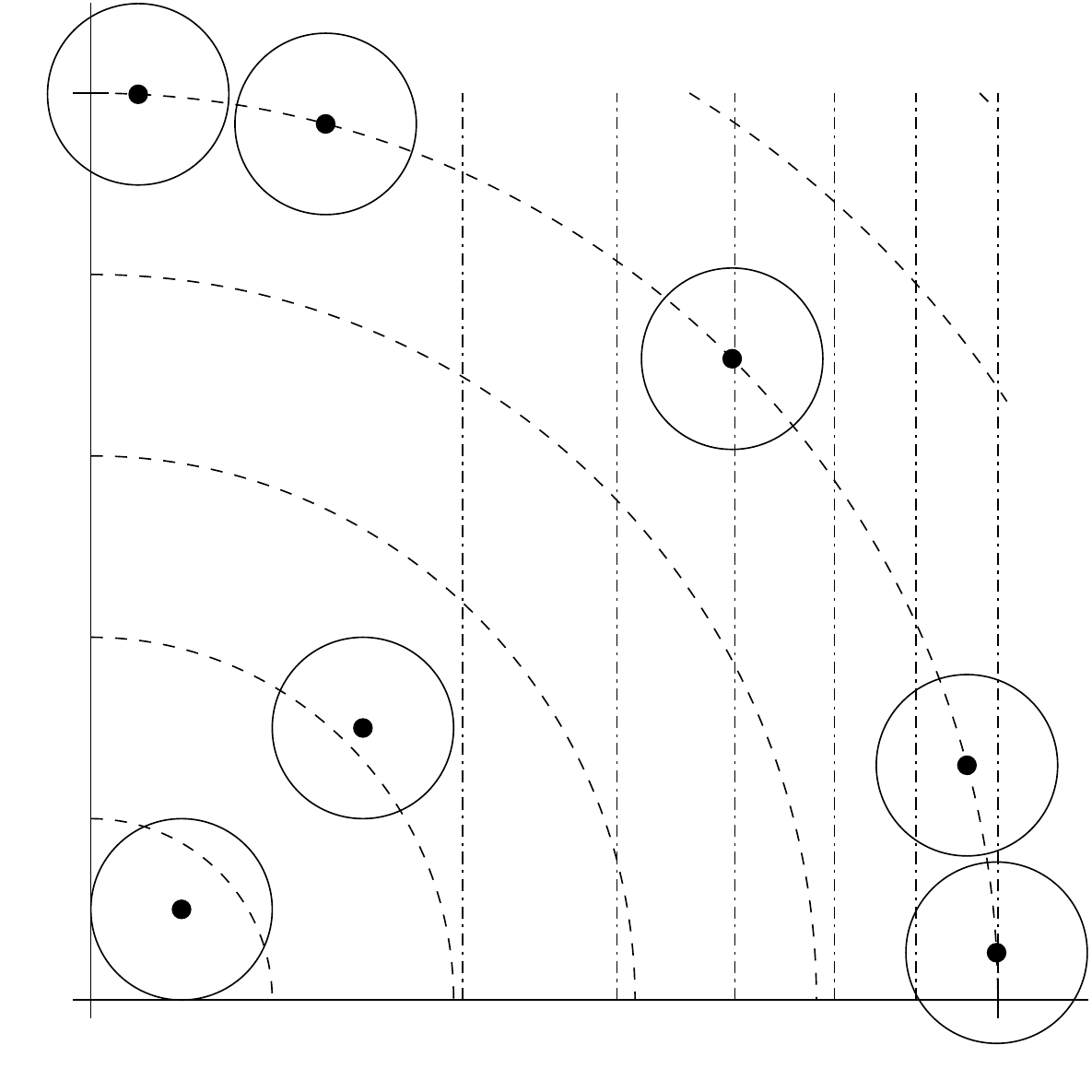}}
\put(111,17){a}
\put(110,36){b}
\put(82,83){c}
\put(37,109){d}
\put(16,113){e}
\put(22,17){f}
\put(42,42){g}
\put(7,2){$0$}
\put(0,109){$1$}
\put(108,0){$1$}
\put(122,9){$x$}
\put(7,125){$y$}
\put(140,0){\parbox[b]{193\unitlength}{%
\caption{\label{fig-exfn2}
Example function $(u,v)=(x^2,\sqrt{x^2+y^2})$ used to demonstrate
PDE approximation of bivariate median filters. Dot-dashed lines are level
lines $u=\mathrm{const}$, dashed lines are level lines
$v=\mathrm{const}$. Points a--g are the sample locations for
which numerical results are given in Table~\ref{tab-ojademo},
surrounded by their structuring elements as solid circles.
}}}
\end{picture}
\end{figure}

\begin{table}[b!]
\caption{\label{tab-l1demo}
Comparison of analytically computed time steps
$(\tau u_t,\tau v_t)$ of the PDE \eqref{l1pde22}--\eqref{l1pde22-W}
with numerical computation of the $L^1$ median $(u^*,v^*)$
for the function $(u,v)=(x^2,\sqrt{x^2+y^2})$.
To compute $(u^*,v^*)$, the structuring element of radius
$\varrho=0.1$ was sampled using a grid with spatial mesh size
$h=0.001$, generating about $31\,000$ data points.
The time step size for \eqref{l1pde22} was chosen as
$\tau=\varrho^2/6=0.001\,667$.
Medians and time steps are scaled by $10^6$ for more compact 
representation.
}
\medskip
\centering
\begin{tabular}{r@{~~}l@{~~}ll@{~~}lr@{~~}rr@{~~}r@{~~}r}
\hline
&
\multicolumn{2}{c}{\rule{0pt}{1.2em}Location}&
\multicolumn{2}{c}{Function value}&
\multicolumn{2}{c}{PDE time step}&
\multicolumn{3}{c}{$L^1$ median}\\
&\multicolumn{1}{c}{$x_0$}         
  &\multicolumn{1}{c}{$y_0$}
&\multicolumn{1}{c}{$u$}           
  &\multicolumn{1}{c}{$v$}
&\multicolumn{1}{c@{~}}{$10^6\tau u_t$}    
  &\multicolumn{1}{c}{$10^6\tau v_t$}
&\multicolumn{1}{c@{~}}{$10^6(u^*\!-\!u)$} 
  &\multicolumn{1}{c@{~}}{$10^6(v^*\!-\!v)$}
  &\multicolumn{1}{c}{rel. error}
\\ \hline
\rule{0pt}{1.2em}%
a)&$0.9986$&$0.0523$&$0.9973$&$1.0000$&
  $       320$&$    1\,045$&
  $       232$&$    1\,026$&$ 8.24\,\%$
\\
b)&$0.9659$&$0.2588$&$0.9330$&$1.0000$&
  $       574$&$       754$&
  $       539$&$       738$&$ 4.06\,\%$
\\
c)&$0.7071$&$0.7071$&$0.5000$&$1.0000$&
  $       942$&$       628$&
  $       926$&$       613$&$ 1.94\,\%$
\\
d)&$0.2588$&$0.9659$&$0.0670$&$1.0000$&
  $    1\,338$&$    1\,095$&
  $    1\,337$&$    1\,083$&$ 0.70\,\%$
\\
e)&$0.0523$&$0.9986$&$0.0027$&$1.0000$&
  $    2\,120$&$    1\,579$&
  $    2\,197$&$    1\,562$&$ 2.98\,\%$
\\
f)&$0.1000$&$0.1000$&$0.0100$&$0.1414$&
  $    2\,098$&$   10\,715$&
  $    2\,071$&$   10\,446$&$ 2.48\,\%$
\\
g)&$0.3000$&$0.3000$&$0.0900$&$0.4243$&
  $    1\,657$&$    2\,593$&
  $    1\,647$&$    2\,577$&$ 0.61\,\%$
\\
\hline
\end{tabular}
\end{table}

\begin{table}[b!]
\caption{\label{tab-ojademo}
Comparison of 
numerical computed Oja median and transformed $L^1$ median
with
analytically computed time steps
$(\tau u_t,\tau v_t)$ of the PDE \eqref{ojapde22}--\eqref{ojapde22-B}
for the function $(u,v)=(x^2,\sqrt{x^2+y^2})$.
To compute medians $(u^*,v^*)$, the structuring element of radius
$\varrho=0.1$ was sampled using a grid with spatial mesh size
$h=0.001$, generating about $31\,000$ data points.
The time step size for \eqref{ojapde22} was chosen as
$\tau=\varrho^2/24=0.000\,417$.
Medians and time steps are scaled by $10^6$ for more compact 
representation.
}
\bigskip
\centering
\small
\begin{tabular}{r@{~~}l@{~}l@{~~~}l@{~}l@{~~~}r@{~}r@{~~~}r@{~}r@{~~}r@{~~~}r@{~}r@{~}r}
\hline
&
\multicolumn{2}{@{}c@{~~~}}{\rule{0pt}{1.2em}Location}&
\multicolumn{2}{@{}c@{~~~}}{Function val.}&
\multicolumn{2}{@{}c@{}}{PDE\,time\,step}&
\multicolumn{3}{@{}c@{~~~}}{Oja median}&
\multicolumn{3}{@{}c}{Transformed $L^1$ median}\\
&\multicolumn{1}{@{}c@{~}}{$x_0$}        
&\multicolumn{1}{@{}c@{~~~}}{$y_0$}
&\multicolumn{1}{@{}c@{~}}{$u$}          
&\multicolumn{1}{@{}c@{~~~}}{$v$}
&\multicolumn{1}{@{}c@{}}{$\tau u_t$}   
&\multicolumn{1}{@{}c@{}}{$\tau v_t$}
&\multicolumn{1}{@{}c@{~}}{$(u^*\!-\!u)$}
&\multicolumn{1}{@{~}c@{}}{$(v^*\!-\!v)$}
&\multicolumn{1}{@{}c@{~~~}}{rel.}
&\multicolumn{1}{@{}c@{~}}{$(u^*\!-\!u)$}
&\multicolumn{1}{@{}c@{}}{$(v^*\!-\!v)$}
&\multicolumn{1}{@{}c@{}}{rel.}
\\
&&&&
&\multicolumn{1}{@{}c@{}}{$\times10^6$} 
&\multicolumn{1}{@{}c@{}}{$\times10^6$}
&\multicolumn{1}{@{}c@{}}{$\times10^6$} 
&\multicolumn{1}{@{}c@{}}{$\times10^6$} 
&\multicolumn{1}{@{}c@{~~~}}{error}
&\multicolumn{1}{@{}c@{}}{$\times10^6$} 
&\multicolumn{1}{@{}c@{}}{$\times10^6$} 
&\multicolumn{1}{@{}c@{}}{error}
\\ \hline
a)&$0.9986$&$0.0523$&$0.9973$&$1.0000$&
  $    2\,495$&$       417$&
  $    1\,896$&$       538$& $24.16\,\%$ &
  $    2\,138$&$       637$& $16.58\,\%$
\\
b)&$0.9659$&$0.2588$&$0.9330$&$1.0000$&
  $    2\,388$&$       417$&
  $    2\,355$&$       417$& $ 1.36\,\%$ &
  $    2\,355$&$       413$& $ 1.37\,\%$
\\
c)&$0.7071$&$0.7071$&$0.5000$&$1.0000$&
  $    1\,667$&$       417$&
  $    1\,650$&$       404$& $ 1.25\,\%$ &
  $    1\,652$&$       403$& $ 1.19\,\%$
\\
d)&$0.2588$&$0.9659$&$0.0670$&$1.0000$&
  $       945$&$       417$&
  $       943$&$       407$& $ 0.99\,\%$ &
  $       948$&$       409$& $ 0.83\,\%$
\\
e)&$0.0523$&$0.9986$&$0.0027$&$1.0000$&
  $       838$&$       417$&
  $       920$&$       448$& $ 9.37\,\%$ &
  $    1\,056$&$       512$& $25.41\,\%$
\\
f)&$0.1000$&$0.1000$&$0.0100$&$0.1414$&
  $    1\,667$&$    2\,946$&
  $    1\,587$&$    3\,668$& $21.46\,\%$ &
  $    1\,689$&$    3\,751$& $23.79\,\%$
\\
g)&$0.3000$&$0.3000$&$0.0900$&$0.4243$&
  $    1\,667$&$       982$&
  $    1\,654$&$    1\,009$& $ 1.55\,\%$ &
  $    1\,666$&$    1\,003$& $ 1.09\,\%$
\\
\hline
\end{tabular}
\end{table}

\subsubsection{Three-Channel Volume Images, Case $\mathrm{D}\bm{u}=\bm{I}$}
\label{sssec-cr33}

For the three-channel case, we consider the case 
$\mathrm{D}\bm{u}(\bm{0})=\bm{I}$ as treated in 
Lemmas~\ref{lem-ojapde33-I} and \ref{lem-l1apde33-I}.

We discretise sample functions $u(x,y,z)$, $v(x,y,z)$, $w(x,y,z)$ in the
structuring element $B_1(\bm{0})$, i.e.\ the ball of radius $1$
around the origin, with a grid resolution of $0.15$ in the $x$, $y$ and
$z$ directions, which yields $1237$ sample points. The coarser
resolution compared to Section~\ref{sssec-cr22} is a tribute to the
unfavourable computational complexity of our three-channel Oja
median computation.

Again, it suffices for symmetry reasons to consider sample functions 
that test only the weights of the second derivatives of $u$
while leaving $v(x,y,z)\equiv y$ and $w(x,y,z)\equiv z$. 
For these input data we compute the Oja and $L^1$ medians and compare 
these with the theoretical values given by the right-hand side of
\eqref{ojapde33-I-1}--\eqref{ojapde33-I-3} with the time
step size $1/20$. 
The results can be found in 
Table~\ref{tab-l1oja33-coeffcomp}. 
The observed deviations in the range of $2\times10^{-4}$
between the computed Oja medians and PDE time steps are expectable given the
grid resolution.
For the $L^1$ median, larger deviations up to $8\times10^{-4}$ are observed.
However, doing the same computation with a finer sampling grid
-- which is computationally feasible with our implementation of the $L^1$
median but not for the Oja median -- yields values also for the $L^1$ median
that match the time step of \eqref{ojapde33-I-1}--\eqref{ojapde33-I-3} closely,
thereby confirming also the asymptotic equivalence of the Oja and
affine equivariant transformed $L^1$ filter for three-channel volume
images.

\begin{table}[b!]
\caption{\label{tab-l1oja33-coeffcomp}
Validation of the PDE approximation of three-channel Oja median
filtering in the case $\mathrm{D}\bm{u}=\bm{I}$,
see Lemma~\ref{lem-ojapde33-I}, including for comparison also 
the $L^1$ median.
Median values $(u^*,v^*,w^*)$
computed from functions sampled with resolution $0.15$
in a ball-shaped structuring element of radius $1$ are juxtaposed with
the time steps of size $\tau=1/20$ of the corresponding PDE system.
Medians and time steps are scaled by $10^4$ for more compact 
representation.
}
\bigskip
\centering
\small
\begin{tabular}
{l@{~}l@{~}l@{~~~}r@{\;}l@{~~~}r@{~}r@{~}r@{~~~}r@{~}r@{~}r@{~~~}r@{~}r@{~}r}
\hline
\multicolumn{3}{@{}c@{~~~}}{\rule{0pt}{1.2em}Function}&
\multicolumn{2}{@{}c@{~~~}}{Nonzero 2\textsuperscript{nd}}&
\multicolumn{3}{@{}c@{~~~}}{$L^1$ median}&
\multicolumn{3}{@{}c@{~~~}}{Oja median}&
\multicolumn{3}{@{}c@{~~~}}{PDE time step}\\
$u$&$v$&$w$&\multicolumn{2}{@{}c@{~~~}}{derivatives}
&\multicolumn{1}{@{}c@{~}}{$10^4u^*$}
&\multicolumn{1}{@{}c@{~}}{$10^4v^*$}
&\multicolumn{1}{@{}c@{~~~}}{$10^4w^*$}
&\multicolumn{1}{@{}c@{~}}{$10^4u^*$}
&\multicolumn{1}{@{}c@{~}}{$10^4v^*$}
&\multicolumn{1}{@{}c@{~~~}}{$10^4w^*$}
&\multicolumn{1}{@{}c@{~}}{$10^4\tau u_t$}
&\multicolumn{1}{@{}c@{~}}{$10^4\tau v_t$}
&\multicolumn{1}{@{}c}{$10^4\tau w_t$}\\\hline
$x+0.05x^2$&$y$&$z$&$\quad u_{xx}$&$=0.1$
&$    42$&$     0$&$     0$
&$    48$&$     0$&$     0$
&$    50$&$     0$&$     0$
\\
$x+0.1xy$&$y$&$z$&$\quad u_{xy}$&$=0.1$
&$     0$&$-    42$&$     0$
&$     1$&$-    48$&$     0$
&$     0$&$-    50$&$     0$
\\
$x+0.1xz$&$y$&$z$&$\quad u_{xz}$&$=0.1$
&$     0$&$     0$&$-    42$
&$     1$&$     0$&$-    48$
&$     0$&$     0$&$-    50$
\\
$x+0.05y^2$&$y$&$z$&$\quad u_{yy}$&$=0.1$
&$    93$&$     0$&$     0$
&$    99$&$     0$&$     0$
&$   100$&$     0$&$     0$
\\
$x+0.1yz$&$y$&$z$&$\quad u_{yz}$&$=0.1$
&$     0$&$     0$&$     0$
&$     0$&$     0$&$     0$
&$     0$&$     0$&$     0$
\\
$x+0.05z^2$&$y$&$z$&$\quad u_{zz}$&$=0.1$
&$    93$&$     0$&$     0$
&$    99$&$     0$&$     0$
&$   100$&$     0$&$     0$
\\
\hline
\end{tabular}
\end{table}

\subsection{Iterated Median Filters and PDE Evolution}
\label{ssec-evo}

In our final experiment, we return to the filtering of RGB images
and make now the transition to iterated median filtering.
In these experiments, a numerical scheme for the PDE
\eqref{ojapde23}--\eqref{ojapde23-B} is used. We start therefore
with a brief description of this scheme.

\subsubsection{Numerical Approximation of the Affine Equivariant
Median PDE}
\label{ssec-pdenum}

We assume that the three-channel input image $\bm{f}$ for the PDE
\eqref{ojapde23}--\eqref{ojapde23-B} is sampled on an isotropic regular
grid with spatial step size $h$ in the $x$ and $y$ directions, and denote
by $\bm{f}_{i,j}$ the intensity triple at pixel $(i,j)$.
We will compute by an explicit finite-difference scheme a sequence
$(\bm{u}^k)$ of filtered images that approximate the PDE at evolution times
$k\tau$ with time step size $\tau$, with $\bm{u}^0\equiv\bm{f}$.
By $\bm{u}_{i,j}^k=(u_{i,j}^k,v_{i,j}^k,w_{i,j}^k)\transpose$ 
we denote the value of pixel $(i,j)$ in the $k$-th iteration.

In computing $\bm{u}_{i,j}^{k+1}$ from the previous image $\bm{u}^k$
we use the pixels $\bm{u}_{i',j'}^k$ from the $3\times3$ patch 
$\mathcal{P}_{i,j}$ given by $i'\in\{i-1,i,i+1\}$ and $j'\in\{j-1,j,j+1\}$.

The numerical scheme transforms the input data $\bm{u}$ within each patch
$\mathcal{P}_{i,j}$ by an orthogonal transform 
$\bm{u}=(u,v,w)\mapsto\hat{\bm{u}}=(\hat{u},\hat{v},\hat{w})$ of the values
and determining a new orthogonal basis $(\bm{\eta},\bm{\xi})$ in the
$(x,y)$ plane such that $\hat{u}_{\bm{\eta}}$ and $\hat{v}_{\bm{\xi}}$ are
the only nonzero entries of the Jacobian 
$\mathrm{D}_{\bm{\eta}\bm{\xi}}\hat{\bm{u}}$ w.r.t.\ the new coordinates
at pixel $(i,j)$. The PDE to be approximated then reads
\begin{align}
\begin{pmatrix}\hat{u}_t\\\hat{v}_t\\\hat{w}_t\end{pmatrix}
&= 
\underbrace{\begin{pmatrix}\hat{u}_{xx}+\hat{u}_{yy}\\
\hat{v}_{xx}+\hat{v}_{yy}\\\hat{w}_{xx}+\hat{w}_{yy}\end{pmatrix}}_{\bm{z}_1}
+ \underbrace{2\begin{pmatrix}\hat{u}_{\bm{\xi\xi}}\\\hat{v}_{\bm{\eta\eta}}\\0
\end{pmatrix}}_{\bm{z}_2}
- \underbrace{2\begin{pmatrix}
\hat{u}_{\bm{\eta}}\hat{v}_{\bm{\eta\xi}}/\hat{v}_{\bm{\xi}}\\
\hat{v}_{\bm{\xi}}\hat{u}_{\bm{\eta\xi}}/\hat{u}_{\bm{\eta}}\\0
\end{pmatrix}}_{\bm{z}_3}
\;.
\label{ojapde23-euctfmd}
\end{align}
Herein, the first contribution $\bm{z}_1$ is approximated by
central differences even in the original $(x,y)$ coordinates.
The second contribution $\bm{z}_2$ is approximated by central 
differences in the $(\bm{\eta},\bm{\xi})$ basis.
For the third contribution $\bm{z}_3$, such a discretisation would be
unstable and also unable to cope with locations where $\hat{u}_{\bm{\eta}}$
or $\hat{v}_{\bm{\xi}}$ vanishes. Therefore, two stabilisations are used.
First, the weight factor $\hat{v}_{\bm{\eta\xi}}/\hat{v}_{\bm{\xi}}$
is approximated by the regularised expression 
$R_v:=\hat{v}_{\bm{\eta\xi}}\hat{v}_{\bm{\xi}}/(\hat{v}_{\bm{\xi}}^2
+\varepsilon)$ with a fixed numerical regularisation parameter $\varepsilon$
using central differences in the numerator and
a combination of minmod-stabilised 
one-sided differences in the denominator; a similar
expression $R_u$ is used for $\hat{u}_{\bm{\eta\xi}}/\hat{u}_{\bm{\eta}}$.
Second, the factor $u_{\bm{\eta}}$ in the first component is discretised
in an upwind way by choosing a one-sided difference according to the
sign of $R_v$; analogously for $v_{\bm{\xi}}$ in the second component.

For utmost explicitness, the scheme is stated as a detailed algorithm 
in Appendix~\ref{app-pdealgo}.

\subsubsection{Image Filtering Experiment}
\label{ssec-exp-itmed}

\begin{figure}[t!]
\unitlength0.0048\textwidth
\begin{picture}(208,126)
\put(  0.0,94)
{\includegraphics[width=32\unitlength]{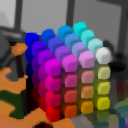}}
\put( 34.0,94)
{\includegraphics[width=32\unitlength]{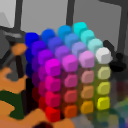}}
\put( 68.0,94)
{\includegraphics[width=32\unitlength]{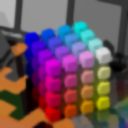}}
\put(  1.1,96)
{\colorbox{white}{\rule{0pt}{.6em}\hbox to.6em{\kern.1em\smash{a}}}}
\put( 35.1,96)
{\colorbox{white}{\rule{0pt}{.6em}\hbox to.6em{\kern.1em\smash{b}}}}
\put( 69.1,96)
{\colorbox{white}{\rule{0pt}{.6em}\hbox to.6em{\kern.1em\smash{c}}}}
\put(  0.0,60)
{\includegraphics[width=32\unitlength]{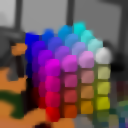}}
\put( 34.0,60)
{\includegraphics[width=32\unitlength]{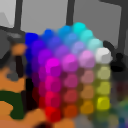}}
\put( 68.0,60)
{\includegraphics[width=32\unitlength]{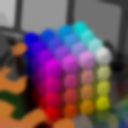}}
\put(  1.1,62)
{\colorbox{white}{\rule{0pt}{.6em}\hbox to.6em{\kern.1em\smash{d}}}}
\put( 35.1,62)
{\colorbox{white}{\rule{0pt}{.6em}\hbox to.6em{\kern.1em\smash{e}}}}
\put( 69.1,62)
{\colorbox{white}{\rule{0pt}{.6em}\hbox to.6em{\kern.1em\smash{f}}}}
\put(  0.0,26)
{\includegraphics[width=32\unitlength]{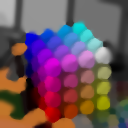}}
\put( 34.0,26)
{\includegraphics[width=32\unitlength]{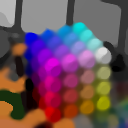}}
\put( 68.0,26)
{\includegraphics[width=32\unitlength]{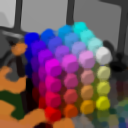}}
\put(  1.1,28)
{\colorbox{white}{\rule{0pt}{.6em}\hbox to.6em{\kern.1em\smash{g}}}}
\put( 35.1,28)
{\colorbox{white}{\rule{0pt}{.6em}\hbox to.6em{\kern.1em\smash{h}}}}
\put( 69.1,28)
{\colorbox{white}{\rule{0pt}{.6em}\hbox to.6em{\kern.1em\smash{i}}}}
\put(108.0,18){\parbox[b]{100\unitlength}{%
\caption{\label{fig-co01-evo}
Comparison of iterated median filtering of the \emph{Colors} test image
from Figure~\ref{fig-co01} (a)
using affine equivariant multivariate medians and
PDE filtering of the same image.
\textbf{Top row, left to right:}
\textbf{(a)} 2D Oja median, $\varrho=2$, with input regularisation, 
$3$ iterations. --
\textbf{(b)} Affine equivariant transformed $L^1$ median, $\varrho=2$,
$3$ iterations. --
\textbf{(c)} PDE \eqref{ojapde23}, explicit finite difference scheme,
$10$ iterations with time step size $\tau=0.05$.~~--~~
\textbf{Middle row, left to right:}
\textbf{(d)} 2D Oja median, $\varrho=2$, with input regularisation, 
$9$ iterations. --
\textbf{(e)} Affine equivariant transformed $L^1$ median, $\varrho=2$,
$9$ iterations. --
\textbf{(f)} PDE \eqref{ojapde23}, explicit finite difference scheme,
$30$ iterations with time step size $\tau=0.05$.~~--~~
\textbf{Bottom}
\textbf{row, left to right:}~~
\textbf{(g)}~~ 2D~~ Oja~~ median,
}}}
\put(0,0){\parbox[b]{208\unitlength}{%
$\varrho=3$, 
with input regularisation, 
$4$ iterations. --
\textbf{(h)} Affine equivariant transformed $L^1$ median, $\varrho=3$,
$4$ iterations. --
\textbf{(i)} PDE \eqref{ojapde23}, explicit finite difference scheme,
$30$ iterations with time step size $\tau=0.05$, with heuristic
anti-diffusion to compensate for numerical dissipation.
}}
\end{picture}
\end{figure}
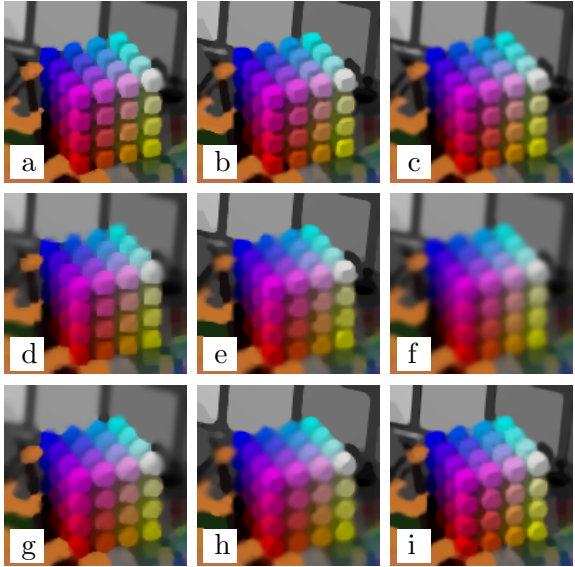

Using the algorithms described so far in the paper, we compute
iterated median filters of the RGB image from Figure~\ref{fig-co01} (a)
and their supposed PDE evolution counterpart.
Results are shown in Figure~\ref{fig-co01-evo}.

In the first row, the filter parameters of the 2D Oja median filter
and the affine equivariant transformed $L^1$ median filter are adjusted
such as to correspond to an evolution time $T=0.5$ of the PDE system
\eqref{ojapde23}. To this end, a structuring element of radius $\varrho=2$
is used, and $3$ iterations of both median filters carried out, see the
results in Figure~\ref{fig-co01-evo} (a) and (b).
In frame (c), the result from the numerical evaluation of the PDE
is shown. With time step size $\tau=0.05$ and $10$ iterations this 
represents also the evolution time $T=0.5$.

In the second row of Figure~\ref{fig-co01-evo}, frames (d)--(f),
the same filters are
shown for an evolution time of $T=1.5$, i.e.\ $9$ median iterations and 
$30$ time steps, respectively.
Regarding the structure simplification by rounding contours etc.,
the results for the same evolution time are largely comparable,
with the transformed $L^1$ median featuring the sharpest
preservation of edges (with exception of a few structures where the
Oja median result appears sharper). The PDE results are visibly more
blurred. It can be conjectured that this blur is not intrinsic to the
PDE but to the numerical dissipation that usually comes with explicit
finite difference schemes for curvature-based PDEs.

Figure~\ref{fig-co01-evo} (g) and (h) show median filtering results
for the same evolution time $T=1.5$, but this time realised with 
structuring element radius $\varrho=3$ and $4$ iterations. 
The sharpness and 
overall degree of structure simplification is fairly comparable
with frames (d) and (e), which confirms that indeed the progress of the
filtering process scales with $\varrho^2$ as suggested
by the approximation theorem. Some corners are being rounded more
pronouncedly with the larger structuring element 
(see for example the grey tiles in the background).

Based on the assumption that the higher amount of blur in the PDE results
so far is caused by numerical dissipation inherent to the finite-difference
discretisation, one might think of modifying the numerical scheme by 
sharpening terms that compensate for this dissipation, see the
\emph{flux-corrected transport} approach established in \cite{Boris-JCP73}
and used in image processing e.g.\ in \cite{Breuss-ss05} for hyperbolic
PDEs. Of course, a well-founded modification of the numerical scheme from
Section~\ref{ssec-pdenum} would require a detailed analysis of its
approximation errors, which we cannot provide at this point.
However, the PDE under consideration offers a simple way to test this
idea on a heuristic level. To see this, note that the PDE 
\eqref{ojapde23-euctfmd} includes the isotropic (forward) diffusion term
$\bm{z}_1=\hat{\bm{u}}_{xx}+\hat{\bm{u}}_{yy}$. Let us therefore introduce 
inverse linear diffusion $-\lambda\bm{z}_1$ with an anti-diffusion weight 
$\lambda>0$ as a heuristic flux correction. 
This is tantamount to just reducing the weight of $\bm{z}_1$ in
\eqref{ojapde23-euctfmd} from $1$ to $1-\lambda$.
As long as $\lambda\le1$, the net linear diffusion $(1-\lambda)\bm{z}_1$
is forward diffusion, thus not harming the stability of the numerical 
scheme.
In Figure~\ref{fig-co01-evo} (i) we present the result of filtering
the test image with the so-modified scheme with $\lambda=1$, i.e.\
completely suppressing $\bm{z}_1$. The filtered image is fairly similar
to the median filtering results in frames (e) and (h) regarding sharpness
and contour simplification. Regarding those details which are filtered
more pronouncedly in frames (g) and (h) than in (d) and (e), visual
inspection places the modified PDE result (i) closer to (d) and (e),
which is natural given that the approximation of the PDE by median
filtering is asymptotic for $\varrho\to0$.

\section{Summary and Outlook}\label{sec-summ}

In this paper, we have analysed multivariate median filters in
a space-continous setting with emphasis on their asymptotic behaviour.
We have considered $L^1$,
2D and 3D Oja median filters and affine equivariant transformed 
(transformation--retransformation) $L^1$ median filters for 
bivariate planar images, three-channel volume images and three-channel
planar images. In all these cases, we have derived PDEs approximated
by multivariate median filters in the limit for vanishing radius of
the structuring element. We have verified these PDE approximation
statements by numerical experiments. 

An important outcome of our
analysis is that the Oja median filter and the affine equivariant transformed
$L^1$ median filter are asymptotically equivalent in relevant settings.
The iterated Oja median filter, the transformed $L^1$ median filter and the 
corresponding PDE can therefore be considered as different approximations to 
the same kind of ideal \emph{affine equivariant median filter}.

Future work on the theoretical side might be directed at obtaining a more
general form of the approximation statements, such as uniform representations
of PDEs for median filtering of $n$-dimensional data over $m$-dimensional 
domains, including affine equivariant transformed $L^1$ and different
$k$-dimensional Oja medians.
The numerical scheme from Section~\ref{ssec-pdenum}, while working in 
experiments, still lacks a detailed stability analysis. As pointed out
in Section~\ref{ssec-exp-itmed}, it would also be of interest to analyse the
numerical dissipation in this scheme by studying the approximation errors
of its finite difference approximations, in order to formulate a 
theoretically well-founded corrected scheme instead of the heuristic
anti-diffusion approach used in Figure~\ref{fig-co01-evo} (i).

Regarding the implementation of multivariate median filtering, more efficient
algorithms for Oja median filtering should be investigated along the lines
sketched in Section~\ref{ssec-demo-num}. In the light of the above-mentioned
asymptotic equivalence of affine equivariant multivariate medians, however, 
using the transformation--retransformation $L^1$ median appears as a viable
alternative.

Finally, the results of the present paper may open different avenues to a 
broader application of multivariate median filters in image processing.
On one hand, based on a proper theoretical understanding of its effect,
affine equivariant (Oja or transformed $L^1$) median filtering can be studied
in practical image processing applications to find out more about its
practical advantages or disadvantages. On the other hand, although the PDE
approximated by affine equivariant median filters is not quite as simple as
the mean curvature motion equation approximated by univariate median filtering,
its geometric contributions are also explicit enough to raise the expectation
that medians can be used as a building block in nonstandard numerical 
approximations of multivariate curvature-based PDEs.

On a wider horizon, a further topic of interest for future research is whether
also other multivariate median concepts from statistical literature, which
generalise other properties of the univariate median than the distance sum 
minimisation, can be incorporated into the theoretical framework and made
useful for image processing.

\appendix

\section{First Proof of Lemma~\ref{lem-ojapde22-I}}
\label{app-22proof1}

We restate here the proof from \cite{Welk-ssvm15} with slight
modifications and additional details.

The Taylor expansion of $(u,v)$ up to second order around $(0,0)$
reads as
\begin{align}
\begin{pmatrix}u(x,y)\\v(x,y)\end{pmatrix}
&= \begin{pmatrix}x\\y\end{pmatrix} + 
\begin{pmatrix}
\alpha_1 x^2+\beta_1 xy+\delta_1 y^2\\
\alpha_2 x^2+\beta_2 xy+\delta_2 y^2
\end{pmatrix}\;,
\label{ojataylor22}
\end{align}
where the coefficients are given by derivatives of $u$, $v$ at $(x,y)=(0,0)$
as
\begin{align}
\alpha_1&=\tfrac12u_{xx}(0,0)\;,&
\beta_1&=u_{xy}(0,0)\;,&
\delta_1&=\tfrac12u_{yy}(0,0)\;,
\label{abg}\\
\alpha_2&=\tfrac12v_{xx}(0,0)\;,&
\beta_2&=v_{xy}(0,0)\;,&
\delta_2&=\tfrac12v_{yy}(0,0)\;.
\label{dez}
\end{align}

Restating the definition of Oja's simplex median for continuous data sets
with density function $f(u,v)$, we seek the point $M:=(u^*,v^*)$ which
minimises the integral over all areas of triangles $MAB$ with
$A=(u_1,v_1)$ and $B=(u_2,v_2)$ with
$(u_1,v_1)=\bigl(u(x_1,y_1),v(x_1,y_1)\bigr),
\bigl(u_2,v_2)=(u(x_2,y_2),v(x_2,y_2)\bigr)$,
$(x_1,y_1),(x_2,y_2)\in D_{\varrho}(0,0)$, weighted with the
density $f(u_1,v_1)f(u_2,v_2)$.

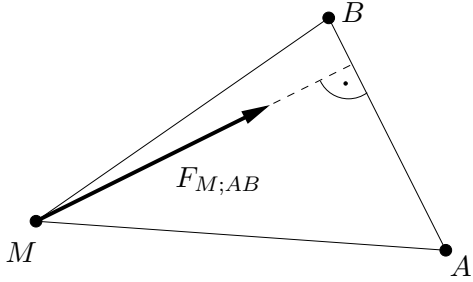
\begin{figure}[t!]
\centering
\unitlength0.004\textwidth
\begin{picture}(250,61)
\put(5,4){\includegraphics[width=90\unitlength]{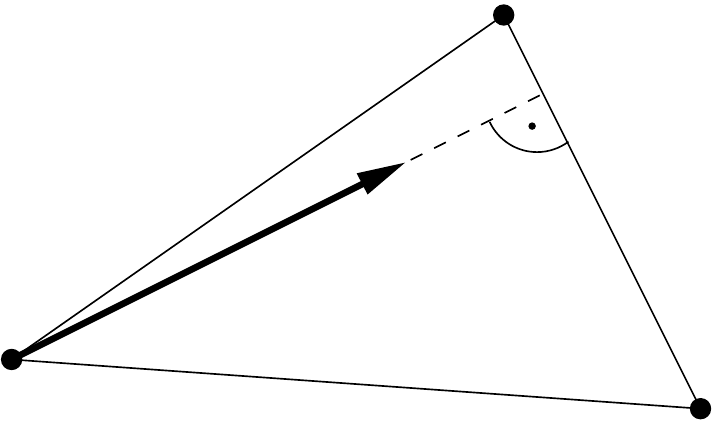}}
\put(0,3){$M$}
\put(94,0){$A$}
\put(71,54){$B$}
\put(36,19){$F_{M;AB}$}
\put(106,0){\parbox[b]{144\unitlength}{%
\caption{
\label{fig-ojatriangle}
Anti-gradient vector $F_{M;AB}$ for the area of a triangle $MAB$ with variable
point $M$. From~\cite{Welk-ssvm15}.
}}}
\end{picture}
\end{figure}

For each triangle $MAB$, the negative gradient of
its area as function of $M$ is a force vector $\tfrac12 F_{M;AB}$
where $F_{M;AB}$ is
perpendicular to $AB$ with a
length proportional to the length $\lvert AB\rvert$,
see Figure~\ref{fig-ojatriangle}. Assuming that $MAB$ is
positively oriented, this vector equals $(v_2-v_1,-u_2+u_1)$.

Sorting the pairs $(A,B)$ by the orientation angles $\varphi$ of the lines
$F_{M;AB}$, we see that the minimisation condition for the Oja median
can be expressed as
\begin{equation}
\bm{\varPhi}(u^*,v^*) =
\frac14
\int_{0}^{2\pi} 
\begin{pmatrix}\cos\varphi\\\sin\varphi\end{pmatrix}
F(u^*,v^*,\varphi) 
\,\mathrm{d}\varphi = 0\;.
\label{Phiint}
\end{equation}
Here, $F(\varphi)$ is essentially the resultant of all forces $F_{M;AB}$
for which the line $AB$ intersects the ray from $M$ in direction
$(\cos\varphi,\sin\varphi)$ perpendicularly. Each force $F_{M;AB}$
is weighted with the combined density $f(A)f(B)=f(u_1,v_1)f(u_2,v_2)$.

The factor $1/4$ in front of the integral \eqref{Phiint} combines the
factor $1/2$ from the force vector mentioned above with another
factor $1/2$ to compensate
that each triangle $MAB$ enters the integral twice
(once as $MAB$ and once as $MBA$, where the orientation factor
cancels by squaring). Note that in \cite{Welk-ssvm15} the integral
was stated differently, integrating only over the triangles with
positive orientation.

Moreover, $u^*,v^*$ will be of
order $\mathcal{O}(\varrho)$ (in fact, even $\mathcal{O}(\varrho^2)$).
Thus, $(u^*,v^*)$ can be expressed up to higher order terms via
linearisation as
\begin{align}
\begin{pmatrix}u^*\\v^*\end{pmatrix} &=
-\bigl(\mathrm{D}\bm{\varPhi}(0,0)\bigr)^{-1}\bm{\varPhi}(0,0)\;.
\label{Philin}
\end{align}

We therefore turn now to derive an expression for $F(0,0,\varphi)$.
Considering first $\varphi=0$, this means that all point pairs
$(A,B)$ in the $u$-$v$ right half-plane with $u_1=u_2$ contribute to
$F(0,0,0)$,
yielding
\begin{align}
F(0,0,0) &=
\int_{0}^{+\infty}
\int_{-\infty}^{+\infty}
\int_{-\infty}^{+\infty} f(u,v_1) f(u,v_2) 
(v_2-v_1)^2
\,\mathrm{d}v_2
\,\mathrm{d}v_1
\,\mathrm{d}u \;.
\label{F0}
\end{align}
Note that the factor $(v_2-v_1)$ occurs squared in the integrand. One factor
$\lvert v_2-v_1\rvert$ originates from the length of the triangle baseline $AB$.
The second factor $\lvert v_2-v_1\rvert$ 
results from the fact that we have organised
in \eqref{Phiint}, \eqref{F0} an integration over point pairs $(A,B)$
in the plane using a polar coordinate system similar to a Radon transform;
$v_2-v_1$ arises as the Jacobian of the corresponding coordinate transform
from Cartesian to Radon coordinates.
The derivatives of $F(u^*,v^*,0)$ with regard to the coordinates
of $M$ are
\begin{align}
F_{u^*}(0,0,0) &= 
-\int_{-\infty}^{+\infty}
\int_{-\infty}^{+\infty}
f(0,v_1) f(0,v_2) 
(v_2-v_1)^2
\,\mathrm{d}v_2
\,\mathrm{d}v_1 \;,
\label{F0u}
\\
F_{v^*}(0,0,0)&=0\;.
\label{F0v}
\end{align}
Forces $F(0,0,\varphi)$ and their derivatives
for arbitrary angles $\varphi$ will later be obtained
from \eqref{F0}, \eqref{F0u}, \eqref{F0v}
by rotating the $u$, $v$ coordinates accordingly.

For the median of the values $(u,v)$ within 
a $\varrho$-neighbourhood of
$(x,y)=(0,0)$, the density $f(u,v)$ is zero outside of an
$\mathcal{O}(\varrho)$-neighbourhood of $(0,0)$,
allowing to limit the indefinite integrals from \eqref{F0} to the intervals 
$u\in[0,\bar{u}]$, $v_1,v_2\in\bigl[\ubar{v}(u),\bar{v}(u)\bigr]$
such that
\begin{align}
F(0,0,0) &= 
\int_{0}^{\bar{u}}
\int_{\ubar{v}(u)}^{\bar{v}(u)}
\int_{\ubar{v}(u)}^{\bar{v}(u)}
f(u,v_1) f(u,v_2) 
(v_2-v_1)^2
\,\mathrm{d}v_2
\,\mathrm{d}v_1
\,\mathrm{d}u \;.
\label{F0bound}
\end{align}
Expanding $(v_2-v_1)^2=v_2^2-2v_1v_2+v_1^2$, \eqref{F0bound} can be
further decomposed into
\begin{align}
F(0,0,0)
&= \int_{0}^{\bar{u}} \bigl(2J_2(u)J_0(u)-2J_1(u)^2\bigr) \,\mathrm{d}u
\label{F0boundJ}
\end{align}
where
\begin{align}
J_k(u)&:=\int_{\ubar{v}(u)}^{\bar{v}(u)}f(u,v)\,v^k\,\mathrm{d}v
\label{Jk}
\end{align}
for $k=0,1,2$.
Similarly, \eqref{F0u} yields
\begin{align}
F_{u^*}(0,0,0) &= -\bigl(2J_2(0)J_0(0)-2J_1(0)^2\bigr) \;.
\label{F0uJ}
\end{align}

To compute $F(0,0,0)$ and $F_{u^*}(0,0,0)$, we write them as functions
of the coefficients of \eqref{ojataylor22}, i.e.\
$F(0,0,0)=:G(\alpha_1,\beta_1,\delta_1,\alpha_2,\beta_2,\delta_2)$ and
$F_{u^*}(0,0,0)=:H(\alpha_1,\beta_1,\delta_1,\alpha_2,\beta_2,\delta_2)$.

We will linearise $G$ and $H$
around the point $(\alpha_1,\beta_1,\delta_1,\alpha_2,\beta_2,\delta_2)=\bm{0}$
that represents the linear function $\bigl(u(x,y),v(x,y)\bigr)=(x,y)$.
To justify this linearisation, remember that we are interested in the limit
$\varrho\to0$, such that only the terms of lowest order in $\varrho$ matter.
Cross-effects between the different coefficients occur only in higher order
terms.
Denoting from now on by $\doteq$ equality up to higher order terms, we
have therefore
\begin{align}
G\doteq G^0
&+G^0_{\alpha_1}\alpha_1+G^0_{\beta_1}\beta_1+G^0_{\delta_1}\delta_1
+G^0_{\alpha_2}\alpha_2+G^0_{\beta_2}\beta_2+G^0_{\delta_2}\delta_2
\label{Glin}
\;,\\
H\doteq H^0
&+H^0_{\alpha_1}\alpha_1+H^0_{\beta_1}\beta_1+H^0_{\delta_1}\delta_1
+H^0_{\alpha_2}\alpha_2+H^0_{\beta_2}\beta_2+H^0_{\delta_2}\delta_2
\label{Hlin}
\end{align}
where $G^0$, $G^0_{\alpha_1}$ etc.\ are short for $G(\bm{0})$,
$G_{\alpha_1}(\bm{0})$ etc.

To compute $G^0$ and $H^0$, we insert into \eqref{F0} the bounds
$\bar{u}=\varrho$,
$\bar{v}(u)=\sqrt{\varrho^2-u^2}$, $\ubar{v}(u)=-\bar{v}(u)$.
The density becomes constant within the region defined by $\bar{u}$,
$\ubar{v}(u)$ and $\bar{v}(u)$, with $f(u,v)=1$.
Thus we have 
\begin{align}
J_2(u)&=\tfrac23(\varrho^2-u^2)^{3/2}\;,\label{J2uG0}\\
J_1(u)&=0\;, \\
J_0(u)&=2(\varrho^2-u^2)^{1/2}\label{J0uG0}
\end{align}
and via \eqref{F0boundJ} and \eqref{F0uJ} finally
\begin{align}
G^0 &= 
\tfrac{64}{45}\varrho^5\;, &
H^0 &=
-\tfrac83\varrho^4\;.
\label{GH0}
\end{align}
For $G^0_{\alpha_1}$ and $H^0_{\alpha_1}$, one has to vary $\alpha_1$
to obtain the bounds $\bar{u}=\varrho+\alpha_1 \varrho^2$,
$\bar{v}(u)=\sqrt{\varrho^2-u^2+2\alpha_1 u^3}$,
$\ubar{v}(u)\doteq-\bar{v}(u)$.
The density $f(u,v)$ within the
so-given bounds is $1/\mathrm{det}(\mathrm{D}\bm{u})$ at the location
$(x(u,v),y(u,v))$ with $x=u-\alpha_1 u^2+\mathcal{O}(\varrho^3)$,
$y=v$, i.e.\ $f(u,v)=1-2\alpha_1 u+\mathcal{O}(\varrho^2)$. Thus we have
\begin{align}
J_2(u) &\doteq \tfrac23(1-2\alpha_1 u)(\varrho^2-u^2+2\alpha_1 u^3)^{3/2}\;,
\label{J2udG0}
\\
J_1(u) &= 0\;, \\
J_0(u) &\doteq 2(1-2\alpha_1 u)(\varrho^2-u^2+2\alpha_1 u^3)^{1/2}
\label{J0udG0}
\end{align}
and therefore by \eqref{F0boundJ}, \eqref{F0uJ}
\begin{align}
G^0_{\alpha_1} &\doteq 
\frac{\mathrm{d}}{\mathrm{d}\alpha_1}
\int_0^{\bar{u}} \frac83(1-2\alpha_1 u)^2(\varrho^2-u^2+2\alpha_1 u^3)^2
\,\mathrm{d}u
\,\bigg|_{\alpha_1=0}
= -\frac8{9}\varrho^6\;,
\label{G0a}
\\
H^0_{\alpha_1} 
&\doteq
-\frac{\mathrm{d}}{\mathrm{d}\alpha_1}
\,\frac83\varrho^4
\,\bigg|_{\alpha_1=0}
= 0\;.
\label{H0a}
\end{align}
Proceeding analogously for the other coefficients, we find
the values of $\bar{u}$, $\bar{v}$, $\ubar{v}$ and $f(u,v)$ and the
resulting coefficients compiled
in Table~\ref{tab-22anavals}.
\begin{table}[b!]
\caption{\label{tab-22anavals}
Integration bounds, densities, integrals $J_k(u)$ and resulting coefficients
$G^0_{\omega}$, $H^0_{\omega}$ of the expansions
\eqref{Glin}, \eqref{Hlin} for 
$\omega\in\{\alpha_1,\beta_1,\delta_1,\alpha_2,\beta_2,\delta_2\}$.
$J_1(u)$ and $H^0_\omega$ are always zero and therefore omitted.
All values are approximated up to higher order terms.}
\medskip
\centering
\small
\begin{tabular}{c@{~~}l@{~~}l@{~~}l@{~~}l@{~~}l@{~~}c}
\hline
\rule{0pt}{1.2em}$\omega$&$\bar{u}$&$\bar{v}(u)$, $\ubar{v}(u)$&$f(u,v)$&
$J_2(u)$&
$J_0(u)$&
$G^0_\omega$
\\
\hline
\rule{0pt}{1.2em}%
$\alpha_1$&$\varrho\!+\!\alpha_1\varrho^2$
&$\pm\sqrt{\varrho^2-u^2+2\alpha_1u^3}$&
  $1\!-\!2\alpha_1u$&
  $\frac23(1-2\alpha_1u)\bar{v}(u)^3$&
  $2(1-2\alpha_1u)\bar{v}(u)$&
  $-\frac89\varrho^6$
  \\
\rule{0pt}{1.1em}%
$\beta_1$&$\varrho$&$\pm\sqrt{\varrho^2\!-\!u^2\!+\!\beta_1^2u^4}\!+\!\beta_1u^2$&
  $1-\beta_1v$&
  $\frac23(\varrho^2\!-\!u^2\!+\!\beta_1^2u^4)^{3/2}$&
  $2(\varrho^2\!-\!u^2\!+\!\beta_1^2u^4)^{1/2}$&
  $0$
  \\
\rule{0pt}{1.1em}%
$\delta_1$&$\varrho$&$\pm\sqrt{(\varrho^2-u^2)(1+2\delta_1u)}$&
  $1$&
  $\frac23\bar{v}(u)^3$&
  $2\bar{v}(u)$&
  $\frac{16}9\varrho^6$
  \\
\rule{0pt}{1.1em}%
$\alpha_2$&$\varrho$&$\pm\sqrt{\varrho^2\!-\!u^2\!+\!\alpha_2u^4}\!+\!\alpha_2u^2$&
  $1$&
  $\frac23(\varrho^2\!-\!u^2\!+\!\alpha_2^2u^4)^{3/2}$&
  $2(\varrho^2\!-\!u^2\!+\!\alpha_2^2u^4)^{1/2}$&
  $0$
  \\
\rule{0pt}{1.1em}%
$\beta_2$&$\varrho$&$\pm\sqrt{(\varrho^2-u^2)(1+2\beta_2u)}$&
  $1-\beta_2u$&
  $\frac23(1-\beta_2u)\bar{v}^3$&
  $2(1-\beta_2u)\bar{v}$&
  $\frac89\varrho^6$
  \\
\rule{0pt}{1.1em}%
$\delta_2$&$\varrho$&$\pm\sqrt{\varrho^2\!-\!u^2}\!+\!\delta_2(\varrho^2\!-\!u^2)$&
  $1-2\delta_2v$&
  $\frac23(\varrho^2-u^2)^{3/2}$&
  $2(\varrho^2-u^2)^{1/2}$&
  $0$
\\
\hline
\end{tabular}
\end{table}

Inserting the values from Table~\ref{tab-22anavals}
into \eqref{Glin} and \eqref{Hlin},
we have
\begin{align} 
F(0,0,0) &= \tfrac{64}{45}\varrho^5
+\tfrac{8}{9}\varrho^6(-\alpha_1+2\delta_1+\beta_2)\;,
\label{F-expans}
\\
F_{u^*}(0,0,0) &= \tfrac83\varrho^4\;,
\label{dF-expans}
\end{align}
and by orthogonal transform in the $u$-$v$ plane
\begin{align}
F(0,0,\varphi) &= \tfrac{64}{45}\varrho^5
+\tfrac{8}{9}\varrho^6\Bigl(
-(\alpha_1\cos\varphi+\alpha_2\sin\varphi)\cos^2\varphi
\notag\\*&\qquad{}
-(\beta_1\cos\varphi+\beta_2\sin\varphi)\cos\varphi\sin\varphi
-(\delta_1\cos\varphi+\delta_2\sin\varphi)\sin^2\varphi
\notag\\*&\qquad{}
+2(\alpha_1\cos\varphi+\alpha_2\sin\varphi)\sin^2\varphi
-2(\beta_1\cos\varphi+\beta_2\sin\varphi)\cos\varphi\sin\varphi
\notag\\*&\qquad{}
+2(\delta_1\cos\varphi+\delta_2\sin\varphi)\cos^2\varphi
-2(-\alpha_1\sin\varphi+\alpha_2\cos\varphi)\cos\varphi\sin\varphi
\notag\\*&\qquad{}
+(-\beta_1\sin\varphi+\beta_2\cos\varphi)(\cos^2\varphi-\sin^2\varphi)
\notag\\*&\qquad{}
+2(-\delta_1\sin\varphi+\delta_2\cos\varphi)\cos\varphi\sin\varphi
\Bigr)
\;,\\
F_{u^*}(0,0,\varphi) &= \tfrac83\varrho^4\cos\varphi
\;, \\
F_{v^*}(0,0,\varphi) &= \tfrac83\varrho^4\sin\varphi
\;.
\label{Fv00phi}
\end{align}
Integration \eqref{Phiint} then yields
\begin{align}
\bm{\varPhi}(0,0) &= \frac{\pi}{18}\varrho^6
\begin{pmatrix}
\alpha_1+3\delta_1-\beta_2\\
3\alpha_2-\delta_2-\beta_1
\end{pmatrix}\;,
\label{22Phi}\\
\mathrm{D}\bm{\varPhi}(0,0) &= 
-\frac23\pi\varrho^4\begin{pmatrix}1&0\\0&1\end{pmatrix}
\label{22DPhi}
\end{align}
and via \eqref{Philin} eventually
\begin{align}
\begin{pmatrix}u^*\\v^*\end{pmatrix} &=
\frac{\varrho^2}{12}
\begin{pmatrix}
\alpha_1+3\delta_1-\beta_2\\
3\alpha_2+\delta_2-\beta_1
\end{pmatrix}\;.
\label{ojapde-unitjaco-abgdez}
\end{align}
Inserting \eqref{abg}, \eqref{dez} into \eqref{ojapde-unitjaco-abgdez},
we see that for $\mathrm{D}\bm{u}=\mathrm{diag}(1,1)$ the Oja median
filtering step approximates an explicit time step of size
$\tau=\varrho^2/24$ of the PDE system 
\eqref{ojapde22-I1}--\eqref{ojapde22-I2}.
\hfill$\Box$

\section{Second Proof of Lemma~\ref{lem-ojapde22-I}}
\label{app-22proof2}

As in the previous proof, we express the minimisation condition as
$\bm{\varPhi}(u^*,v^*)=0$ where $\bm{\varPhi}(u^*,v^*)$ expresses
an anti-gradient of the objective function of the Oja median
(the sum of triangle areas) at the median candidate point $M=(u^*,v^*)$.

Let $M=(u^*,v^*)$ with $u^*,v^*=\mathcal{O}(\varrho^2)$.
For two points $A=(u_1,v_1)$, $B=(u_2,v_2)$ in the $u$-$v$ plane, the
force exercised on $M$ by the negative gradient of the area of triangle
$MAB$ is $\frac12F_{M;AB}$ where
\begin{equation}
F_{M;AB} = \begin{pmatrix}v_2-v_1\\u_1-u_2\end{pmatrix}
= \begin{pmatrix}v_2\\-u_2\end{pmatrix} 
- \begin{pmatrix}v_1\\-u_1\end{pmatrix}
+\mathcal{O}(\varrho^2)
\label{force22MAB}
\end{equation}
provided the triangle $MAB$ is positively oriented. If $MAB$ is negatively
oriented, the sign of $F_{M;AB}$ changes.

Let now $A$ and $B$ given by
\begin{align}
A &= (u(x_1,y_1),v(x_1,y_1))\;,\\
B &= (u(x_2,y_2),v(x_2,y_2))
\end{align}
with $(x_1,y_1),(x_2,y_2)\in D_\varrho(\bm{0})$.

Aggregating the forces $F_{M;AB}$ directly by integration over
$x_1$, $y_1$, $x_2$, $y_2$, and denoting again by $\doteq$ equality up
to higher order terms, one sees that the resulting force can be
stated as
\begin{align}
\bm{\varPhi}&:=
\frac12
\iint_{D_\varrho}\iint_{D_\varrho}F_{M;AB} 
\,\mathrm{d}x_2\,\mathrm{d}y_2\,\mathrm{d}x_1\,\mathrm{d}y_1
\notag\\*
&\doteq \frac12\iint_{D_\varrho}\iint_{\mathcal{A}_+(x_1,y_1)}
\begin{pmatrix}v_2-v_1\\u_1-u_2\end{pmatrix}
\,\mathrm{d}x_2\,\mathrm{d}y_2\,\mathrm{d}x_1\,\mathrm{d}y_1
\notag\\*&\quad{}
-\frac12\iint_{D_\varrho}\iint_{\mathcal{A}_-(x_1,y_1)}
\begin{pmatrix}v_2-v_1\\u_1-u_2\end{pmatrix}
\,\mathrm{d}x_2\,\mathrm{d}y_2\,\mathrm{d}x_1\,\mathrm{d}y_1
\notag\\*
&= \frac12\iint_{D_\varrho}\iint_{\mathcal{A}_+(x_1,y_1)}
\begin{pmatrix}v_2\\-u_2\end{pmatrix}
\,\mathrm{d}x_2\,\mathrm{d}y_2\,\mathrm{d}x_1\,\mathrm{d}y_1
\notag\\*&\quad{}
- \frac12\iint_{D_\varrho}\iint_{\mathcal{A}_+(x_1,y_1)}
\begin{pmatrix}v_1\\-u_1\end{pmatrix}
\,\mathrm{d}x_2\,\mathrm{d}y_2\,\mathrm{d}x_1\,\mathrm{d}y_1
\notag\\*&\quad{}
- \frac12\iint_{D_\varrho}\iint_{\mathcal{A}_-(x_1,y_1)}
\begin{pmatrix}v_2\\-u_2\end{pmatrix}
\,\mathrm{d}x_2\,\mathrm{d}y_2\,\mathrm{d}x_1\,\mathrm{d}y_1
\notag\\*&\quad{}
+ \frac12\iint_{D_\varrho}\iint_{\mathcal{A}_-(x_1,y_1)}
\begin{pmatrix}v_1\\-u_1\end{pmatrix}
\,\mathrm{d}x_2\,\mathrm{d}y_2\,\mathrm{d}x_1\,\mathrm{d}y_1
\;.
\end{align}
Here, $\mathcal{A}_{\pm}(x_1,y_1)$ denote the regions for 
$(x_2,y_2)\in D_\varrho$
for which $MAB$ is positively or negatively oriented, respectively.
Since $B\in\mathcal{A}_+(x_1,y_1)$ if and only if $A\in\mathcal{A}_-(x_2,y_2)$
and vice versa, we can switch the roles of $(x_1,y_1)$ and $(x_2,y_2)$
in two of the integrals to combine the previous expressions into
\begin{align}
\bm{\varPhi}
&\doteq -\iint_{D_\varrho}\iint_{\mathcal{A}_+(x_1,y_1)}
\begin{pmatrix}v_1\\-u_1\end{pmatrix}
\,\mathrm{d}x_2\,\mathrm{d}y_2\,\mathrm{d}x_1\,\mathrm{d}y_1
\notag\\*&\quad{}
+ \iint_{D_\varrho}\iint_{\mathcal{A}_-(x_1,y_1)}
\begin{pmatrix}v_1\\-u_1\end{pmatrix}
\,\mathrm{d}x_2\,\mathrm{d}y_2\,\mathrm{d}x_1\,\mathrm{d}y_1
\notag\\*
&=-\iint_{D_\varrho}
\begin{pmatrix}v_1\\-u_1\end{pmatrix}
\left(\lvert\mathcal{A}_+(x_1,y_1)\rvert-\lvert\mathcal{A}_-(x_1,y_1)\rvert
\right)
\,\mathrm{d}x_1\,\mathrm{d}y_1
\label{22areadiffint}
\end{align}
where $\lvert\mathcal{A}_{\pm}(x_1,y_1)\rvert$ denote the areas of the
respective regions.

It remains to determine
the area differences
\begin{equation}
\Delta\mathcal{A}(x_1,y_1):=
\lvert\mathcal{A}_+(x_1,y_1)\rvert-\lvert\mathcal{A}_-(x_1,y_1)\rvert
\label{22bdeltaa}
\end{equation}
for all $(x_1,y_1)\in D_\varrho$.

To this end, we 
use again the Taylor expansion \eqref{ojataylor22}.
For $u^*=v^*=0$ and
$\alpha_1=\beta_1=\delta_1=\alpha_2=\beta_1=\delta_2=0$ we have
$u(x,y)=x$, $v(x,y)=y$, and 
$\mathcal{A}_+(x_1,y_1)$ and $\mathcal{A}_-(x_1,y_1)$ are
half-discs separated by the diameter of $D_\varrho$ through $M$ and $A$.
Generally, the two regions are separated by the curve 
$(u_1-u^*)(v_2-v^*)-(u_2-u^*)(v_1-v^*)=0$, 
which after inserting \eqref{ojataylor22} and dropping
higher order terms becomes
\begin{align}
0&=x_1y_2-x_2y_1
-u^*(y_2-y_1)
-v^*(x_1-x_2)
\notag\\*&\quad{}
+\alpha_1(x_1^2y_2-x_2^2y_1)
+\beta_1(x_1y_1y_2-x_2y_1y_2)
+\delta_1(y_1^2y_2-y_1y_2^2)
\notag\\*&\quad{}
+\alpha_2(x_1x_2^2-x_1^2x_2)
+\beta_2(x_1x_2y_2-x_1x_2y_1)
+\delta_2(x_1y_2^2-x_2y_1^2)
\;.
\label{22bcondline}
\end{align}
To determine the deviation of this line from the bisecting diameter
mentioned above, we introduce coordinates aligned to the line $MA$
by $x_1=r\cos\varphi$, $y_1=r\sin\varphi$ and 
$x_2=s\cos\varphi-t\sin\varphi$, $y_2=s\sin\varphi+t\cos\varphi$.
We can then write \eqref{22bcondline} up to higher order terms as
\begin{align}
t=t(s) &\doteq s^2(
\alpha_1\cos^2\varphi\sin\varphi
+\beta_1\cos\varphi\sin^2\varphi
+\delta_1\sin^3\varphi
\notag\\*&\qquad{}
-\alpha_2\cos^3\varphi
-\beta_2\cos^2\varphi\sin\varphi
-\delta_2\cos\varphi\sin^2\varphi)
\notag\\&\quad{}
+s\Bigl(
\frac{u^*}r\sin\varphi-\frac{v^*}r\cos\varphi
\notag\\*&\qquad{}
-r\alpha_1\cos^2\varphi\sin\varphi
-r\beta_1\cos\varphi\sin^2\varphi
-r\delta_1\sin^3\varphi
\notag\\*&\qquad{}
+r\alpha_2\cos^3\varphi
+r\beta_2\cos^2\varphi\sin\varphi
+r\delta_2\cos\varphi\sin^2\varphi\Bigr)
\notag\\*&\quad{}
-\frac{u^*}r\sin\varphi+\frac{v^*}r\cos\varphi\;.
\end{align}
Up to higher order terms, 
the area difference $\Delta\mathcal{A}(x_1,y_1)$ is minus double the
area between this line and the $s$-axis in the interval 
$s\in[-\varrho,\varrho]$, i.e.\
\begin{align}
\Delta\mathcal{A}(x_1,y_1) 
&\doteq -2\int_{-\varrho}^\varrho t(s)\,\mathrm{d}s
\notag\\*
&\doteq
4\varrho(u^*\sin\varphi-v^*\cos\varphi)
\notag\\*&\quad{}
-\frac43\varrho^3(
\alpha_1\cos^2\varphi\sin\varphi
+\beta_1\cos\varphi\sin^2\varphi
+\delta_1\sin^3\varphi
\notag\\*&\qquad{}
-\alpha_2\cos^3\varphi
-\beta_2\cos^2\varphi\sin\varphi
-\delta_2\cos\varphi\sin^2\varphi)
\;.
\label{22areadiff}
\end{align}
Inserting \eqref{ojataylor22}, \eqref{22bdeltaa} and \eqref{22areadiff}
into \eqref{22areadiffint} yields
\begin{align}
\bm{\varPhi}
&\doteq-\iint_{D_\varrho}
\begin{pmatrix}-y-\alpha_2x^2-\beta_2xy-\delta^2y^2\\
x+\alpha_1x^2+\beta_1xy+\delta_1y^2\end{pmatrix}
\Delta\mathcal{A}(x,y)
\,\mathrm{d}x\,\mathrm{d}y
\notag\\*
&\doteq
\frac12\int_0^\varrho\int_0^{2\pi}r^2
\begin{pmatrix}\sin\varphi\\-\cos\varphi\end{pmatrix}
\Delta\mathcal{A}(x,y)
\,\mathrm{d}\varphi\,\mathrm{d}r
\notag\\
&=
\int_0^\varrho r^2\,\mathrm{d}r
\Biggl(2\varrho
\biggl(u^*
\int_0^{2\pi}\begin{pmatrix}\sin^2\varphi\\-\cos\varphi\sin\varphi
\end{pmatrix}\,\mathrm{d}\varphi
-v^*
\int_0^{2\pi}\begin{pmatrix}\cos\varphi\sin\varphi\\-\cos^2\varphi
\end{pmatrix}\,\mathrm{d}\varphi\biggr)
\notag\\&\qquad\qquad\qquad{}
-\frac23\varrho^3
\biggl(
\alpha_1\int_0^{2\pi}\begin{pmatrix}
\cos^2\varphi\sin^2\varphi\\-\cos^3\varphi\sin\varphi
\end{pmatrix}\,\mathrm{d}\varphi
+\beta_1\int_0^{2\pi}\begin{pmatrix}
\cos\varphi\sin^3\varphi\\-\cos^2\varphi\sin^2\varphi
\end{pmatrix}\,\mathrm{d}\varphi
\notag\\&\qquad\qquad\qquad\qquad{}
+\delta_1\int_0^{2\pi}\begin{pmatrix}
\sin^4\varphi\\-\cos\varphi\sin^3\varphi
\end{pmatrix}\,\mathrm{d}\varphi
-\alpha_2\int_0^{2\pi}\begin{pmatrix}
\cos^3\varphi\sin\varphi\\-\cos^4\varphi
\end{pmatrix}\,\mathrm{d}\varphi
\notag\\&\qquad\qquad\qquad\qquad{}
-\beta_2\int_0^{2\pi}\begin{pmatrix}
\cos^2\varphi\sin^2\varphi\\-\cos^3\varphi\sin\varphi
\end{pmatrix}\,\mathrm{d}\varphi
-\delta_2\int_0^{2\pi}\begin{pmatrix}
\cos\varphi\sin^3\varphi\\-\cos^2\varphi\sin^2\varphi
\end{pmatrix}\,\mathrm{d}\varphi
\biggr)\Biggr)
\notag\\
&=\frac23\pi\varrho^4\begin{pmatrix}u^*\\v^*\end{pmatrix}
-\frac1{18}\pi\varrho^6\begin{pmatrix}\alpha_1+3\delta_1-\beta_2\\
3\alpha_2+\delta_2-\beta_1\end{pmatrix} \;,
\end{align}
which reproduces the result \eqref{22Phi}, \eqref{22DPhi} from
the first proof such that one can again infer 
\eqref{ojapde-unitjaco-abgdez} and thereby 
\eqref{ojapde22-I1}, \eqref{ojapde22-I2}.
\hfill$\Box$

\section{Proof of Lemma~\ref{lem-ojapde33-I}}
\label{app-33proof}

We start with the Taylor expansion of $\bm{u}(x,y,z)$ around
$(0,0,0)$ up to second order given as
\begin{align}
u(x,y,z) = x 
&+ \alpha_1 x^2 + \beta_1 xy + \gamma_1 xz
+ \delta_1 y^2 + \varepsilon_1 yz + \zeta_1 z^2 \;,
\label{33tayloru}
\\
v(x,y,z) = y 
&+ \alpha_2 x^2 + \beta_2 xy + \gamma_2 xz
+ \delta_2 y^2 + \varepsilon_2 yz + \zeta_2 z^2 \;,
\label{33taylorv}
\\
w(x,y,z) = z 
&+ \alpha_3 x^2 + \beta_3 xy + \gamma_3 xz
+ \delta_3 y^2 + \varepsilon_3 yz + \zeta_3 z^2 
\label{33taylorw}
\end{align}
where $\alpha_1=\frac12u_{xx}$, $\beta_1=u_{xy}$ etc.

Similarly as in Appendix~\ref{app-22proof1} for the bivariate
planar case, we seek the point $M:=(u^*,v^*,w^*)$ that minimises
the integral over all volumes of tetrahedra $MABC$ with
$A=(u_1,v_1,w_1)$, $B=(u_2,v_2,w_2)$, $C=(u_3,v_3,w_3)$ where
$(u_i,v_i,w_i)=(u(x_i,y_i,z_i),v(x_i,y_i,z_i),w(x_i,y_i,z_i))$
with $(x_i,y_i,z_i)\in B_\varrho(0,0,0)$, weighted with the
density $f(u_1,v_1,w_1)f(u_2,v_2,w_2)f(u_3,v_3,w_3)$.

For each tetrahedron $MABC$, the negative gradient of its volume
as a function of $M$ is a force vector $\frac16F_{M;ABC}$ perpendicular
to the plane $ABC$ with a length proportional to the area
$\lvert ABC\rvert$ of the triangle $ABC$. Assuming positive
orientation of that triangle, $F_{M;ABC}$ equals the vector
(cross) product
$-(u_2-u_1,v_2-v_1,w_2-w_1)\times(u_3-u_1,v_3-v_1,w_3-w_1)$.

Organising the integration over point triples $(A,B,C)$ again
by orientations of the force vectors, we consider
the resultant $F(u^*,v^*,w^*,\bm{p})$ of all forces in direction
of any given unit vector $\bm{p}\in\mathrm{S}^2$. Linearising for
$(u^*,v^*,w^*)$ around $\bm{0}$,
\begin{align}
\begin{pmatrix}u^*\\v^*\\w^*\end{pmatrix}
&= - \bigl(\mathrm{D}\bm{\varPhi}(0,0,0)\bigr)^{-1}\bm{\varPhi}(0,0,0)
\label{Philin33}
\end{align}
(compare \eqref{Philin}), 
and considering first the
case where $\bm{p}=\mathbf{e}_1=(1,0,0)$ is the first coordinate
vector, we can state the analogue of \eqref{F0} as
\begin{align}
F(0,0,0,\mathbf{e}_1) 
&=
\int_0^{+\infty}\!\!
\int_{-\infty}^{+\infty}\!\!
\int_{-\infty}^{+\infty}\!\!
\int_{-\infty}^{+\infty}\!\!
\int_{-\infty}^{+\infty}\!\!
\int_{-\infty}^{+\infty}\!\!
\int_{-\infty}^{+\infty}\!\!
f(u,v_1,w_1)f(u,v_2,w_2)f(u,v_3,w_3)
\notag\\*&\qquad\qquad{}\times
\bigl((v_2-v_1)(w_3-w_1)-(v_3-v_1)(w_2-w_1)\bigr)^2
\notag\\*&\qquad\qquad{}
\,\mathrm{d}w_3
\,\mathrm{d}v_3
\,\mathrm{d}w_2
\,\mathrm{d}v_2
\,\mathrm{d}w_1
\,\mathrm{d}v_1
\,\mathrm{d}u
\;.
\end{align}
The appearance of the square of the area
$(v_2-v_1)(w_3-w_1)-(v_3-v_1)(w_2-w_1)$
is again due to the Radon-like polar coordinates underlying the
integration over directions.

As in Appendix~\ref{app-22proof1}, the
indefinite integrals can be limited to finite intervals
$u\in[0,\bar{u}]$, $v_i\in\bigl[\ubar{v}(u),\bar{v}(u)\bigr]$,
$w_i\in\bigl[\ubar{w}(u,v_i),\bar{w}(u,v_i)\bigr]$ for $i=1,2,3$.
Expanding 
\begin{align}
\kern1em&\kern-1em
\bigl((v_2-v_1)(w_3-w_1)-(v_3-v_1)(w_2-w_1)\bigr)^2
\notag\\*
&=\sum_{\substack{i,j\in\{1,2,3\}\\i\ne j}}v_i^2w_j^2
-2\sum_{\substack{i,j\in\{1,2,3\}\\i<j}}v_iw_iv_jw_j
-2\sum_{\substack{i,j,k\in\{1,2,3\}\\i<j;k\ne i,j}}v_iv_jw_k^2
\notag\\*&\quad{}
-2\sum_{\substack{i,j,k\in\{1,2,3\}\\j<k;i\ne j,k}}v_i^2w_jw_k
+2\sum_{\substack{i,j,k\in\{1,2,3\}\\i\ne j\ne k\ne i}}v_iw_jv_kw_k
\end{align}
then leads to
\begin{align}
F(0,0,0,\mathbf{e}_1) &=
\int_0^{\bar{u}} \bigl(6J_{20}(u)J_{02}(u)J_{00}(u)
-6J_{11}(u)^2J_{00}(u)
-6J_{02}(u)J_{10}(u)^2
\notag\\*&\qquad\qquad{}
-6J_{20}(u)J_{01}(u)^2
+12J_{10}(u)J_{01}(u)J_{00}(u)
\bigr)
\,\mathrm{d}u
\label{F0J33}
\end{align}
with
\begin{align}
J_{kl}(u)&:=
\int_{\ubar{v}(u)}^{\bar{v}(u)}
\int_{\ubar{w}(u,v)}^{\bar{w}(u,v)}
f(u,v,w)v^kw^l
\,\mathrm{d}w
\,\mathrm{d}v
\end{align}
for $k,l=0,1,2$.

We linearise $F(0,0,0,\mathbf{e}_1)$ with regard to
the 18 coefficients
$\omega\in\{\alpha_i,\beta_i,\gamma_i,\delta_i,
\varepsilon_i,\zeta_i~|~i=1,2,3\}$
of the Taylor expansion
\eqref{33tayloru}--\eqref{33taylorw} 
\begin{align}
F(0,0,0,\mathbf{e}_1) =
G^0
&+ \sum\limits_{\omega} G^0_\omega \omega\;.
\label{Glin33}
\end{align}
Like in the bivariate case of Appendix~\ref{app-22proof1},
cross-effects between the coefficients $\omega$ take effect
only in higher-order terms that can be neglected for our
purpose. 
Moreover, $G^0$ is again a constant that vanishes in the 
integration over directions, so we refrain from explicitly
calculating it. 

\begin{table}[b!]
\caption{\label{tab-33anavals-1}
Integration bounds, densities 
and resulting coefficients
$G^0_{\omega}$ of the expansion
\eqref{Glin33} for 
$\omega\in\{\alpha_i,\beta_i,\gamma_i,\delta_i,\varepsilon_i,\zeta_i
~|~i=1,2,3\}$.
Coefficients $H^0_\omega$ are always zero and therefore omitted.
All values are approximated up to higher order terms.
The integrals $J_{kl}(u)$ are found in Table~\ref{tab-33anavals-2}.
Coefficients listed in the second column are inferred from the ones
in the first column by symmetry.
}
\medskip
\centering
\small
\begin{tabular}{c@{~~}c@{~~}l@{~~}l@{~~}l@{~~}l@{~~}c}
\hline
\rule{0pt}{1.2em}%
$\omega$&
$\omega$\,(symm.)&
$\bar{u}$&
$\bar{v}(u)$, $\ubar{v}(u)$&
$\bar{w}(u,v)$, $\ubar{w}(u,v)$&
$f(u,v,w)$&
$G^0_\omega$
\\ \hline
\rule{0pt}{1.2em}%
$\alpha_1$&&
$\varrho\!+\!\alpha_1\varrho^2$&
$\pm\sqrt{\varrho^2-u^2+2\alpha_1u^3}$&
$\pm\sqrt{\varrho^2-u^2+2\alpha_1u^3-v^2}$&
$1-2\alpha_1u$&
$-\frac18\pi^3\varrho^{12}$
\\
\rule{0pt}{1.1em}%
$\gamma_1$&$\beta_1$&
$\varrho$&
$\pm\sqrt{\varrho^2-u^2}$&
$\pm\sqrt{\varrho^2-u^2-v^2}+\gamma_1u^2$&
$1-\gamma_1w$&
$0$
\\
\rule{0pt}{1.1em}%
$\delta_1$&$\zeta_1$&
$\varrho$&
$\pm\sqrt{(\varrho^2-u^2)(1+2\delta_1u)}$&
$\pm\sqrt{\varrho^2-u^2+2\delta_1uv^2-v^2}$&
$1$&
$\frac5{32}\pi^3\varrho^{12}$
\\
\rule{0pt}{1.1em}%
$\varepsilon_1$&&
$\varrho$&
$\pm\sqrt{\varrho^2-u^2}$&
$\pm\sqrt{\varrho^2-u^2-v^2}+\varepsilon_1uv$&
$1$&
$0$
\\
\rule{0pt}{1.1em}%
$\beta_2$&$\gamma_3$&
$\varrho$&
$\pm\sqrt{(\varrho^2-u^2)(1+2\beta_2u)}$&
$\pm\sqrt{\varrho^2-u^2+2\beta_2uv^2-v^2}$&
$1-\beta_2u$&
$\frac1{16}\pi^3\varrho^{12}$
\\
\rule{0pt}{1.1em}%
$\gamma_2$&$\beta_3$&
$\varrho$&
$\pm\sqrt{\varrho^2-u^2}$&
$\pm\sqrt{\varrho^2-u^2-v^2}+\gamma_2uv$&
$1$&
$0$
\\
\rule{0pt}{1.1em}%
$\delta_2$&$\zeta_3$&
$\varrho$&
$\pm\sqrt{\varrho^2\!-\!u^2}\!+\!\delta_2(\varrho^2\!-\!u^2)$&
$\pm\sqrt{\varrho^2-u^2-v^2+2\delta_2v^3}$&
$1-2\delta_2v$&
$0$
\\
\rule{0pt}{1.1em}%
$\varepsilon_2$&$\varepsilon_3$&
$\varrho$&
$\pm\sqrt{\varrho^2-u^2}$&
$\pm\sqrt{\varrho^2-u^2-v^2}+\varepsilon_2v^2$&
$1-\varepsilon_2w$&
$0$
\\
\rule{0pt}{1.1em}%
$\zeta_2$&$\delta_3$&
$\varrho$&
$\pm\sqrt{\varrho^2-u^2}$&
$\pm\sqrt{\varrho^2-u^2-v^2}(1+\zeta_2 v)$&
$1$&
$0$
\\
\rule{0pt}{1.1em}%
$\alpha_3$&$\alpha_2$&
$\varrho$&
$\pm\sqrt{\varrho^2-u^2}$&
$\pm\sqrt{\varrho^2-u^2-v^2}+\alpha_3u^2$&
$1$&
$0$
\\
\hline
\end{tabular}
\end{table}

\begin{table}[b!]
\caption{\label{tab-33anavals-2}
Integrals $J_{kl}(u)$ from the computation of the
coefficients
$G^0_{\omega}$ from Table~\ref{tab-33anavals-1}.
All values are approximated up to higher order terms.
For abbreviation, $U:=\varrho^2-u^2$ is used.
}
\medskip
\centering
\small
\begin{tabular}{c@{~~}l@{~~}l@{~~}l@{~~}l@{~~}l@{~~}l}
\hline
\rule{0pt}{1.2em}%
$\omega$&
$J_{20}(u)$&
$J_{11}(u)$&
$J_{02}(u)$&
$J_{10}(u)$&
$J_{01}(u)$&
$J_{00}(u)$
\\ \hline
\rule{0pt}{1.2em}%
$\alpha_1$&
$\frac14\pi(1-2\alpha_1u)\bar{v}^4$&
$0$&
$\frac14\pi(1-2\alpha_1u)\bar{v}^4$&
$0$&
$0$&
$\pi(1-2\alpha_1u)\bar{v}^2$
\\
\rule{0pt}{1.1em}%
$\gamma_1$&
$\frac14\pi U^2$&
$0$&
$\frac14\pi U^2$&
$0$&
$-\frac14\pi\gamma_1 U(\varrho^2-5u^2)$&
$\pi U$
\\
\rule{0pt}{1.1em}%
$\delta_1$&
$\frac14\pi(1+\delta_1u)^3U^2$&
$0$&
$\frac14\pi(1+\delta_1u)U^2$&
$0$&
$0$&
$\pi(1+\delta_1u) U$
\\
\rule{0pt}{1.1em}%
$\varepsilon_1$&
$\frac14\pi U^2$&
$\frac12\pi\varepsilon_1u U^2$&
$\frac14\pi U^2$&
$0$&
$0$&
$\pi U$
\\
\rule{0pt}{1.1em}%
$\beta_2$&
$\frac14\pi(1+\beta_2u)^2U^2$&
$0$&
$\frac14\pi U^2$&
$0$&
$0$&
$\pi U$
\\
\rule{0pt}{1.1em}%
$\gamma_2$&
$\frac14\pi U^2$&
$\frac12\pi\gamma_2u U^2$&
$\frac14\pi U^2$&
$0$&
$0$&
$\pi U$
\\
\rule{0pt}{1.1em}%
$\delta_2$&
$\frac14\pi U^2$&
$0$&
$\frac14\pi U^2$&
$\frac14\pi\delta_2 U^2$&
$0$&
$\pi U$
\\
\rule{0pt}{1.1em}%
$\varepsilon_2$&
$\frac14\pi U^2$&
$0$&
$\frac14\pi U^2$&
$0$&
$0$&
$\pi U$
\\
\rule{0pt}{1.1em}%
$\zeta_2$&
$\frac14\pi U^2$&
$0$&
$\frac14\pi U^2$&
$\frac14\pi\zeta_2 U^2$&
$0$&
$\pi U$
\\
\rule{0pt}{1.1em}%
$\alpha_3$&
$\frac14\pi U^2$&
$0$&
$\frac14\pi U^2$&
$0$&
$\pi\alpha_3u^2 U$&
$\pi U$
\\
\hline
\end{tabular}
\end{table}

To calculate the value $G^0_\omega$ for each coefficient $\omega$
one can then assume that only this particular coefficient varies
around $0$ while all other coefficients vanish. For 10 of the
coefficients $\omega$ one calculates
then the integration bounds $\bar{u}$, $\ubar{v}(u)$,
$\bar{v}(u)$, $\ubar{w}(u,v)$, $\bar{w}(u,v)$ and the density
function $f(u,v,w)$ as stated in Table~\ref{tab-33anavals-1},
the respective integrals $J_{kl}(u)$ as given in 
Table~\ref{tab-33anavals-2} and finally using \eqref{F0J33}
the coefficients $G^0_\omega$ which are again listed in
Table~\ref{tab-33anavals-1}. The remaining 8 coefficients need
not be calculated in this tedious way as they can be derived
from the obvious symmetry of $F(0,0,0,\mathbf{e}_1)$ under
the exchange of $y$ and $z$; the detailed symmetries of coefficients
are also stated in Table~\ref{tab-33anavals-1}.

For the derivatives of $F$ we have
\begin{align}
F_{u^*}(0,0,0,\mathbf{e}_1) &= H^0 (1+\mathcal{O}(\varrho^2))\;,\\
F_{v^*}(0,0,0,\mathbf{e}_1) &= 0\;,\\
F_{w^*}(0,0,0,\mathbf{e}_1) &= 0\;.
\end{align}
Here, $H^0$ is calculated from the unperturbed case
$u=x$, $v=y$, $w=z$ via
\begin{align}
J_{20}(0)&=J_{02}(0)=\tfrac14\pi\varrho^4\;,\\
J_{00}(0)&=\pi\varrho^2\;,\\
J_{11}(0)&=J_{10}(0)=J_{01}(0)=0 \;.
\end{align}
Thus only the first summand of the integrand of \eqref{F0J33} 
is non-zero, leading to
\begin{align}
H^0 &= 
-6J_{20}(0)J_{02}(0)J_{00}(0)
=-\tfrac38\pi^3\varrho^{10}\;.
\end{align}
From the so obtained expressions
\begin{align}
F(0,0,0,\mathbf{e}_1) &\doteq G^0 + \tfrac1{32}\pi^3\varrho^{12}(
-4\alpha_1+5\delta_1+5\zeta_1
+2\beta_2+2\gamma_3)\;,
\\
F_{u^*}(0,0,0,\mathbf{e}_1) &\doteq -\tfrac38\pi^3\varrho^{10}
\end{align}
general expressions for $F(0,0,0,\bm{p})$ and
its derivatives w.r.t.\ $u^*$, $v^*$, $w^*$ can be obtained.
Given the parametrisation 
\begin{align}
\bm{p}&=\bm{p}(\varphi,\psi)
=(\cos\varphi,\sin\varphi\cos\psi,\sin\varphi\sin\psi)\transpose
\end{align}
one can use e.g.\ the rotation matrix 
\begin{align}
R &= \begin{pmatrix}
\cos\varphi&\sin\varphi\cos\psi&\sin\varphi\sin\psi\\
-\sin\varphi&\cos\varphi\cos\psi&\cos\varphi\sin\psi\\
0&\sin\psi&-\cos\psi
\end{pmatrix}
\end{align}
to transform the $(u,v,w)$ and $(x,y,z)$ coordinates simultaneously.
(There is a degree of freedom in the choice of $R$ that corresponds
to a rotation around the direction of $\bm{p}$.)

Integration over directions then yields
\begin{align}
\bm{\varPhi}(0,0,0)
&=\int_{\mathrm{S}^2}F(0,0,0,\bm{p})\bm{p}
\,\mathrm{d}\sigma(\bm{p})
\notag\\*
&=\int_0^\pi\int_0^{2\pi}F(0,0,0,\bm{p}(\varphi,\psi))
\begin{pmatrix}
\cos\varphi\\\sin\varphi\cos\psi\\\sin\varphi\sin\psi
\end{pmatrix}
\sin\varphi\,\mathrm{d}\psi\,\mathrm{d}\varphi
\notag\\*
&=
\frac{\pi^6}{40}\varrho^{12}
\begin{pmatrix}
2\alpha_1+4(\delta_1+\zeta_1)-(\beta_2+\gamma_3)\\
2\delta_2+4(\alpha_2+\zeta_2)-(\beta_1+\varepsilon_3)\\
2\zeta_3+4(\alpha_3+\delta_3)-(\gamma_1+\varepsilon_2)
\end{pmatrix}
\end{align}
and similarly
\begin{align}
\mathrm{D}\bm{\varPhi}(0,0,0)&=
-\frac{\pi^4}2\varrho^{10}
\begin{pmatrix}1&0&0\\0&1&0\\0&0&1\end{pmatrix}
\end{align}
and by \eqref{Philin33}
\begin{align}
\begin{pmatrix}u^*\\v^*\\w^*\end{pmatrix}
&=\frac{\varrho^2}{20}
\begin{pmatrix}
2\alpha_1+4(\delta_1+\zeta_1)-(\beta_2+\gamma_3)\\
2\delta_2+4(\alpha_2+\zeta_2)-(\beta_1+\varepsilon_3)\\
2\zeta_3+4(\alpha_3+\delta_3)-(\gamma_1+\varepsilon_2)
\end{pmatrix}\;,
\end{align}
hence an explicit time step of size $\tau=\varrho^2/20$ of
the PDE system \eqref{ojapde33-I-1}--\eqref{ojapde33-I-3}.
\hfill$\Box$

\section{Proof of Lemma~\ref{lem-l1apde33-I}}
\label{app-l1a33proof}

We start again from the Taylor expansion 
\eqref{33tayloru}--\eqref{33taylorw} of the function $\bm{u}$
around the point $\bm{x}=\bm{0}$.

The $L^1$ median $(\bm{u}^*,\bm{v}^*,\bm{w}^*)$ of the function
values of $\bm{u}$ within the structuring element $B_\varrho$
is determined by the equilibrium conditions
\begin{align}
0 &= \iiint_{B_\varrho}
\frac{\bm{u}(x,y,z)-\bm{u}^*}{\lvert\bm{u}(x,y,z)-\bm{u}^*\rvert}
\,\mathrm{d}z\,\mathrm{d}y\,\mathrm{d}x \;.
\label{l1-33-medcond}
\end{align}
With the goal of the PDE approximation, we will determine the
median as linear function of the Taylor coefficients. 
Cross-effects between the Taylor coefficients will again be restricted
to higher order terms in $\varrho$ that can be neglected in our
asymptotic analysis for $\varrho\to0$. We can therefore study the
effects of the Taylor coefficients separately.

To start with $\alpha_1$, we insert in \eqref{l1-33-medcond} the
function $u=x+\alpha_1x^2$, $v=y$, $w=z$, and obtain the three
conditions
\begin{align}
&0 = \iiint_{B_\varrho}
\frac{x+\alpha_1x^2-u^*}
{\sqrt{N(x,y,z)}}
\,\mathrm{d}z\,\mathrm{d}y\,\mathrm{d}x \;,
\label{l1a33-medcond-alpha1-u}
\\
&0 = \iiint_{B_\varrho}
\frac{y-v^*}
{\sqrt{N(x,y,z)}}
\,\mathrm{d}z\,\mathrm{d}y\,\mathrm{d}x \;,
\label{l1a33-medcond-alpha1-v}
\\
&0 = \iiint_{B_\varrho}
\frac{z-w^*}
{\sqrt{N(x,y,z)}}
\,\mathrm{d}z\,\mathrm{d}y\,\mathrm{d}x \;,
\label{l1a33-medcond-alpha1-w}
\end{align}
where 
\begin{equation}
N(x,y,z)=(x+\alpha_1x^2-u^*)^2+(y-v^*)^2+(z-w^*)^2 \;.
\label{l1a33-medcond-alpha1-N}
\end{equation}
Condition~\eqref{l1a33-medcond-alpha1-v} can be turned by substituting
$-y$ for $y$ into the same equation with $-v^*$ in place of $v^*$, i.e.\
for any triple $(u^*,v^*,w^*)$ that satisfies the three conditions,
$(u^*,-v^*,w^*)$ does the same. By the convexity of the objective 
function of the $L^1$ median it follows that $(u^*,0,w^*)$ is in this case
also a minimiser. The same argument works for $w^*$. Hence, we can 
seek the median as $(u^*,v^*,w^*)=(\lambda\varrho^2,0,0)$ and need to
consider only Condition~\eqref{l1a33-medcond-alpha1-u}. Using the
substitution $x=\xi\varrho$, $y=\eta\varrho$, $z=\zeta\varrho$ we
obtain
\begin{align}
0 &= \iiint_{B_1}
\frac{\xi+(\alpha_1\xi^2-\lambda)\varrho}
{\sqrt{\bigl(\xi+(\alpha_1\xi^2-\lambda)\varrho\bigr)^2+\eta^2+\zeta^2}}
\,\mathrm{d}\zeta\,\mathrm{d}\eta\,\mathrm{d}\xi \;.
\end{align}
Splitting this integral into the integrals over $B_{\varrho^{2/5}}$ and
$B_1\setminus B_{\varrho^{2/5}}$, we see that the integral over
$B_{\varrho^{2/5}}$ is of order $\mathcal{O}(\varrho^{6/5})$ because the
integrand is absolutely bounded by $1$.
In the domain of the second integral, we have that 
$(\alpha_1\xi^2-\lambda)\varrho/(\xi^2+\eta^2+\zeta^2)=
\mathcal{O}(\varrho^{3/5})$
and therefore by Taylor expansion
\begin{align}
\kern2em&\kern-2em
\bigl((\xi+(\alpha_1\xi^2-\lambda)\varrho)^2+\eta^2+\zeta^2\bigr)^{-1/2}
\notag\\*
&= (\xi^2+\eta^2+\zeta^2)^{-1/2}\left(
1 - \frac{(\alpha_1\xi^2-\lambda)\varrho}{\xi^2+\eta^2+\zeta^2}
+ \mathcal{O}(\varrho^{6/5})\right)
\label{tayinvsqrt}
\end{align}
which leads to
\begin{align}
0 &= \underbrace{
\iiint_{B_1\setminus B_{\varrho^{2/5}}}
\frac{\xi}{\sqrt{\xi^2+\eta^2+\zeta^2}}
\,\mathrm{d}\zeta\,\mathrm{d}\eta\,\mathrm{d}\xi
}_{{}=0}
+ \varrho \iiint_{B_1\setminus B_{\varrho^{2/5}}}
\frac{\alpha_1\xi^2-\lambda}{\sqrt{\xi^2+\eta^2+\zeta^2}}
\,\mathrm{d}\zeta\,\mathrm{d}\eta\,\mathrm{d}\xi
\notag\\&\quad{}
- \varrho \iiint_{B_1\setminus B_{\varrho^{2/5}}}
\frac{\alpha_1\xi^2-\lambda}{(\xi^2+\eta^2+\zeta^2)^{3/2}}
\,\mathrm{d}\zeta\,\mathrm{d}\eta\,\mathrm{d}\xi
+\mathcal{O}(\varrho^{6/5})
\notag\\
&= \varrho \iiint_{B_1\setminus B_{\varrho^{2/5}}}
\frac{(\alpha_1\xi^2-\lambda)(\xi^2+\eta^2+\zeta^2-\xi^2)}
{(\xi^2+\eta^2+\zeta^2)^{3/2}}
\,\mathrm{d}\zeta\,\mathrm{d}\eta\,\mathrm{d}\xi
+\mathcal{O}(\varrho^{6/5})
\end{align}
and finally, by neglecting $\mathcal{O}(\varrho^{1/5})$ terms, to
\begin{align}
\frac{\lambda}{\alpha_1} &= \frac
{\iiint_{B_1\setminus B_{\varrho^{2/5}}}
\xi^2(\eta^2+\zeta^2)/(\xi^2+\eta^2+\zeta^2)^{3/2}
\,\mathrm{d}\zeta\,\mathrm{d}\eta\,\mathrm{d}\xi}
{\iiint_{B_1\setminus B_{\varrho^{2/5}}}^{\vphantom{A}}
(\eta^2+\zeta^2)/(\xi^2+\eta^2+\zeta^2)^{3/2}
\,\mathrm{d}\zeta\,\mathrm{d}\eta\,\mathrm{d}\xi}
\notag\\*
&= \frac{(1-\varrho^{8/5})\cdot2\pi/15}{(1-\varrho^{4/5})\cdot4\pi/3}
\underset{\varrho\to0}{\longrightarrow}\frac1{10}\;,
\end{align}
thus in the limit $(u^*,v^*,w^*) = \frac1{20}\varrho^2(u_{xx},0,0)$.

By permutation of variables, one finds 
$(u^*,v^*,w^*) = \frac1{20}\varrho^2(0,v_{yy},0)$
if $u=x$, $v=y+\delta_2y^2$, $w=z$, and 
$(u^*,v^*,w^*) = \frac1{20}\varrho^2(0,0,w_{zz})$
if $u=x$, $v=y$, $w=z+\zeta_3z^2$.

Turning to the case $u=x+\delta_1y^2$, $v=y$, $w=z$, we can conclude
$v^*=w^*=0$ by a similar symmetry argument as before. For the
remaining condition
\begin{align}
&0 = \iiint_{B_\varrho}
\frac{x+\delta_1y^2-u^*}
{\sqrt{N(x,y,z)}}
\,\mathrm{d}z\,\mathrm{d}y\,\mathrm{d}x \;,
\label{l1a33-medcond-delta1-u}
\\
&N(x,y,z) = (x+\delta_1y^2-u^*)^2+(y-v^*)^2+(z-w^*)^2\;,
\label{l1a33-medcond-delta1-N}
\end{align}
we proceed by the same substitutions, splitting of the integral
domain, and Taylor expansion of the denominator to finally obtain
\begin{align}
\frac{\lambda}{\delta_1} &= \frac
{\iiint_{B_1\setminus B_{\varrho^{2/5}}}
\eta^2(\eta^2+\zeta^2)/(\xi^2+\eta^2+\zeta^2)^{3/2}
\,\mathrm{d}\zeta\,\mathrm{d}\eta\,\mathrm{d}\xi}
{\iiint_{B_1\setminus B_{\varrho^{2/5}}}
(\eta^2+\zeta^2)/(\xi^2+\eta^2+\zeta^2)^{3/2}
\,\mathrm{d}\zeta\,\mathrm{d}\eta\,\mathrm{d}\xi}
\notag\\*
&= \frac{(1-\varrho^{8/5})\cdot4\pi/15}{(1-\varrho^{4/5})\cdot4\pi/3}
\underset{\varrho\to0}{\longrightarrow}\frac1{5}\;,
\end{align}
thus in the limit $(u^*,v^*,w^*) = \frac1{20}\varrho^2(2u_{yy},0,0)$.

By permutation of variables this also determines the $u_{zz}$, $v_{xx}$,
$v_{zz}$, $w_{xx}$ and $w_{yy}$ contributions of the PDE system
\eqref{ojapde33-I-1}--\eqref{ojapde33-I-3}.

For $u=x+\beta xy$, $v=y$, $w=z$ we find $u^*=w^*=0$ by symmetry considerations
and evaluate the second component of \eqref{l1-33-medcond} with 
$v^*=\mu\varrho^2$ to
\begin{align}
\frac{\mu}{\beta_1} &= -\frac
{\iiint_{B_1\setminus B_{\varrho^{2/5}}}
\xi^2\eta^2/(\xi^2+\eta^2+\zeta^2)^{3/2}
\,\mathrm{d}\zeta\,\mathrm{d}\eta\,\mathrm{d}\xi}
{\iiint_{B_1\setminus B_{\varrho^{2/5}}}
(\xi^2+\zeta^2)/(\xi^2+\eta^2+\zeta^2)^{3/2}
\,\mathrm{d}\zeta\,\mathrm{d}\eta\,\mathrm{d}\xi}
\notag\\*
&= -\frac{(1-\varrho^{8/5})\cdot\pi/15}{(1-\varrho^{4/5})\cdot4\pi/3}
\underset{\varrho\to0}{\longrightarrow}-\frac1{20}\;,
\end{align}
thus in the limit $(u^*,v^*,w^*) = \frac1{20}\varrho^2(0,-u_{xy},0)$.

Permutation of the variables yields all remaining terms of 
\eqref{ojapde33-I-1}--\eqref{ojapde33-I-3}.
\hfill$\Box$

\section{Proof of Lemma~\ref{lem-ojapde23-I}}
\label{app-23proof}

The Taylor expansion of $\bm{u}(x,y)$ around
$(0,0,0)$ up to second order is given as
\begin{align}
u(x,y) &= x 
+ \alpha_1 x^2 + \beta_1 xy + \delta_1 y^2 \;,
\label{23tayloru}
\\
v(x,y) &= y 
+ \alpha_2 x^2 + \beta_2 xy + \delta_2 y^2 \;,
\label{23taylorv}
\\
w(x,y) &= \phantom{z + {}}
\alpha_3 x^2 + \beta_3 xy + \delta_3 y^2 \;,
\label{23taylorw}
\end{align}
where $\alpha_1=\frac12u_{xx}$, $\beta_1=u_{xy}$ etc.

We will again express the median of data values within 
the structuring element $D_\varrho$ in terms of the
Taylor coefficients, neglecting terms of higher order in $\varrho$.
As in the settings before, cross-effects between the Taylor
coefficients influence only higher-order terms such that 
the Taylor coefficients can be considered separatedly.

As long as $\alpha_3=\beta_3=\delta_3=0$, the component $w$ is
identically zero, and thus $w^*=0$.
Moreover, the effects of $\alpha_1$, \ldots, $\delta_2$ on $u^*$
and $v^*$ are the same as in Lemma~\ref{lem-ojapde22-I}, such that
we need only to consider $\alpha_3$, $\beta_3$ and $\delta_3$.

Since by the influence of $w$ which varies just in order 
$\mathcal{O}(\varrho^2)$ around zero, the triangles whose area sum
is minimised by the 2D Oja median stay approximately in the $u$-$v$
plane and their deformation is restricted to higher order terms,
neither of $\alpha_3$, $\beta_3$ and $\delta_3$ influences the
first two median components $u^*$, $v^*$ asymptotically.

It remains to study $w^*$.
For the case $u=x$, $v=y$, $w=\beta_3xy$ we notice that mirroring the
structuring element by replacing $y$ with $-y$, followed by
replacing $w$ with $-w$, restores the original function. Thus, for
each minimiser $w^*$ in this case, $-w^*$ is also a minimiser, and
by convexity of the objective function $w^*=0$ is a minimiser.

Regarding $\alpha_3$ and $\delta_3$, notice that rotation of the
structuring element by $90$ degrees switches the roles of $\alpha$
and $\delta$. As this rotation leaves the input value set unchanged,
we see that $\alpha_3$ and $\delta_3$ must have equal effects.
We can therefore consider the rotationally symmetric case
$\alpha_3=\delta_3$.

Assume therefore that we have $u=x$, $v=y$, and $w=\alpha(x^2+y^2)$,
and $M=(0,0,w^*)$ is the sought 2D Oja median.
The median point constitutes an equilibrium between forces
exercised by point pairs $(A,B)$ with
$A=(u_1,v_1,w_1)=\bigl(x_1,y_1,\alpha(x_1^2+y_1^2)\bigr)$,
$B=(u_2,v_2,w_2)=\bigl(x_2,y_2,\alpha(x_2^2+y_2^2)\bigr)$. The force coming
from a single point pair $(A,B)$ is expressed by a vector of length
$\lvert AB\rvert$ in direction $MH$, where $H$ is the foot of the
altitude on $AB$ in the triangle $MAB$. Thus, the sought 2D Oja
median is a weighted $L^1$ median of the feet $H$, weighted with the
base lengths $\lvert AB\rvert$. The equilibrium condition for $M$ can 
therefore be written 
similarly as in \eqref{l1-33-medcond} 
as
\begin{align}
0 &= \iint_{D_\varrho}\iint_{D_\varrho}
\frac{MH}{\lvert MH\rvert} \lvert AB\rvert
\,\mathrm{d}y_2\,\mathrm{d}x_2\,\mathrm{d}y_1\,\mathrm{d}x_1
\label{int-mhab}
\end{align}
where the points $H$, $A$, $B$ still need to be expressed in
coordinates. Before we do so, we notice that reorganisation of
the quadruple integral in Radon-like polar coordinates as
in Appendix~\ref{app-22proof1} creates an additional weight
factor $\lvert AB\rvert$ and an integrand that is rotationally symmetric
with regard to the angular coordinate $\varphi$. One can therefore
drop the integration over $\varphi$ and consider just $\varphi=0$
as minimality condition. Denoting by $H'$, $A'$, $B'$ the projections
of $H$, $A$, $B$, respectively, to the $u$-$v$ plane, the case $\varphi=0$
describes a configuration in which $A'B'$ is aligned in $v$ direction,
and the altitude in $MA'B'$ therefore in $u$ direction. Since $H'$
deviates from the foot of the altitude in $MA'B'$ at most by higher order
terms, we can assume that $H'=(x,0)$, $A'=(x,y_1)$, $B'=(x,y_2)$.
The 3D points $A$ and $B$ are then given by 
$A=\bigl(x,y_1,\alpha(x^2+y_1^2)\bigr)$,
$B=\bigl(x,y_2,\alpha(x^2+y_2^2)\bigr)$. $H$ is given up to higher order terms
by $H=\bigl(x,0,\alpha(x^2-y_1y_2)\bigr)$.

This leads to the simplified equilibrium condition
\begin{align}
0 &= \int_0^\varrho 
\int_{-\sqrt{\varrho^2-x^2}}^{\sqrt{\varrho^2-x^2}}
\int_{-\sqrt{\varrho^2-x^2}}^{\sqrt{\varrho^2-x^2}}
\frac{
(y_2-y_1)^2
\bigl(\alpha(x^2-y_1y_2)-w^*\bigr)}
{\sqrt{x^2+\bigl(\alpha(x^2-y_1y_2)-w^*\bigr)^2}}
\,\mathrm{d}y_2\,\mathrm{d}y_1\,\mathrm{d}x \;,
\end{align}
and after substituting $x=\xi\varrho$, $y_1=\eta_1\varrho$, 
$y_2=\eta_2\varrho$, $w^*=\nu\varrho^2$ one has
\begin{align}
0 &= \int_0^1
\int_{-\sqrt{1-\xi^2}}^{\sqrt{1-\xi^2}}
\int_{-\sqrt{1-\xi^2}}^{\sqrt{1-\xi^2}}
\frac{
(\eta_2-\eta_1)^2
\bigl(\alpha(\xi^2-\eta_1\eta_2)-\nu\bigr)}
{\sqrt{\xi^2+\varrho^2\bigl(\alpha(\xi^2-\eta_1\eta_2)-\nu\bigr)^2}}
\,\mathrm{d}\eta_2\,\mathrm{d}\eta_1\,\mathrm{d}\xi\;.
\end{align}
Splitting the integration range of the outer integral to the two intervals
$[0,\varrho^{2/3}]$ and $[\varrho^{2/3},1]$ we see that the
first integral yields $\mathcal{O}(\varrho^{2/3})$ since its integrand
is absolutely bounded by $1$, whereas the second integral is simplified
further by noticing that 
$\sqrt{\xi^2+\varrho^2\bigl(\alpha(\xi^2-\eta_1\eta_2)-\nu\bigr)^2}
=\xi\bigl(1+\mathcal{O}(\varrho^{2/3})\bigr)$
to
\begin{align}
0 &= \int_{\varrho^{2/3}}^1
\int_{-\sqrt{1-\xi^2}}^{\sqrt{1-\xi^2}}
\int_{-\sqrt{1-\xi^2}}^{\sqrt{1-\xi^2}}
\frac{
(\eta_2-\eta_1)^2
\bigl(\alpha(\xi^2-\eta_1\eta_2)-\nu\bigr)}{\xi}
\,\mathrm{d}\eta_2\,\mathrm{d}\eta_1\,\mathrm{d}\xi
+\mathcal{O}(\varrho^{2/3})
\end{align}
from which we obtain
\begin{align}
\frac{\nu}{\alpha} &= \frac
{
\int_{\varrho^{2/3}}^1
\int_{-\sqrt{1-\xi^2}}^{\sqrt{1-\xi^2}}
\int_{-\sqrt{1-\xi^2}}^{\sqrt{1-\xi^2}}
(\eta_2-\eta_1)^2(\xi^2-\eta_1\eta_2)\xi^{-1}
\,\mathrm{d}\eta_2\,\mathrm{d}\eta_1\,\mathrm{d}\xi
}
{\int_{\varrho^{2/3}}^1
\int_{-\sqrt{1-\xi^2}}^{\sqrt{1-\xi^2}}
\int_{-\sqrt{1-\xi^2}}^{\sqrt{1-\xi^2}}
(\eta_2-\eta_1)^2\xi^{-1}
\,\mathrm{d}\eta_2\,\mathrm{d}\eta_1\,\mathrm{d}\xi}
\end{align}
and after integral evaluation
\begin{align}
\frac{\nu}{\alpha} &= \frac
{-\ln\varrho\cdot16/27-10/27}
{-\ln\varrho\cdot16/9-16/9}
\underset{\varrho\to0}{\longrightarrow}\frac13 \;.
\end{align}
Since for the given function one has $w_{xx}=w_{yy}=2\alpha$, it follows
that each of $w_{xx}$ and $w_{yy}$ effects $w^*$ with weight $1/12$,
which concludes the proof.
\hfill$\Box$

\section{Numerical Scheme for the PDE \eqref{ojapde23}}
\label{app-pdealgo}

We use the notations from Section~\ref{ssec-pdenum}.
By square brackets $[\ldots]$ we denote discrete approximations of the
enclosed derivative expressions at pixel $(i,j)$ in time step $k$.
The numerical scheme
for the PDE \eqref{ojapde23} proceeds for each pixel $(i,j)$ as follows.

\begin{enumerate}
\item Compute the central difference approximations
\begin{align}
[\bm{u}_x]
&:= \tfrac{1}{2h} (\bm{u}^k_{i+1,j}-\bm{u}^k_{i-1,j})\;,
\\
[\bm{u}_y]
&:= \tfrac{1}{2h} (\bm{u}^k_{i,j+1}-\bm{u}^k_{i,j-1})\;.
\end{align}
\item
From $[\mathrm{D}\bm{u}]=\bigl([\bm{u}_x]~|~[\bm{u}_y]\bigr)$ compute the
$3\times3$ tensor product matrix 
$\bm{C}:=[\mathrm{D}\bm{u}][\mathrm{D}\bm{u}]\transpose$.
Compute the spectral decomposition of $\bm{C}$,
\begin{align}
\bm{C} &= \bm{Q} \bm{\varLambda} \bm{Q}\transpose
\end{align}
where $\bm{Q}$ is orthogonal, and $\bm{\varLambda}$ is the diagonal
matrix of the (nonnegative) eigenvalues of $\bm{C}$ in decreasing order,
$\lambda_1\ge\lambda_2\ge\lambda_3$.
\item Apply to the input values $\bm{u}_{i,j}^k$ the orthogonal
transform
\begin{align}
\hat{\bm{u}}_{i,j}^k &:= \bm{Q}\transpose\bm{u}_{i,j}^k\;.
\end{align}
Note that hereafter, the first and second channel of $\hat{\bm{u}}$
hold the directions of dominant variation within the patch, i.e.\ the
first two basis vectors of the transformed data set span the
tangential space of the image graph. Moreover, the gradients of the first
two channels of $\hat{\bm{u}}$ are orthogonal in the $(x,y)$ plane.
\item
Compute the central difference approximations
\begin{align}
[\hat{\bm{u}}_x]
&:= \tfrac{1}{2h} (\hat{\bm{u}}^k_{i+1,j}-\hat{\bm{u}}^k_{i-1,j})\;,
\\
[\hat{\bm{u}}_y]
&:= \tfrac{1}{2h} (\hat{\bm{u}}^k_{i,j+1}-\hat{\bm{u}}^k_{i,j-1})\;,
\\
[\hat{\bm{u}}_{xx}]
&:= \tfrac1{h^2} (\hat{\bm{u}}^k_{i+1,j}-2\hat{\bm{u}}^k_{i,j}
+\hat{\bm{u}}^k_{i-1,j})\;,
\\
[\hat{\bm{u}}_{yy}]
&:= \tfrac1{h^2} (\hat{\bm{u}}^k_{i,j+1}-2\hat{\bm{u}}^k_{i,j}
+\hat{\bm{u}}^k_{i,j-1})\;,
\\
[\hat{\bm{u}}_{xy}]
&:= \tfrac1{4h^2} (\hat{\bm{u}}^k_{i+1,j+1}-\hat{\bm{u}}^k_{i+1,j-1}
-\hat{\bm{u}}^k_{i-1,j+1}+\hat{\bm{u}}^k_{i-1,j-1})\;.
\end{align}
\item Compute the first contribution to $\hat{\bm{u}}_t$ as
\begin{align}
\hat{\bm{z}}_1 &:= [\hat{\bm{u}}_{xx}] + [\hat{\bm{u}}_{yy}]
\;.
\end{align}
\item
From the first component $\hat{u}$ of $\hat{\bm{u}}$, 
determine the image adaptive directions
\begin{align}
\bm{\eta} &:= \begin{pmatrix}c\\s\end{pmatrix} =
\frac{1}{\sqrt{[\hat{u}_x]^2+[\hat{u}_y]^2}}
\begin{pmatrix} [\hat{u}_x]\\{}[\hat{u}_y] \end{pmatrix}\;,\\
\bm{\xi} &:= \begin{pmatrix}-s\\c\end{pmatrix}
\end{align}
and the directional derivatives
\begin{align}
[\hat{\bm{u}}_{\bm{\eta}}] &:=  c[\hat{\bm{u}}_x]+s[\hat{\bm{u}}_y]\;,\\
[\hat{\bm{u}}_{\bm{\xi}}]  &:= -s[\hat{\bm{u}}_x]+c[\hat{\bm{u}}_y]\;,\\
[\hat{\bm{u}}_{\bm{\eta\eta}}] &:= c^2[\hat{\bm{u}}_{xx}]
+2cs[\hat{\bm{u}}_{xy}]+s^2[\hat{\bm{u}}_{yy}]\;,\\
[\hat{\bm{u}}_{\bm{\xi\xi}}] &:= s^2[\hat{\bm{u}}_{xx}]
-2cs[\hat{\bm{u}}_{xy}]+c^2[\hat{\bm{u}}_{yy}]\;,\\
[\hat{\bm{u}}_{\bm{\eta\xi}}] &:= 
cs\bigl([\hat{\bm{u}}_{yy}]-[\hat{\bm{u}}_{xx}]\bigr)
+(c^2-s^2)[\hat{\bm{u}}_{xy}]\;.
\end{align}
\item
Compute the second contribution to $\hat{\bm{u}}_t$ as
\begin{align}
\hat{\bm{z}}_2 &:= 2 
\bigl( [\hat{u}_{\xi\xi}], [\hat{v}_{\eta\eta}], 0 \bigr)\transpose
\;.
\end{align}
\item
Compute one-sided derivatives
\begin{align}
[\hat{\bm{u}}_x]^+
&:= \tfrac{1}{h} (\hat{\bm{u}}^k_{i+1,j}-\hat{\bm{u}}^k_{i,j})\;,
\\
[\hat{\bm{u}}_x]^-
&:= \tfrac{1}{h} (\hat{\bm{u}}^k_{i,j}-\hat{\bm{u}}^k_{i-1,j})\;,
\\
[\hat{\bm{u}}_y]^+
&:= \tfrac{1}{h} (\hat{\bm{u}}^k_{i,j+1}-\hat{\bm{u}}^k_{i,j})\;,
\\
[\hat{\bm{u}}_y]^-
&:= \tfrac{1}{h} (\hat{\bm{u}}^k_{i,j}-\hat{\bm{u}}^k_{i,j-1})\;.
\end{align}
If $[\hat{\bm{u}}_x]^+$ and $[\hat{\bm{u}}_x]^-$ have opposite sign,
replace the one with larger absolute value with their sum
$[\hat{\bm{u}}_x]^++[\hat{\bm{u}}_x]^-$ and set the other one to zero
(minmod stabilisation).
Proceed in the same way for $[\hat{\bm{u}}_y]^\pm$.
From the so obtained approximations, compute one-sided directional derivatives
\begin{align}
[\hat{\bm{u}}_{\bm{\eta}}]^\pm 
&:=  c[\hat{\bm{u}}_x]^\pm+s[\hat{\bm{u}}_y]^\pm\;,\\
[\hat{\bm{u}}_{\bm{\xi}}]^\pm  
&:= -s[\hat{\bm{u}}_x]^\mp+c[\hat{\bm{u}}_y]^\pm\;,
\end{align}
if $c,s\ge0$, and analogously for other sign combinations of $c$, $s$.
\item Compute regularised approximations 
$R_v$ for $v_{\eta\xi}/v_{\xi}$ 
and 
$R_u$ for $u_{\eta\xi}/u_{\eta}$
as
\begin{align}
R_v &:= \frac{2[\hat{v}_{\bm{\eta\xi}}][\hat{v}_{\bm{\xi}}]}
{\bigl([\hat{v}_{\bm{\xi}}]^+\bigr)^2+
\bigl([\hat{v}_{\bm{\xi}}]^-\bigr)^2+2\varepsilon}\;,
\\
R_u &:= \frac{2[\hat{u}_{\bm{\eta\xi}}][\hat{u}_{\bm{\eta}}]}
{\bigl([\hat{u}_{\bm{\eta}}]^+\bigr)^2+
\bigl([\hat{u}_{\bm{\eta}}]^-\bigr)^2+2\varepsilon}
\end{align}
with a fixed numerical regularisation parameter $\varepsilon$.
\item Compute the third contribution to
$\hat{\bm{u}}_t$ by the upwind discretisation
\begin{align}
\bm{z}_3 &:= 
\bigl( R_v [\hat{u}_{\bm{\eta}}]^\mp, R_u [\hat{v}_{\bm{\xi}}]^\mp,
0\bigr)\transpose\;,
\end{align}
choosing in each component the backward approximation $[\ldots]^-$ if the 
preceding factor is positive, and $[\ldots]^+$ otherwise.
\item
Let $[\hat{\bm{u}}_t]:= \bm{z}_1+\bm{z}_2-\bm{z}_3$ and by inverting the
orthogonal transform
\begin{align}
[\bm{u}_t] &:= \bm{Q}\,[\hat{\bm{u}}_t] \;.
\end{align}
\item Compute $\bm{u}_{i,j}^{k+1} = \bm{u}_{i,j}^k + \tau[\bm{u}_t]$.
\end{enumerate}

\end{document}